\documentclass{article}
\usepackage{float}

\usepackage{enumerate}
\usepackage{natbib}

\usepackage{amsthm,amssymb, amsmath} 
\usepackage{dsfont}
\usepackage{color}
\usepackage{ifthen} 
\usepackage{xspace}
\usepackage{url}
\usepackage{braket}
\usepackage[title]{appendix}
\usepackage{graphicx}
\usepackage{booktabs}

\usepackage{geometry}
 \newtheorem{Theorem}{Theorem}

\newtheorem{lemma}[Theorem]{Lemma}

\newtheorem{Definition}[Theorem]{Definition}
\newtheorem{Proposition}[Theorem]{Proposition}

\newcommand{\Appendix}{Appendix}

\renewcommand{\Pr}{\mathds{P}} 
\newcommand{\Prob}{\Pr}
\newcommand{\E}{\mathbb{E}}
\newcommand{\R}{\mathbb{R}} 
\renewcommand{\H}{\mathcal{H}}
\newcommand{\Var}{\mathbb{V}ar}
\newcommand{\tb}{\textbf} 
\newcommand{\ind}{\mathds{1}}

\usepackage[colorlinks=true]{hyperref}
\setcitestyle{square}
\newcommand*\samethanks[1][\value{footnote}]

\makeatletter
\newcommand{\specificthanks}[1]{\@fnsymbol{#1}}
\makeatother

\begin{document}

\title{Kernelized Stein Discrepancy Tests of Goodness-of-fit for Time-to-Event Data\thanks{TF, WX, AG are grateful for the support from the Gatsby
Charitable Foundation. NR is supported by Thomas Sauer-
wald’s ERC Starting Grant 679660.}}

\author{
  Tamara Fern\'andez \thanks{Corresponding author.} \textsuperscript{ \specificthanks{3}}\\
  University College London\\
  \texttt{t.a.fernandez@ucl.ac.uk} \\
   \and
  Nicol\'as Rivera\thanks{ Same contribution.} \\
  University of Cambridge\\
  \texttt{nicolas.rivera@cl.cam.ac.uk} \\ 
  \and
   Wenkai Xu\textsuperscript{ \specificthanks{3}}\\
  University College London\\
  \texttt{wenkai.xu.16@ucl.ac.uk} \\
   \and
  Arthur Gretton \\
  University College London\\
  \texttt{arthur.gretton@gmail.com } \\
}

\maketitle

\numberwithin{equation}{section}

\begin{abstract}
Survival Analysis and Reliability Theory are concerned with the analysis of time-to-event data, in which observations correspond to waiting times until an event of interest such as death from a particular disease or failure of a component in a mechanical system. This type of data is unique due to the presence of censoring, a type of missing data that occurs when we do not observe the actual time of the event of interest but, instead, we have access to an approximation for it given by random interval in which the observation is known to belong. Most traditional methods are not designed to deal with censoring, and thus we need to adapt them to censored time-to-event data. In this paper, we focus on non-parametric goodness-of-fit testing procedures based on combining the Stein's method and kernelized discrepancies. While for uncensored data, there is a natural way of implementing a kernelized Stein discrepancy test, for censored data there are several options, each of them with different advantages and disadvantages. In this paper, we propose a collection of kernelized Stein discrepancy tests for time-to-event data, and we study each of them theoretically and empirically; our experimental results show that our proposed methods perform better than existing tests, including previous tests based on a kernelized maximum mean discrepancy.
\end{abstract}

\section{Introduction}\label{sec:intro}
An important topic of study in statistics is the distribution of times to a critical event, otherwise known as survival times: examples include the infection time from a disease \cite{andersen2012competing, mirabello2009osteosarcoma}; the death time of a patient in a clinical trial \cite{collett2015modelling, biswas2007statistical};  or the possible re-offending times for released criminals \cite{chung1991survival}. Survival data are frequently subject to censoring: the time of interest is not observed, but rather a bound on it. The most common scenario studied is {\em right censoring}, where a lower bound on the survival time is observed, for instance, a patient might leave a clinical trial before it is completed, meaning that we only obtain a lower bound on the time of death (the definitions and terminologies for the survival analysis setting will be provided in Section \ref{sec:background}).

We address the setting where a model of survival times is proposed, and it is desired to test this model against observed data in the presence of censoring: this is known as \emph{goodness-of-fit} testing. When departures from the model follow a known parametric family, a number of classical tests are available, being the most popular in practice the Log-rank test \cite{hollander1979testing},  and its generalization, the weighted Log-rank test \cite{brendel2014weighted}. For an overview of these and other methods we refer the reader to \cite{klein2006survival}

In the event of more general departures from the null, kernel methods may be used to construct a powerful class of non-parametric tests to detect a greater range of alternative scenarios. For the uncensored case, a popular class of kernel goodness-of-fit tests utilize Stein's method \cite{barbour2005introduction, chen2010, ley2017stein,gorham2015measuring}  to develop a test statistic  \cite{liu2016kernelized, chwialkowski2016kernel,GorMac2017, jitkrittum2017linear}, which can be computed even when the model is known only up to normalization. In this paper we consider the particular case of kernel Stein discrepancies (KSDs) which are described in  Section \ref{sec:background}. While an alternative strategy would be simply to run a two-sample test using samples from the model, using for instance the maximum mean discrepancy (MMD) \cite{GreBorRasSchetal12}, Stein tests are more computationally efficient (no additional sampling is needed), and can take advantage of model structure to achieve better test power. KSD tests have been extended to various settings such as discrete variable models \cite{yang2018goodness}, point process \cite{yang2019stein}, latent variable models \cite{kanagawa2019kernel}, and directional data \cite{xu2020stein}.

\begin{figure*}[ht]
    \centering
    \includegraphics[width=1.\textwidth]{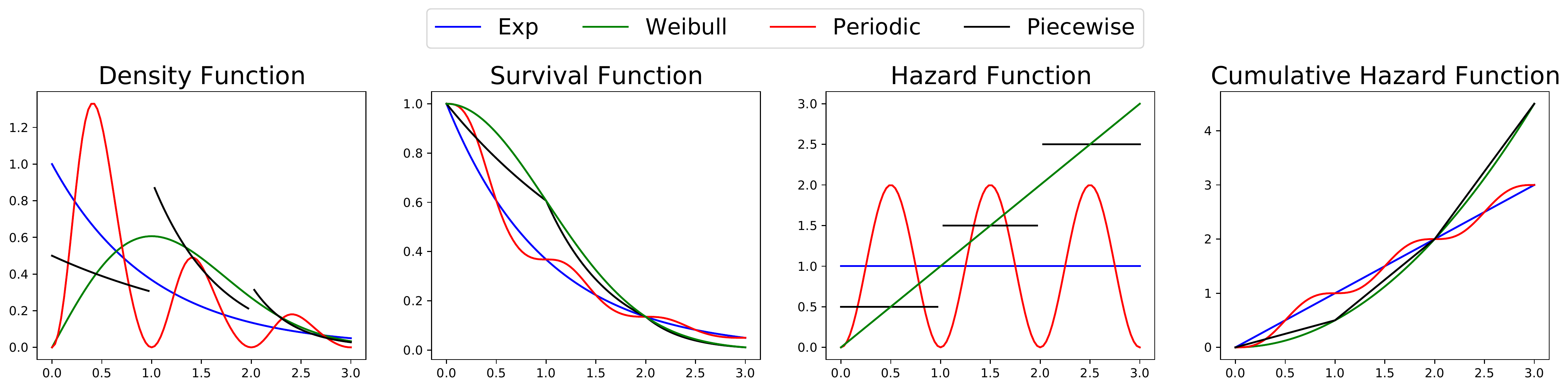}
    \vspace{-0.7cm}
    \caption{Example functions related to survival analysis. }
    \label{fig:dist_demo}
\end{figure*}

In the present work, we  propose to generalize Stein goodness-of-fit tests  to the setting of survival analysis with right-censored data. In Section \ref{sec: stein operator}, we introduce three separate approaches to constructing a Stein operator in the presence of censoring: the first, the {\em Survival Stein Operator}, is the most direct generalization of the Stein operator used in the uncensored KSD test. The second, the {\em Martingale Stein Operator}, uses a different construction, based on a classical martingale studied in the survival analysis literature. The third, the {\em Proportional Stein Operator}, is designed for composite null hypotheses: in this case, the {\em hazard function} (that is, the instantaneous probability of an event at a given time, conditioned on survival to that time) is known only up to a constant of proportionality. For instance, we may wish to use a constant hazard as the null hypothesis, without specifying in advance the value of the constant.

The rest of the paper is structured as follows: in Section \ref{sec:CKSD}, we construct kernel statistics of goodness-of-fit, based on each of the operators previously introduced.  We characterize the asymptotics of each statistic in Section \ref{sec:gof-test}. We find that in order to guarantee convergence in distribution under the null, the kernel statistic based on the Survival Stein Operator requires more restrictive conditions than the statistic built on the Martingale Stein Operator. In other words, the straightforward extension of the uncensored test is in fact the more restrictive approach of the two. Stronger assumptions again are required in obtaining convergence in distribution for the Proportional Stein Operator statistic, which should come as no surprise, given that the null is now an entire model class. For each statistic, we propose a wild bootstrap approach to obtain the test threshold. Empirical studies and results are presented in Section \ref{sec:exp}, where we compare with a recent state-of-the-art
non-parametric test  for censored data by \citet{fernandez2019maximum} based on the MMD, which
has been shown to outperform classical tests. For challenging cases, our Stein tests surpass the MMD test.

\section{Background} \label{sec:background}

\paragraph{Kernel Stein Discrepancy}
We briefly review   
the kernel Stein discrepancy (KSD) in the absence of censoring \cite{chwialkowski2016kernel,liu2016kernelized}, which is inspired from \cite{gorham2015measuring,ley2017stein}.
Let $f_0$ be a smooth probability density on $\mathbb{R}$.
For a bounded smooth function $\omega:\mathbb{R} \to \mathbb{R}$, the Stein operator $\mathcal{T}_0$ is given by
\begin{align}
\mathcal{T}_0 \hspace{0.05cm} \omega(x)&= \omega(x)(\log f_0(x))' + \omega'(x),
\label{eq:steinRd}
\end{align}
where $'$ denotes derivative w.r.t $x$.
Since $f_0$ vanishes at the boundary and $\omega$ is bounded, integration by parts on $\mathbb{R}$ results in Stein's Lemma:
$$
{\E}_0 [\mathcal{T}_0 \omega] = \int (\mathcal{T}_0 \omega)(x)f_0(x) = 0,
$$
under some regularity conditions.
Since the Stein operator $\mathcal{T}_0$ depends on the density $f_0$ only through the derivative of $\log f_0$, it does not involve the normalization constant of $f_0$, which is a useful property for dealing with unnormalized models \cite{hyvarinen2005estimation}.

Let $\mathcal{H}$ be a reproducing kernel Hilbert space (RKHS) on $\mathbb{R}$ with associated  kernel $K$. By using the Stein operator above, the kernel Stein discrepancy (KSD) \cite{chwialkowski2016kernel,liu2016kernelized} between two densities $f_X$ and $f_0$ is defined as
\begin{equation}
\operatorname{KSD}(f_X \| f_0) =\sup_{\omega \in B_1(\mathcal H)} \mathbb{E}_{X}[\mathcal{T}_0 \omega],
\label{eq:ksd}
\end{equation}
where $B_1(\mathcal H)$ denotes the unit ball of $\mathcal{H}$, and $\E_X$ denotes the expectation w.r.t. the density $f_X$. It is easy to see that $\mathrm{KSD}(f_X,f_0) \geq 0$ and that $\mathrm{KSD}(f_X\|f_0) = 0$ for $f_X = f_0$. Moreover, under some regularity conditions, we have that $\mathrm{KSD}(f_X,f_0) = 0$ if and only if $f_X = f_0$ \cite{chwialkowski2016kernel}.

By using standard properties of RKHSs, we can conveniently write  $\mathrm{KSD}(f_X,f_0)$  as
\begin{align}
\operatorname{KSD}^2(f_X\|f_0) = {\E}_{x,y \sim f_X} [h_0(x,y)], \label{eq:KSDequiv}
\end{align}
where $h_0(x,y) = $
$$\langle \log f_0(x)'K(x,\cdot) +  K'(x,\cdot), \log f_0(y)'K(y,\cdot) +  K'(y,\cdot)\rangle_{\mathcal H},$$
with $\langle \cdot, \cdot \rangle_{\mathcal H}$ denoting the inner product of $\mathcal H$.

\paragraph{Censored Data} 
Let $(X_1,\ldots,X_n) \overset{\textnormal{ i.i.d.}}{\sim}F_X$ be the survival times, which are non-negative real-valued random variables of interest, and let
$(C_1,\ldots,C_n)\overset{\textnormal{ i.i.d.}}{\sim}F_C$ be another collection of non-negative random variables called censoring times. In this work, we assume the non-informative censoring setting, where the censoring times are independent of the survival times. The data we observe correspond to $(T_i,\Delta_i)$ where $T_i = \min\{X_i,C_i\}$ and $\Delta_i=\ind_{\{X_i\leq C_i\}}$. We can imagine that $X_i$ is the time of interest (death of a patient) and $C_i$ is the time a patient leaves the study for some other reason, thus, for some patients we observe their actual death time, whereas for others we just observe a lower bound (the time they left the study). $\Delta_i$ indicates if we are observing $X_i$ or $C_i$.

We denote by $f_T$, $f_X$ and $f_C$, the respective density functions associated with the random variables $T$, $X$ and $C$. Similarly, we denote by $F_T$, $F_X$ and $F_C$, the respective cumulative distribution functions; and by $S_T=1-F_T$, $S_X=1-F_X$ and $S_C=1-F_C$, the survival functions. An important element in survival analysis is the hazard function which represents the instantaneous risk of dying at a given time (as $X$ usually refers to a death time). Given a distribution with density $f_X$ and survival function $S_X$, the hazard function $\lambda_X(x)$ is given by $f_X(x)/S_X(x)$, which can be seen as the density at $x$ of a random variable $X$ conditioned on the event $\{X\geq x\}$. The corresponding cumulative hazard function is defined as $\Lambda_X(x) = \int_0^x \lambda_X(t)dt$. A useful feature of the hazard function is that there is a one-to-one relation between hazard and density functions through the relation $S_X(x) = e^{-\Lambda_X(x)}$. For the random variables $T$ and $C$, we denote by $\lambda_T$ and $\lambda_C$ their respective hazard functions, and by $\Lambda_T$ and $\Lambda_C$, their cumulative hazards functions. As a remark, every continuous non-negative function $\lambda:\R_+ \to \R$ can be a hazard function, as long as $\int_{\R_+}\lambda(t)dt = \infty$, thus, describing hazards is much easier than describing densities, as we do not need to worry about normalization constants. Examples of corresponding functions for different models are displayed in Figure \ref{fig:dist_demo}.

As observations correspond to pairs $(T_i, \Delta_i)$, it is convenient to consider the joint measure $\mu$ on $\R_+\times\{0,1\}$ induced by the pair $(T,\Delta)$. We write $\mu_X$ to denote the measure $\mu$ when the survival times of interest $X_i$ are generated according to $f_X$, and  $\mu_0$ if they are generated under $f_0$ (i.e., under the null). Note that $\mu_X$ and $\mu_0$ also depend on $f_C$, however we don't make this dependence  explicit, since for goodness-of-fit we only care about $f_0$ and $f_X$.

Finally, for any function $\phi$, the following identities hold, which the reader should keep in mind for later use:

\begin{align}
    \E_{X}
[\Delta \phi(T)]& = \int_0^{\infty} \phi(s)f_X(s)S_C(s)ds,\label{eqn:expectedPhi}\\
    \E_{X}
[(1-\Delta) \phi(T)]& = \int_0^{\infty} \phi(s)f_C(s)S_X(s)ds.\label{eqn:expectedPhi0}
\end{align}
Here $\E_X=\E_{\mu_X}$ means that we are taking expectation w.r.t. $(T,\Delta)\sim \mu_X$. Similarly, we write $\E_0$ to indicate $(T,\Delta)\sim \mu_0$ (under the null hypothesis).

\section{Stein Operator for Censored Data} \label{sec: stein operator}
In this section, we describe a set of Stein operators for censored data.  We denote by $\Omega$ the set of functions $\R_+\times\{0,1\}\to\R$, and recall that $\mu_0$ is the measure induced by 
data $(T,\Delta)$ under the null hypothesis. 

\begin{Definition}
Let $\mathcal H\subseteq L^2(f_0)$. We call $\mathcal{T}_0: L^2(f_0) \to \Omega
$ a Stein operator for $\mathcal H $ if for each $\omega \in \mathcal H$
\begin{equation}\label{eq:stein_identity}
\E_0\left[(\mathcal T_0\omega)(T, \Delta)\right]= 0.
\end{equation}
\end{Definition}

An interesting technical point is that our operator takes functions $\omega:\R_+\to\R$ and maps them to $\Omega$. The idea behind having these two spaces is that while our data of interest is a time (hence the space $\mathcal H$ of functions $\R_+ \to \R$), we actually observe pairs $(T_i,\Delta_i)$, hence we need functions in $\Omega$.

We choose the general class $\mathcal{H}$ to be an RKHS.
We assume that $\mathcal H$ contains only differentiable and bounded functions, and that if $\omega \in \mathcal H$ then $\omega' \in \mathcal H$.  These requirements are not restrictive and most of the standard kernels in the literature generate RKHSs with these properties, including the Gaussian kernel (furthermore, we can avoid this restriction, but we keep it as it is convenient for the exposition of the paper). Further properties of $\mathcal H$ will be imposed if needed in particular cases.

\subsection{Survival Stein Operator}
Observe that $T_i=X_i$ if and only if $\Delta_i=1$. One might be tempted to  use only the uncensored observations to approximate $\int_0^{\infty}(\mathcal{T}_0 \omega)(x)f_0(x)dx$ (where $\mathcal{T}_0$ is the standard Stein operator in \eqref{eq:steinRd}) by computing $$\frac{1}{n}\sum_{i=1}^n \Delta_i(\mathcal{T}_0\omega)(T_i)=\frac{1}{n}\sum_{i=1}^n \Delta_i(\mathcal{T}_0\omega)(X_i),$$
however, this  sum does not converge to $\int_0^{\infty}(\mathcal{T}_0\omega)(x)f_0(x)dx$ as the term $\Delta_i$ introduces bias due to censoring. Indeed, such an empirical average converges to 
$\int_0^{\infty} (\mathcal{T}_0\omega)(x)S_C(x)f_X(x)$. To account for this bias we redefine  $\mathcal T_0:\H^{(s)}\to \Omega$ as
\begin{align}\label{eqn:defiT0}
    (\mathcal T_0 \omega)(x,\delta) = \delta \frac{\left(\omega(x)S_C(x)f_0(x)\right)'}{S_C(x)f_0(x)}+\omega(0)f_0(0).
\end{align}
 Here we write $\mathcal H^{(s)}$ instead of $\mathcal H$ whenever we assume that the additional condition is satisfied,
\begin{align}
     \int_{\R^+}  |\left(\omega(x)S_C(x)f_0(x) \right)'|dx<\infty,\quad  \forall \omega \in \H\label{eqn:conditionS},
\end{align}
which guarantees that the operator is well-defined. Notice that $\omega(0)f_0(0)$ in equation \eqref{eqn:defiT0} appears since we do not necessarily assume a vanishing boundary at 0.

Under the null hypothesis, $(T_i,\Delta_i)\sim \mu_0$, it holds that  
\begin{align}\label{eq:T_0stein}
\frac{1}{n}\sum_{i=1}^n (\mathcal T_0 \omega)(T_i,\Delta_i) \to\E_{0}[(\mathcal T_0 \omega)(T,\Delta)]
\end{align}
as the number of data points tends to infinity, and $\E_{0}[(\mathcal T_0 \omega)(T,\Delta)] = 0$  due to Equation \eqref{eqn:expectedPhi} and the fact that
\begin{align}
\int_{\R^+} (\omega(x)S_C(x)f_0(x))'dx+\omega(0)f_0(0) = 0,
\end{align}
which is proved using integration by parts. Notice that in this argument we use that  $\mathcal H^{(s)}$ only contains bounded functions, allowing us to get rid of the boundary at infinity.

The operator $\mathcal T_0$ can be seen as a natural extension of the Stein operator \cite{gorham2015measuring} to censored data. Observe that in the uncensored case, $S_C(x) \equiv 1$ recovers the standard Stein operator.

Unfortunately, in the goodness-of-fit setting, we only have access to the null distribution $f_0(x)$ but not to the censoring distribution $f_C(x)$, thus $S_C(x)$ needs to be estimated. The standard estimator for $S_C$ is the Kaplan-Meier estimator \cite{kaplan1958nonparametric} which is very data inefficient, leading to an unsatisfactory testing procedure.

To bypass the approximation of $S_C$ we define the survival Stein operator $\mathcal T_0^{(s)}:\mathcal H^{(s)}\to \Omega$ as
\begin{align}\label{eqn:defiT0S}
(\mathcal{T}_0^{(s)}{\omega})(x,\delta)&=
\delta\omega'(x)+\frac{\lambda_0'(x)}{\lambda_0(x)}\delta\omega(x)-\lambda_0(x)\omega(x)+\lambda_0(0)\omega(0)
\end{align}

\begin{Proposition}\label{Prop:T0TsExEquiv}
Consider $\mathcal T_0$ and $\mathcal T_0^{(s)}$ defined in equations~\eqref{eqn:defiT0} and \eqref{eqn:defiT0S}, respectively. Let $(T,\Delta)\sim \mu_{0}$. 
Then
\begin{align}
    \E_{0}[(\mathcal T_0^{(s)} \omega)(T,\Delta)]= \E_0[(\mathcal T_0 \omega)(T,\Delta)]=0,\quad \forall  \omega \in \mathcal H^{(s)}.\nonumber
    \end{align}
\end{Proposition}

The previous proposition says that if the data we observed was generated from $\mu_0$ then the expectation of the operators $\mathcal T_0$ and $\mathcal T_0^{(s)}$ are equal for each function in $\mathcal H^{(s)}$. However, the relation between $\mathcal T_0$ and $\mathcal T_0^{(s)}$ is stronger than merely equality in expectation, indeed, under a slightly stronger condition on the form of the distribution $f_0$ and $f_C$ we get the following result, which is proven in \Appendix~\ref{app:proofs}.
\begin{Proposition}\label{Prop:T0TsExEquivSample}
 Assume that 
\begin{align}
\int_0^\infty (\lambda_C(x)+\lambda_0(x))f_C(x)f_0(x)<\infty,\label{eqn:moment1}
\end{align}
then, under the null hypothesis, i.e. $(T_i,\Delta_i)\sim \mu_0$, we have that, as the number of data points tends to infinity,
\begin{align*}
\sup_{\omega\in B_1(\mathcal{H})} \frac{1}n \sum_{i=1}^n (\mathcal{T}_0^{(s)}{\omega})(T_i,\Delta_i)-(\mathcal{T}_0{\omega})(T_i,\Delta_i) \overset{\Pr}{\to} 0.
\end{align*}
\end{Proposition}
To better understand the survival Stein operator, we interpret the proposed Stein operator by making connections to the Stein operator used in the uncensored case.

A careful computation  gives the following equivalent expression for the expectation of $(\mathcal{T}_0^{(s)}\omega)(T,\Delta)$ for $(T,\Delta)\sim \mu_X$:
\begin{align*}
&\E_X[(\mathcal{T}_0^{(s)}\omega)(T,\Delta)]=\E_X\left[\omega(T)\Delta\left(\log \frac{f_0(T)}{f_X(T)}\right)'\right] \\
&-\E_X\left[\omega(T)(1-\Delta)(\lambda_0-\lambda_X)(T)\right] +\omega(0)(\lambda_0-\lambda_X)(0).
\end{align*}
Here, we can relate the first expectation to uncensored observations: $\Delta=1$; the second expectation to censored observations: $\Delta=0$; and the third term describes a shift due to boundary conditions.

The expectation of the uncensored part is equal to
\begin{align*}
&\int_0^\infty\omega(x)\left(\log \frac{f_0(x)}{f_X(x)}\right)'S_C(x)f_X(x)dx,
\end{align*} 
which is analogous to what we obtain in the uncensored case, with an additional $S_C$ weighting. If we have no censoring, then $S_C \equiv 1$, recovering the expression found by \citet{chwialkowski2016kernel}. On the other hand, the expectation of the censored part is equal to
\begin{align*}
&\int_0^\infty\omega(x)\left(\frac{S_X(x)}{S_0(x)}f_0(x)-f_X(x)\right)f_C(x)dx,
\end{align*}
which measures the discrepancy between $f_0$ and $f_X$ through survival weights, under the measure of censoring $f_C$.  In the absence of censoring, $f_C=0$ a.e., so this term appears due to the censoring variable. Notice that if differences between $f_0$ and $f_X$ occur at times $t$ where $S_C(t)=0$, then no method will detect these differences (the observations at this time are entirely censored).

\subsection{Martingale Stein Operator}

While the previous approach mimics the classic Stein operator, it has similar drawbacks. Similarly to what we observe in the works of \citet{chwialkowski2016kernel} and \citet{liu2016kernelized}, our Stein operator $\mathcal T_0^{(s)}$ requires very strong integrability conditions on the involved distribution functions. In our setting, we find, for example condition c.1 in Section~\ref{sec:gof_theory}, which involves integrals with respect to hazard functions which are known  to satisfy $\int \lambda_0(x)dx = \infty$, leading to a testing procedure with weak theoretical guarantees. While these conditions may hold for some models, it is not hard to find simple examples where they do not hold.

In order to get a more robust test, we exploit a well-known identity in survival analysis, allowing us to deduce a more natural Stein operator. Such an identity is given by
 \begin{align}
  \label{eqn:martingaleidentity1}
     \E_0\left[\Delta \phi(T)-\int_0^{T} \phi(t)\lambda_0(x)dx\right]=0,
 \end{align}
 which holds for any function $\phi$ such that $\E_0(|\phi(T)|)<\infty$ under $\mu_0$ \cite{aalen2008survival}. This equality is derived by using a martingale identity that appears in the derivation of classical estimators in survival analysis (see \Appendix~\ref{app:knwon_identity}).
 
 Assuming $\lambda_0(t)>0$, we replace $\phi = \omega'/\lambda_0$ in  \eqref{eqn:martingaleidentity1} to get
 $$\E_0\left[\Delta\frac{\omega'(T)}{\lambda_0(T)}-(\omega(T)-\omega(0))\right]=0.$$
 Define the martingale Stein Operator $\mathcal T_0^{(m)}:\mathcal H^{(m)} \to \Omega$ as
\begin{align}
 (\mathcal T_0^{(m)}\omega)(x,\delta) = \delta\frac{\omega'(x)}{\lambda_0(x)}-(\omega(x)-\omega(0))   \label{eqn:defiT0M}
\end{align}
where we write  $\mathcal H^{(m)}$ instead of $\mathcal H$ whenever $\mathcal H$ satisfies
 \begin{align}
     \int_{\R^+} \left|\frac{\omega'(x)}{\lambda_0(x)}\right|S_C(x)f_0(x)dx<\infty,\quad \forall  \omega \in \H.\label{eqn:conditionM}
 \end{align}
 From its definition, it is clear that $\E_0[(\mathcal T_0^{(m)} \omega)(T,\Delta)]=0.$
 Note that, by the definition of the hazard functions, condition~\eqref{eqn:conditionM} is equivalent to 
 \begin{align}
     \int_{\R^+} \left|\omega'(x)\right|S_C(x)S_0(x)dx<\infty,\quad \forall  \omega \in \H,
 \end{align}
 which holds true if the kernel is bounded (recall we assume that $\omega' \in \mathcal H$), therefore, compared to $ \mathcal T_0^{(s)} $, the testing procedure associated to $\mathcal T_0^{(m)}$ has very strong theoretical guarantees. Indeed, we observe that condition c.2 in Section~\ref{sec:gof-test}.1 is much simpler to satisfy because, this time, we consider integrals with respect to the inverse of the hazard function.

\textbf{Model-Free Implementation:} Inspired by the test of uniformity via a $F_0$ transformation \cite{fernandez2019kernel}, we transform our data $U_i=F_0(T_i)$ to generate pairs $(U_i,\Delta_i)$. Notice that since $F_0$ is monotone $U_i=F_0(T_i)=\min\{F_0(X_i),F_0(C_i)\}$, thus $\Delta_i$ remains consistent. Under this transformation, testing the null hypothesis is equivalent to test whether $F_0(X_i)$ is distributed as a uniform random variable, thus, in this setting, $\lambda_0=\lambda_{\mathcal{U}}=\frac{1}{1-x}$ and 
$$(\mathcal{T}_0^{(m)} {\omega})(u,\delta)=\delta {\omega'(u)}(1-u) - \omega(u)+\omega(0)$$
for $u=F_0(x)$  (notice that $F_0(0) = 0$). It will be shown in the experiments that this transformation is beneficial in terms of power performance. 
Similarly, we can exploit that $\Lambda_0(X) \sim \textnormal{Exp}(1)$ under the null when the model is described via the cumulative hazard function.

\subsection{Proportional Stein Operator}
In some scenarios, we are interested in the shape of the hazard function up to a multiplicative constant, i.e. $\lambda_0(t) = \gamma\lambda(t)$ where we know $\lambda(t)$ but not the constant $\gamma$. The family indexed by $\gamma$ is called a proportional hazards family and it is one of the key objects of study in Survival Analysis. This object is fundamental because sometimes it is more important to test for qualitative results as ``the hazard rate is growing at a constant speed", rather than obtaining precise values of the hazard function. If we only know $\lambda_X(t)$ up to constant and we can ensure that $\omega(0)\lambda(0) = 0$, then we can define a Stein operator based on unnormalized hazard. 

In order to define our operator, we assume that
\begin{align}
    \int_{\R_+}&|(\omega(x)\lambda_0(x))'|dx<\infty, \quad \text{and} \nonumber\\
    \omega(0)\lambda_0(0)&= \lim_{x\to \infty} \omega(x)\lambda_0(x)=0, \quad \forall  \omega \in \mathcal H.\label{eqn:conditionP}
\end{align}
As usual, we write $\mathcal H^{(p)}$ to indicate that $\mathcal H$ satisfies property \eqref{eqn:conditionP}. Note that for any function $\omega \in \mathcal H^{(p)}$ it holds that
\begin{align}
    \int_0^\infty \frac{(\omega(x)\lambda_0(x))'}{\lambda_0(x)} \lambda_0(x)dx=0. \label{eqn:conditionP1}
\end{align}
The integral above can be estimated  using the Nelson-Aalen estimator \cite{nelson1972theory}, leading to the statistic
$$\frac{1}{n} \sum_{i=1}^n \frac{(\omega(T_i)\lambda_0(T))'}{\lambda_0(T_i)}\frac{\Delta_{i}}{Y(T_i)/n},$$
where  $Y(t) = \sum_{k=1}^n \ind_{\{T_k\geq t\}}$ is the so-called \emph{risk function}, which counts the number of individuals at risk at time $t$. This suggests the following operator 
\begin{align}
   (\widehat{\mathcal T_0}^{(p)} \omega)(x,\delta) = \left( \omega'(x)+\frac{\omega(x)\lambda_0'(x)}{\lambda_0(x)}\right) \frac{\delta}{Y(x)/n}.\label{eqn:defiHT0P}
\end{align}
In the definition above we use the notation $\widehat{\mathcal T_0}^{(p)}$ to indicate that, the function $Y(t)$ depends on all data points, hence $\widehat{\mathcal T_0}^{(p)}$ can be seen as an empirical estimator of a deterministic operator. Indeed, if $(T_i,\Delta_i) \sim \mu_0$, then $\frac{Y(x)}n \to S_C(x)S_0(x)$, which indicates that under the null hypothesis, the operator $\widehat{\mathcal T_0}^{(p)}$ is similar to $\mathcal T_0^{(p)}$, given  by
$$(\mathcal T_0^{(p)} \omega)(x,\delta) =  \left(\omega'(x)+\frac{\omega(x)\lambda_0'(x)}{\lambda_0(x)}\right)\frac{\delta}{S_C(x)S_X(x)}.$$

This operator cannot be directly evaluated since we do not have access to $S_C$. The following proposition establishes the formal relation between $\widehat{\mathcal T_0}^{(p)}$ and $\mathcal T_0^{(p)}$.
\begin{Proposition}\label{Prop:Tp_identity}
Let $(T_i,\Delta_i) \sim \mu_0$, then for every $\omega \in \mathcal H^{(p)}.$
\begin{align}
\frac{1}{n}\sum_{i=1}^n(\widehat{\mathcal{T}}_0^{(p)}\omega)(T_i,\Delta_i)
\overset{\Pr}{\to} \E_0\left[(\mathcal{T}_0^{(p)}\omega)(T_1,\Delta_1)\right]=0.
\end{align}
\end{Proposition}

\section{Censored-Data Kernel Stein Discrepancy} \label{sec:CKSD}
In this section, we derive censored-data Kernel Stein Discrepancies (c-KSD) using each of our three  Stein operators defined in the previous section. The idea is to compare the largest discrepancy between two distributions $f_X$ and $f_0$ over a class of test functions in the RKHS $\mathcal H$. Since we have access to censored data, we compare $f_X$ and $f_0$ through the measures $\mu_X$ and $\mu_0$, defined in Section~\ref{sec:background}.

We proceed to defined three censored-data kernel Stein discrepancies: the Survival Kernel Stein Discrepancy ($\operatorname{s-KSD}$), the Martingale Kernel Stein Discrepancy ($\operatorname{m-KSD}$), and the Proportional Kernel Stein Discrepancy ($\operatorname{p-KSD}$) based on the respective Stein operators $\mathcal T_0^{(s)}$, $\mathcal T_0^{(m)}$ and $\widehat{ \mathcal T_0}^{(p)}$. In general, for any given Stein operator $\mathcal T_0^{(c)}:\mathcal H^{(c)} \to \Omega$ we define the $\operatorname{c-KSD}$ as
\begin{align*}
\operatorname{c-KSD}(f_X\|f_0) &= \sup_{\omega \in B_1(\mathcal H^{(c)})}\E_{X}[({\mathcal T_0}^{(c)}  \omega)(T,\Delta)].
\end{align*}
Denote by $K^{(c)}$ the reproducing kernel of $\mathcal H^{(c)}$. By using this kernel we can get a close-form expression for $\operatorname{c-KSD}$: For any of the operators $\mathcal T_0^{(c)}$, we define the application of $\mathcal T_0^{(c)}$ on $K^{(c)}(x,\cdot)$ as a function $\R_+ \to \R$ which is defined as  $(\mathcal T_0^{(c)} \omega)(x,\delta)$ 
 but replacing $\omega(x)$ by $K^{(c)}(x,\cdot)$ and $\omega'(x)$ by $\frac{\partial }{\partial x}K^{(c)}(x,\cdot)$.
For example, for $c=m$, we get that $\left[\mathcal (T_0^{(m)}K^{(m)})(x,\delta)\right](\cdot)$ equals
 $$ \frac{\delta}{\lambda_0(x)}\left(\frac{\partial }{\partial x}K^{(m)}(x,\cdot)\right)-(K^{(m)}(x,\cdot)-K^{(m)}(0,\cdot)), $$
 which the reader should compare with equation~\eqref{eqn:defiT0M}.

Recall that for $c \in \{s,m,p\}$, we assumed that if $\omega \in \mathcal H^{(c)}$ then $\omega' \in \mathcal H^{(c)}$, and thus $\xi^{(c)}(x,\delta)(\cdot)  = \left[(\mathcal T^{(c)} K^{(c)})(x,\delta)\right](\cdot) \in \mathcal H^{(c)}$ since all operators involve $\omega$ or $\omega'$. Define the Stein kernel $h^{(c)}:(\R_+\times \{0,1\})^2 \to \R$ by
\begin{align*}
  h^{(c)}((x,\delta),(x',\delta')) = \langle\xi^{(c)}(x,\delta) ,\xi^{(c)}(x',\delta')  \rangle_{\mathcal H^{(c)}}  
\end{align*}
The following proposition gives a closed form for the kernel Stein discrepancies $\operatorname{c-KSD}(f_X\| f_0).$

\begin{Proposition}\label{Prop:cKSDderive}
For $c \in \{s,m,p\}$, and let $(T,\Delta)$ and $(T',\Delta')$ be independent samples from $\mu_X$, and suppose that
\begin{align}
    \E_X\left[\sqrt{h^{(c)}((T,\Delta),(T,\Delta))}\right]<\infty,\label{eqn:condBotchner}
\end{align}
then
\begin{align*}
    \left(\operatorname{c-KSD}(f_X\|f_0)\right)^2= \E_X\left[h^{(c)}((T,\Delta),(T',\Delta')) \right].
\end{align*}
\end{Proposition}
Detailed forms and the derivation for Stein kernels $h^{(c)}((x,\delta),(x',\delta'))$ can be found in \Appendix~\ref{sec:explih}.

\section{Goodness-of-fit Test via $\operatorname{c-KSD}$}\label{sec:gof-test}

In this section, we study goodness-of-fit testing procedures based on $\operatorname{c-KSD}$. We begin by estimating  $\operatorname{c-KSD}^2$ using
\begin{align*}
    \widehat{\operatorname{c-KSD}^2}(f_X \| f_0)
    &= \frac{1}{n^2}\sum_{i=1}^n\sum_{j=1}^n h^{(c)}((T_i,\Delta_i),(T_j,\Delta_j))
\end{align*}
where $(T_i,\Delta_i)$ are independent samples from $\mu_X$. By construction, under the null hypothesis, the estimator above should be close to zero, while under the alternative we expect it to be separated from zero.

\subsection{Theoretical Analysis}\label{sec:gof_theory}
We state some technical conditions that feature our analysis in order to establish  the asymptotic behavior of $\widehat{\text{c-KSD}^2}$.

\subsubsection*{Technical Conditions}
\textbf{a) \emph{Reproducing kernel conditions:}}  We assume that $K$ has continuous second-order  derivatives, and that $K(x,y)$ and $\frac{\partial^2}{\partial x\partial y}K(x,y)$ are bounded and $c_0$-universal kernels.\\
\newline
\textbf{b) \emph{Boundary condition:}} $\lim_{x\to 0+}\sqrt{K(x,x)}\lambda_0(x)<\infty.$\\
\newline
\textbf{c) \emph{Null integrability conditions:}}  Let $(T,\Delta),(T',\Delta')\overset{i.i.d.}{\sim}\mu_0$, and recall that $\E_0=\E_{\mu_0}$. Depending on $c\in\{s,m,p\}$, we assume: 
\begin{itemize}
    \item[1)] $\operatorname{s-KSD}$: 
    \begin{itemize}
        \item[i)]$\E_0[\phi(T,\Delta)^2 |K(T,T)|]<\infty$, and 
        \item[ii)]$\E_0[\phi(T,\Delta)^2\phi(T',\Delta')^2K(T,T')^2]<\infty$,
    \end{itemize}
    where  $\phi(x,\delta) = \delta\frac{\lambda_0'(x)}{\lambda_0(x)}-\lambda_0(x)$.
    \item[2)]  $\operatorname{m-KSD}$: 
        \begin{itemize}
        \item[i)] $\E_0\left[\frac{|K^\star(T,T)|\Delta}{\lambda_0(T)^2}\right]<\infty$, and 
        \item[ii)]$\E_0\left[\frac{K^\star(T,T')^2\Delta\Delta'}{\lambda_0(T)^2\lambda_0(T')^2}\right]<\infty$,
        \end{itemize}
        where $K^\star(x,y) = \frac{\partial^2 }{\partial x \partial y } K(x,y)$.
    \item[3)]  $\operatorname{p-KSD}$: 
            \begin{itemize}
        \item[i)] $\E_0\left[\frac{|K^\star(T,T)|\Delta}{(f_0(T)S_C(T))^2}\right]<\infty$, and
        \item[ii)]
        $\E_0\left[\frac{K^\star(T,T')^2\Delta\Delta'}{(f_0(T)f_0(T')S_C(T)S_C(T'))^2}\right]<\infty$,
    \end{itemize}
    where $K^\star(x,y) = \left(\frac{\partial^2}{\partial x \partial y} K(x,y)\lambda_0(x)\lambda_0(y)\right)$.
\end{itemize}

\textbf{d) \emph{Alternative integrability conditions:}} Let $(T,\Delta)\sim\mu_X$. Then, for each $c\in\{s,m,p\}$ we assume: 
\begin{itemize}
    \item[1)]  $\operatorname{s-KSD}$: 
    \begin{itemize}
        \item[i)] $\E_X[\phi(T,\Delta)^2|K(T,T)|]<\infty$,
    \end{itemize}
    where $\phi(x,\delta)=\delta \frac{\lambda_0'(x)}{\lambda_0(x)}-\lambda_0(x)$.
        \item[2)]  $\operatorname{m-KSD}$: 
                \begin{itemize}
        \item[i)] $\E_X\left[\frac{|K^\star(T,T)|\Delta}{\lambda_0(T)^2}\right]<\infty$,
    \end{itemize}
    where $K^\star(x,y) = \frac{\partial^2}{\partial x \partial y} K(x,y)$.
        \item[3)]  $\operatorname{p-KSD}$:
            \begin{itemize}
        \item[i)] $\E_X\left[\frac{|K^\star(T,T)|\Delta}{S_T(T)^2\lambda_0(T)^2}\right]<\infty$,
 \end{itemize}
 where $K^\star(x,y) = \left(\frac{\partial^2}{\partial x \partial y} K(x,y)\lambda_0(x)\lambda_0(y)\right)$.
\end{itemize} 

The following theorem establishes consistency of our empirical kernel Stein discrepancies to their population versions.

\begin{Theorem}\label{thm:asymptotic_h1}[Asymptotics under the alternative $H_1$]
Let  $c \in \{s,m,p\}$, and suppose that  $f_X$ satisfies conditions a),  b), and the corresponding condition d).
Then it holds 
$$\left(\widehat{\operatorname{c-KSD}}(f_X \| f_0)\right)^2 \overset{\Pr}{\to} \left(\operatorname{c-KSD}(f_X \| f_0)\right)^2.$$
\end{Theorem}
The previous theorem is not enough to ensure  good behavior under the alternative as we need to be sure that the discrepancy of two different distribution functions $f_X$ and $f_0$ is different from 0 (regardless of censoring). We can prove this for $\operatorname{c-KSD}$ for $c\in \{s,m\}$. This does not hold true for $\operatorname{p-KSD}$ since it is designed to test if the hazard function $\lambda_X$ is proportional to $\lambda_0$, and not for goodness-of-fit testing purposes. Indeed, whenever the hazards are in a proportional relation, $\operatorname{p-KSD}$ is 0.

\begin{Theorem}\label{thm:injective}
Let $c\in\{s,m\}$. Assume $S_C(x)=0$ implies $S_X(x)=0$ and that $K$ is $c_0$-universal. Then, under Conditions $a)$, $b)$ and $d)$, $f_0\neq f_X$ implies $\operatorname{c-KSD}(f_0\|f_X)>0$. 
\end{Theorem}

Under the null distribution, $f_X = f_0$, we also have that $\widehat{\operatorname{c-KSD}}(f_0 \| f_0)\to 0$, but we can prove an even stronger result that follows from the theory of $V$-statistics.

\begin{Theorem}[Asymptotics under the null $H_0$]\label{thm:h_0}
Let  $c \in \{s,m,p\}$, and suppose that $f_X=f_0$ and that conditions a),  b), and the corresponding condition c) are satisfied. Then 
$$n\left(\widehat{\operatorname{c-KSD}}(f_X \| f_0)\right)^2 \overset{\mathcal D}{\to} r_c+\mathcal{Y}_c.$$
where $r_c$ is a constant and $\mathcal{Y}_c$ is an infinite sum of independent $\chi^2$ random variables.
\end{Theorem}

While Theorem \ref{thm:h_0} ensures the existence of a limiting null distribution, which implies that a rejection region for the test is well-defined, in practice it is very hard to approximate the limit distribution and the corresponding rejection regions, for which, we rely on a wild bootstrap approach.

We remark that we can obtain concentrations bounds for the test-statistics under the null hypothesis if we assume that the kernels $h^{(s)}$ and $h^{(m)}$ are bounded, by using standard methods. Obtaining concentration bounds for $h^{(p)}$ is harder as it is a random kernel, depending on all data points.

\subsection{Wild Bootstrap Tests}\label{sec:wild}
To  resample from the null distribution we use the wild bootstrap technique \cite{dehling1994random}. This technique is quite generic and it can be applied to any kernel. 

The Wild Bootstrap estimator is given by
\begin{align}\label{eqn:WBestimator}
\frac{1}{n^2}\sum_{i=1}^n\sum_{j=1}^n W_iW_j h^{(c)}((T_i,\Delta_i),(T_j,\Delta_j)),
\end{align}
where $W_1,\ldots,W_n$ are  independent random variables from a common distribution $\mathcal W$ with $\E(W_1) = 0$ and $\Var(W_1) = 1$. In our experiments we consider $W_i$ sampled from a Rademacher distribution, but any distribution with the properties above is suitable.  \citet{dehling1994random} proved that if the limit distribution exists (in the sense of Theorem~\ref{thm:h_0}), then the wild-bootstrap statistic also converges to the same limit distribution.

The testing procedure for goodness-of-fit is performed as follows: \textbf{1)} Set a type 1 error $\alpha\in (0,1)$. \textbf{2)} Compute  $\widehat{\operatorname{c-KSD}^2}(f_X \| f_0)$
using our $n$ data points. \textbf{3)} Compute $m$-independent copies of the Wild Bootstrap estimator ~\eqref{eqn:WBestimator}. \textbf{4)} Compute the proportion of wild bootstrap samples that are larger than $\widehat{\operatorname{c-KSD}^2}(f_X \| f_0)$; if such a proportion is smaller than $\alpha$ we reject the null hypothesis, otherwise the do not reject it.

\begin{figure*}[ht]
\centering
\includegraphics[width=0.9\textwidth]{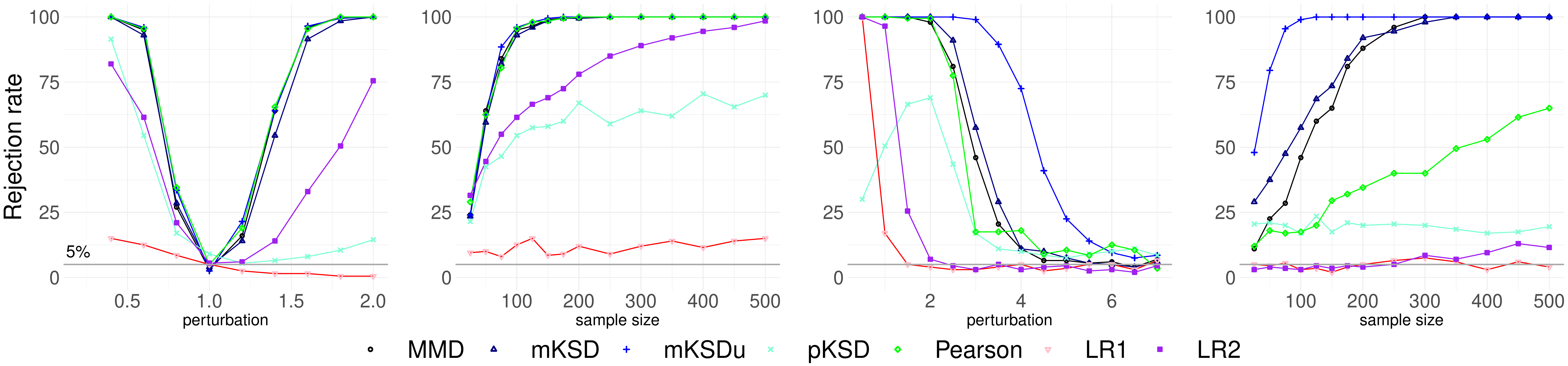}
\vspace{-0.4cm}
\caption{Rejection rate w.r.t. sample size and model perturbation. Left two for Weibull Hazard; Right two for Periodic Hazard. $\alpha=0.01$.
}
        \label{fig:Power_overall}
\end{figure*}

\begin{figure*}[ht]
\centering
    \includegraphics[scale=0.18]{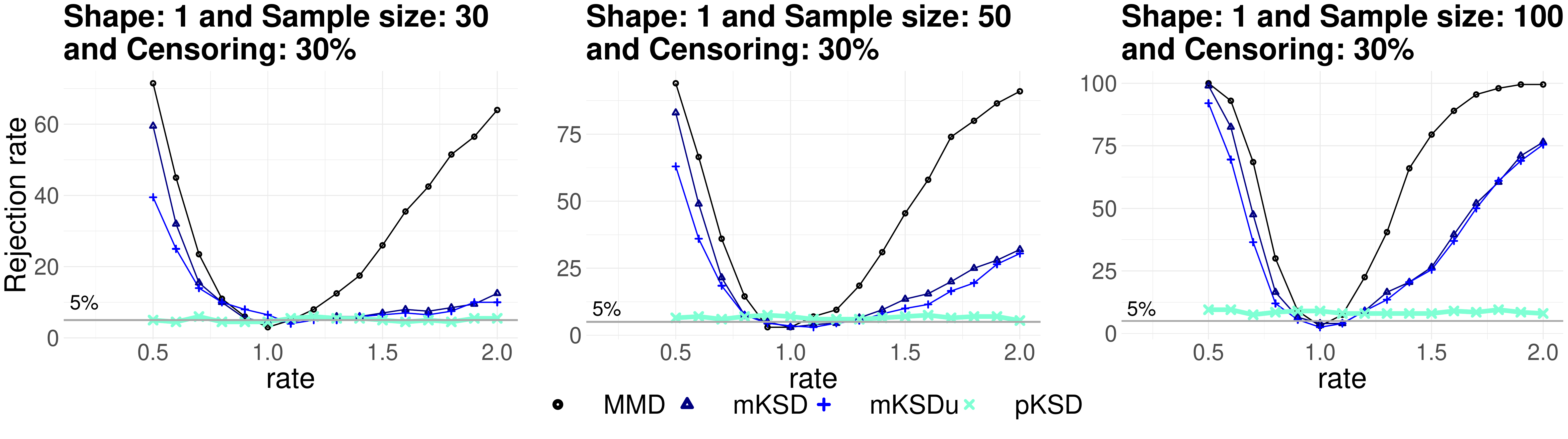}
    \caption{Rejection rate for a proportional class model. As expected, the proportional KSD test does not reject the null for different rates as all the alternatives belong to the same proportional family. }\label{Fig:3}
\end{figure*}



\section{Experiments and  Results}\label{sec:exp}

\paragraph{Proposed approaches:}
In our experiments, we denote by \textbf{mKSD} and by \textbf{pKSD}, the tests based on the martingale and the proportional kernel Stein discrepancies described in Section~\ref{sec: stein operator}, implemented using the  Wild bootstrap approach as described in Section~\ref{sec:wild}. 
In all our experiments we choose the null as an exponential distribution of rate 1, and in this case we can check that sKSD and mKSD coincide. Additionally, we implement \textbf{mKSDu}, which is given by the test \textbf{mKSD} applied to the transformed data $((F_0(T_i),\Delta_i))_{i=1}^n$ to test $H_0:F_0(X)\sim\mathcal{U}(0,1)$. Finally, we use an Gaussian kernel with length-scale chosen by using the median-heuristic, which is the median of all the absolute differences between two different data points. We did not perform any further optimization to improve the performance of the tests.

\paragraph{Competing Approaches:}
\textbf{MMD} denotes the maximum-mean-discrepancy approach proposed by \citet{fernandez2019maximum}, which provides state-of-the-art results,
\textbf{Pearson} denotes the Pearson-type goodness-of-fit test proposed by \citet{akritas1988pearson}, which is quite competitive. 
\textbf{LR1} and \textbf{LR2} denote the  weighted log-rank tests with respective weights functions $w_1(t)=1$ and $w_2(t) = \sum_{i=1}^n \ind_{\{T_i \geq t\}}$, which are classical tests, but not very competitive except for some very simple settings (e.g. testing  $H_0:\lambda_0(t) = 1$ against $\lambda_X(t) = c$, for $c\neq 1$).

\subsection*{Simulated experiments}

\paragraph{Data Setting} We begin by studying our method in a simulated environment where we can control all the possible parameters. We consider two data scenarios.

\textbf{1. Weibull hazard functions:} In our first experiment, we consider the Weibull model, which is commonly used in Survival Analysis \cite{bradburn2003survival}. The Weibull distribution is characterized by the density function $f(x; k,  r) = kr\left(rx\right)^{k-1} \exp\{-(rx)^k\}$, where $k$ and $r$ denote shape and rate parameters, respectively. 
\textbf{2. Periodic hazard functions:} A much more interesting scenario is the so-called periodic hazards, which are used to describe, for example, seasonal diseases such as Influenza.  In this example, we consider the hazard function $\lambda_X(x) = 1 - \cos{(\theta \pi x)}$ studied by \citet{fernandez2019maximum}. Note that when $\theta \to \infty$, then the distribution tends to a exponential of parameter $1$. See Figure \ref{fig:dist_demo} for a comparison between the models.

For both models, we investigate the performance of our test in two setting: \emph{perturbations from the null} and \emph{increasing sample size}, which we proceed to explain. \textbf{Perturbations from the null:} In this experiment, we investigate how the power  changes for perturbations of the null hypothesis. For the Weibull data, we set $H_0: f_0(x) = f(x;1,1)$ and consider Weibull alternatives $f_X(x) = f(x;k,1)$ with $k \in (0,\ldots, 2]$.  Notice that we recover the null hypothesis when $k=1$. Also, we consider a constant $30\%$ of censored observations and a fixed sample size of $n=100$. For the periodic experiment  we set $H_0: f_0(x) = e^{-x}$, which is recovered when we take $\theta$ tending to infinity. In this case, we consider alternatives $\theta \in \{1,2,\ldots, 8\}$. We consider, again, a constant $30\%$ of censoring, and a fixed sample size of $n=100$. \textbf{Increasing sample size:} In this scenario, we investigate how the rejection rate of our test increases as the sample size increases. In the Weibull setting we set the null $H_0: f_0(x) = f(x;1,1)$, the alternative as $f_X(x) = f(x;1.5,1),$ and in the periodic setting, we consider the null $H_0:f_0(x) = e^{-x}$, and generate data from the alternative $\theta = 3$. In both settings we consider 30\% of censored data points

\paragraph{Results} We show our results in Figure \ref{fig:Power_overall}. For the Weibull data (first and second plots), observe that all kernel-based methods, except the pSKD, perform very similar to the Pearson test designed to perform extremely well in these types of setting. For the Periodic data (third and fourth plots), the goodness-of-fit problem is much more challenging, and we see differences in the performances of the methods. We observe that the MMD test of \citet{fernandez2019maximum} has a better performance than the Pearson test, as  was suggested by the experiments in their work. Our test mSKD performs slightly better than the the MMD test, whereas mSKDu outperform all the other methods by a huge margin. We can see that it is the most resistant to the increment in the perturbation parameter (third plot), and, for example, for $\theta = 4$, most methods cannot differentiate between null and alternative with large probability, whereas our method has power of around $75 \%$.

\paragraph{Proportionality}

Our results show that pKSD is not a very powerful test. A possible explanation lies in the fact that, since this method tests against a model class, it must ignore all differences within this class, which affects the power of the test. Despite its lower power, it remains the only test out of the proposed methods that can test if our data was generated by a hazard proportional to $\lambda_0$. In Figure \ref{Fig:3}, we consider a Weibull hazard, given by $\lambda_X(x;k,r) = r^kkx^{k-1}$,  with shape $k=1$ and rate $r\in (0,2)$.  Note that changing the parameter $r$ gives the same hazard up to a constant. Figure \ref{Fig:3} shows that for a family of proportional hazards our method reaches the right type 1 error at low sample sizes, while all the other methods have non-trivial power. We observe, however, that for larger sample sizes, the test has a type-1 error that is slightly elevated over the design level. This may occur as the conditions of Theorem~\ref{thm:h_0} are hard to satisfy in general, and have yet to be proven to hold for this case.  Boostraping methods  with  strong theoretical guarantees under broader conditions are the subject of ongoing research.

\subsection*{Real Data Experiments}

\paragraph{Data Sources}
We perform our tests on the following real datasets to check relevant model assumptions.
\textbf{aml:} \tb Acute \tb Myelogenous \tb Leukemia survival dataset \citep{miller2011survival};
\textbf{cgd:} \tb Chronic  \tb Granulotamous \tb Disease dataset
\citep{fleming2011counting};
\textbf{ovarian:} Ovarian Cancer Survival dataset \citep{edmonson1979different};
\textbf{lung:} North Central Cancer Treatment Group (NCCTG) Lung Cancer dataset \citep{loprinzi1994prospective};
\textbf{stanford:} Stanford Heart Transplant Data \citep{crowley1977covariance};
\textbf{nafld:} Non-alcohol fatty liver disease (NAFLD) \citep{allen2018nonalcoholic}.

\paragraph{Test Results}
We apply our proposed tests on real dataset for the Testing hazard proportionality and Goodness-of-fit  settings. First, we check model class assumption using $\textbf{pKSD}$ to test whether the observed data is from a desired family model without fitting model parameters. We check the exponential model class and the Weibull model with shape=2. As the results shown in Table \ref{tab:model_class}, our tests does not reject the Exponential model, which is coherent with scientific domain knowledge from the literature.\footnote{High-grade serous ovarian carcinoma (HG-SOC) is a major cause of cancer-related death. The growth of HG-SOC acts as an indicator of survival time of ovarian cancer \citep{gu2019computational}. This paper also suggests that HG-SOC follows exponential expansion, which implies exponentially distributed survival time of ovarian patient.} 

For the Goodness-of-fit test setting, we fit a cox proportional hazard model from the covariates provided in the datasets. The cox-proportional hazard function has the form $\lambda_X(x_i) =\lambda_b(x_i)\exp(\beta Y_i) $, where $\lambda_b(x)$ is the base hazard and $Y_i$ is the covariate for subject $i$. The procedure is done via spliting the data into training set and test sets. Fitting the cox proportional-hazard model is applied on the training sets and the test sets are used to perform the goodness-of-fit tests. Results in Table \ref{tab:cox_ph} shows that all the models does not reject the fitted cox proportional hazard models and validate the proportional hazard assumptions for relevant fitted models, which is coherent with scientific experience stated in the literature.\footnote{
\citet{chansky2016survival} suggests that cox proportional hazard model is a reasonable tool among practitioners for \textbf{lung} dataset.
\citep{crowley1977covariance} suggests a fit for cox proportional hazard model for \textbf{stanford} dataset.
\citet{allen2018nonalcoholic} states that cox proportional hazards is often used to study the impact of NAFLD on incident metabolic syndrome or death.}

\begin{table}[]
    \centering
    \begin{tabular}{c|ccc}
    \toprule
    {p-value} &  \textbf{aml}  &  \textbf{cgd} &     \textbf{ovarian} \\
    \hline
    Exponential &  0.585 &  0.460 &  0.681 \\
    Weibull: shape=2 &  0.001 &  0.002 &  0.063 \\
    \bottomrule
    \end{tabular}
    \caption{Real data applications on testing hazard proportionality.}
    \label{tab:model_class}
\end{table}

\begin{table}[]
    \centering
    \begin{tabular}{c|c|c}
    \toprule
    Dataset &  Covarites &  p-value \\
    \hline
    \textbf{lung} & Age &  0.167 \\
    \textbf{stanford} &  T5 mismatch score  &  0.594 \\
    \textbf{nafld} &  Weight and Gender &  0.108 \\
    \bottomrule
    \end{tabular}
    \caption{Real data applications on testing goodness of fit}
    \label{tab:cox_ph}
\end{table}

\appendix
\newpage
\section*{Appendix}
\section{Proofs and Derivations}\label{app:proofs}

\subsection{Proofs of Section 3.1: Survival Stein Operator}
\subsubsection{Proof of Proposition \ref{Prop:T0TsExEquiv}}
Let  $\omega \in \mathcal H^{(s)}$. Then

\begin{align}
\E_0((\mathcal T_0\omega)(T,\Delta) -(\mathcal T_0^s\omega)(T,\Delta))&=\E_0\left( \omega(T)\left[\Delta\left(\frac{f_0'(T)}{f_0(T)}-\lambda_C(x) \right) - \left(\Delta\frac{\lambda_0'(T)}{\lambda_0(T)}-\lambda_0(T)\right)\right]\right)\label{eqn:diffExpect1}
\end{align}

Observe that 
\begin{align*}
\E_0\left(\Delta\omega(T)\lambda_C(T)\right)=\int_0^\infty\omega(x)\frac{f_C(x)}{S_C(x)}S_C(x)f_0(x)dx=\int_0^\infty\omega(x)\frac{f_0(x)}{S_0(x)}S_0(x)f_C(x)dx=\E_0((1-\Delta)\omega(T)\lambda_0(T)),
\end{align*}
therefore, the RHS of Equation~\eqref{eqn:diffExpect1} is equal to
\begin{align*}
\E_0\left( \omega(T)\Delta\left(\frac{f_0'(T)}{f_0'(T)}+\lambda_0(T)-\frac{\lambda_0(T)}{\lambda_0(T)} \right)\right).
\end{align*}
Finally, the last expectation is 0 due to the identity $\frac{f_0'(x)}{f_0(x)} = \frac{\lambda_0'(x)}{\lambda_0(x)}-\lambda_0(x)$, which follows from a simple computation. 

\subsubsection{Proof of Proposition \ref{Prop:T0TsExEquivSample}}
By definition,
\begin{align*}
\sup_{\omega\in B_1(\mathcal{H})} \frac{1}n \sum_{i=1}^n (\mathcal{T}_0^{(s)}\omega)(T_i,\Delta_i)-(\mathcal{T}_0\omega)(T_i,\Delta_i)
&=\sup_{\omega\in B_1(\mathcal{H})} \frac{1}n \sum_{i=1}^n  \omega(T_i)\left(\Delta_i\lambda_C(T_i)-(1-\Delta_i)\lambda_0(T_i)\right)\\
&=\sup_{\omega\in B_1(\mathcal{H})} \left\langle \omega,\frac{1}n \sum_{i=1}^n K(T_i,\cdot)\left(\Delta_i\lambda_C(T_i)-(1-\Delta_i)\lambda_0(T_i)\right)\right\rangle_{\mathcal{H}}\\
&=\left\|\frac{1}n \sum_{i=1}^n K(T_i,\cdot)\left(\Delta_i\lambda_C(T_i)-(1-\Delta_i)\lambda_0(T_i)\right)\right\|_{\mathcal{H}}
\end{align*}
We continue by proving that the previous norm converges to zero in probability. Observe that by the symmetrization lemma \cite[Lemma 6.4.2]{vershynin2019high}, it holds 
\begin{align*}
\E\left[\left\|\frac{1}n \sum_{i=1}^n K(T_i,\cdot)\left(\Delta_i\lambda_C(T_i)-(1-\Delta_i)\lambda_0(T_i)\right)\right\|_{\mathcal{H}}\right]
&\leq 2\E\left[\left\|\frac{1}n \sum_{i=1}^n W_i K(T_i,\cdot)\left(\Delta_i\lambda_C(T_i)-(1-\Delta_i)\lambda_0(T_i)\right)\right\|_{\mathcal{H}}\right]
\end{align*}
 where $W_1,\ldots,W_n$ are i.i.d. Rademacher random variables, independent of the data $(T_i,\Delta_i)_{i=1}^n$. Then, by Jensen's inequality, and by using that $\E(W_i)=0$, we conclude that the previous expression converges to zero in probability, as 
\begin{align*}
\E\left[\left\|\frac{1}n \sum_{i=1}^n W_i K(T_i,\cdot)\left(\Delta_i\lambda_C(T_i)-(1-\Delta_i)\lambda_0(T_i)\right)\right\|_{\mathcal{H}}^2\right]
&=\E\left[\frac{1}{n^2}\sum_{i=1}^n K(T_i,T_i)\left(\Delta_i \lambda_C(T_i)-(1-\Delta_i)\lambda_0(T_i)\right)^2\right] \to0,
\end{align*}
a.s., where the limit result holds due the law of large numbers which can be applied under the Condition in Equation~\eqref{eqn:moment1} and since $|K(x,y)|\leq c_1$, as
\begin{align*}
\E\left[K(T_i,T_i)\left(\Delta_i \lambda_C(T_i)-(1-\Delta_i)\lambda_0(T_i) \right)^2\right]
&\leq c_1\E\left[\left(\Delta_i \lambda_C(T_i)^2+(1-\Delta_i)\lambda_0(T_i)^2 \right)\right]\\
&=c_1\int_0^\infty \left(\lambda_C(x)+\lambda_0(x)\right)f_C(x)f_0(x)dx<\infty.
\end{align*}

\subsection{Proofs of Section 3.3: Proportional Stein Operator}

\subsubsection*{Proof of Propositon  \ref{Prop:Tp_identity}}
We start by claiming that the following equation holds true for every  $\omega\in\mathcal{H}^{(s)}$:
\begin{align}
\frac{1}{n}\sum_{i=1}^n\left((\widehat{\mathcal T}_0^{(p)}\omega)(T_i,\Delta_i)-({\mathcal T}_0^{(p)}\omega)(T_i,\Delta_i)\right)\overset{\Prob}{\to} 0.\label{eqn:uniform}
\end{align}
Then, the main result follows from Equation~\eqref{eqn:uniform}, by using the law of large numbers and that
\begin{align*}
\E_0\left[(\mathcal{T}_0^{(p)}\omega)(T_1,\Delta_1)\right]=\int_0^\infty\frac{(\omega(t)\lambda_0(t))'}{\lambda_0(t)}\frac{1}{S_0(t)S_C(t)}S_C(t)f_0(t)dt
&=\int_0^\infty\frac{(\omega(t)\lambda_0(t))'}{\lambda_0(t)}\lambda_0(t)dt=0,
\end{align*}
which follows from the definition of our operator (see Equation~\eqref{eqn:conditionP1}). 

We finish the proof by proving our claim in Equation~\eqref{eqn:uniform}. Observe that
\begin{align}
\left|\frac{1}{n}\sum_{i=1}^n\left((\widehat{\mathcal T}_0^{(p)}\omega)(T_i,\Delta_i)-({\mathcal T}_0^{(p)}\omega)(T_i,\Delta_i)\right)\right|
&\leq\frac{1}{n}\sum_{i=1}^n\frac{|(\omega(T_i)\lambda_0(T_i))'|}{\lambda_0(T_i)}\left|\frac{\Delta_i}{Y(T_i)/n}-\frac{\Delta_i}{S_T(T_i)}\right|,\label{eqn:sumsplit}
\end{align}
where $S_T(t)=S_C(t)S_0(t)$ holds under the null hypothesis. We proceed to prove that the previous sum tends to $0$ in probability when $n$ grows to infinity. Let $\varepsilon>0$ and define $t_{\varepsilon}>0$ as the infimum of all $t$ such that $\int_{t}^{\infty}\left|(\omega(x)\lambda_0(x))'\right|dx<\varepsilon$. Notice that such $t_{\varepsilon}$ is well-defined since  $\int_{0}^{\infty}\left|(\omega(x)\lambda_0(x))'\right|dx<\infty$. We continue by splitting the sum in Equation~\eqref{eqn:sumsplit} into two regions, $\{T_i\leq t_{\epsilon}\}$ and $\{T_i>t_{\epsilon}\}$, obtaining that Equation~\eqref{eqn:sumsplit} equals
\begin{align}
\frac{1}{n}\sum_{i=1}^n\frac{|(\omega(T_i)\lambda_0(T_i))'|}{\lambda_0(T_i)}\left|\frac{\Delta_i}{Y(T_i)/n}-\frac{\Delta_i}{S_T(T_i)}\right|\ind_{\{T_i\leq t_{\varepsilon}\}}+\frac{1}{n}\sum_{i=1}^n\frac{|(\omega(T_i)\lambda_0(T_i))'|}{\lambda_0(T_i)}\left|\frac{\Delta_i}{Y(T_i)/n}-\frac{\Delta_i}{S_T(T_i)}\right|\ind_{\{T_i>t_{\varepsilon}\}},\label{eqn:randomvjjrj294s}
\end{align}
and we prove that both sums tend to 0 in probability when $n$ grows to infinity.  We start with the first term. Observe that
\begin{align*}
\frac{1}{n}\sum_{i=1}^n\frac{|(\omega(T_i)\lambda_0(T_i))'|}{\lambda_0(T_i)}\left|\frac{\Delta_i}{Y(T_i)/n}-\frac{\Delta_i}{S_T(T_i)}\right|\ind_{\{T_i\leq t_{\varepsilon}\}}&\leq \sup_{t\leq t_{\epsilon}}\left|\frac{1}{Y(t)/n}-\frac{1}{S_T(t)}\right|\frac{1}{n}\sum_{i=1}^n\frac{|(\omega(T_i)\lambda_0(T_i))'|}{\lambda_0(T_i)}\Delta_i\ind_{\{T_i\leq t_{\varepsilon}\}}\\
&= o_p(1),
\end{align*}
where the previous result holds since $\sup_{t\leq t_{\epsilon}}\left|\frac{1}{Y(t)/n}-\frac{1}{S_T(t)}\right|\to 0$ almost surely by the Glivenko-Cantelli Theorem,  and since
\begin{align*}
\frac{1}{n}\sum_{i=1}^n\frac{|(\omega(T_i)\lambda_0(T_i))'|}{\lambda_0(T_i)}\Delta_i1_{\{T_i\leq t_{\epsilon}\}}\to \E\left[\frac{|(\omega(T_1)\lambda_0(T_1))'|}{\lambda_0(T_1)}\Delta_1 \ind_{\{T_1\leq t_{\varepsilon}\}} \right]&=\int_0^{t_{\epsilon}}\frac{|(\omega(t)\lambda_0(t))'|}{\lambda_0(t)}S_C(t)f_0(t)dt\\
&=\int_0^{t_{\epsilon}}\left|(\omega(t)\lambda_0(t))'\right|dt<\infty,
\end{align*}
where the last expression is finite due to Equation~\eqref{eqn:conditionP}.

Next, we deal with the second term in equation~\eqref{eqn:randomvjjrj294s}. Theorem 3.2.1. of  \citet{gill1980censoring} yields $\sup_{t\leq \tau_n}\left|1-\frac{Y(T_i)/n}{S_T(T_i)}\right|=O_p(1)$, where $\tau_n = \max\{T_1,\ldots, T_n\}$, and, Lemma 2.7 of  \citet{gill1983large} yields $\sup_{t\leq \tau_n} nS_T(t)/Y(t)= O_p(1)$ (recall that $S_T(t) = S_0(t)S_C(t)$). From the previous results, we get
\begin{align*}
\frac{1}{n}\sum_{i=1}^n\Delta_i\frac{|(\omega(T_i)\lambda_0(T_i))'|}{\lambda_0(T_i)}\left|\frac{1}{Y(T_i)/n}-\frac{1}{S_T(T_i)}\right|\ind_{\{T_i>t_{\varepsilon}\}}
&=\frac{1}{n}\sum_{i=1}^n\Delta_i\frac{|(\omega(T_i)\lambda_0(T_i))'|}{\lambda_0(T_i)}\frac{1}{Y(T_i)/n}\left|1-\frac{Y(T_i)/n}{S_T(T_i)}\right|\ind_{\{T_i>t_{\varepsilon}\}}\\
&=O_p(1)\frac{1}{n}\sum_{i=1}^n\Delta_i\frac{|(\omega(T_i)\lambda_0(T_i))'|}{\lambda_0(T_i)}\frac{1}{Y(T_i)/n}\ind_{\{T_i>t_{\varepsilon}\}}\\
&=O_p(1)\frac{1}{n}\sum_{i=1}^n\Delta_i\frac{|(\omega(T_i)\lambda_0(T_i))'|}{\lambda_0(T_i)}\frac{1}{S_0(T_i)S_C(T_i)}\ind_{\{T_i>t_{\varepsilon}\}}.
\end{align*}
Now, notice that 
\begin{align*}
\frac{1}{n}\sum_{i=1}^n\Delta_i\frac{|(\omega(T_i)\lambda_0(T_i))'|}{\lambda_0(T_i)}\frac{1}{S_0(T_i)S_C(T_i)}\ind_{\{T_i>t_{\varepsilon}\}}
&\overset{a.s.}{\to} \E_0\left[\Delta_1\frac{|(\omega(T_1)\lambda_0(T_1))'|}{\lambda_0(T_1)}\frac{1}{S_0(T_1)S_C(T_1)}\ind_{\{T_1>t_{\varepsilon}\}} \right]\nonumber\\
&=\int_{t_\varepsilon}^{\infty} \frac{|(\omega(x)\lambda_0(x))'|}{\lambda_0(x)}\frac{f_0(x)S_C(x)}{S_0(x)S_C(x)}dx\nonumber\\
&=\int_{t_\varepsilon}^{\infty} |(\omega(x)\lambda_0(x))'|dx<\varepsilon,
\end{align*}
where the first equality holds by Equation \eqref{eqn:expectedPhi}, and the last inequality comes from the definition of $t_{\varepsilon}$. Since  we can choose $\varepsilon>0$ as small as desired, we conclude the result.

\subsection{Proofs Section 4: Censored-Data Kernel Stein Discrepancy}

\subsubsection{Proof of Proposition \ref{Prop:cKSDderive}}
\begin{proof}
Notice that, by the definition of the random function $\xi^{(c)}(\Delta,T)$, we have that $ (\mathcal T^{(c)}\omega)(T, \Delta)=\langle \omega , \xi^{(c)}(T,\Delta)\rangle_{\mathcal H^{(c)}} $. Also notice that, $\xi^{(c)}(x,\delta) \in \mathcal H^{(c)}$ for each fixed $(x,\delta)$, and that the expectation, $\E_X\left[\xi^{(c)}(T,\Delta)\right]\in \mathcal{H}^{(c)}$ if and only if equation~\eqref{eqn:condBotchner} is satisfied (the previous expectation has to be understood in the Bochner sense, as we are taking expectation of a random function).

Then,
\begin{align*}
\operatorname{c-KSD}(f_X\|f_0)^2=\sup_{\omega\in B_1(\mathcal{H}^{(c)})}\E_X\left[(\mathcal T_0^{(c)}\omega)(T,\Delta)\right]^2&=\sup_{\omega\in B_1(\mathcal{H}^{(c)})}\E_{X}\left[\left\langle \omega , \xi^{(c)}(T,\Delta)\right\rangle_{\mathcal H^{(c)}}\right] ^2\\
&= \sup_{\omega\in B_1(\mathcal{H}^{(c)})}\left\langle \omega , \E_{X}\left[\xi^{(c)}(T,\Delta)\right]\right\rangle^2_{\mathcal H^{(c)}}\\
&= \left\|\E_X\left[\xi^{(c)}(T,\Delta)\right]\right\|^2_{\mathcal H^{(c)}}\\
&= \left \langle\E_X\left[\xi^{(c)}(T,\Delta)\right] ,\E_X\left[\xi^{(c)}(T',\Delta')\right] \right \rangle_{\mathcal H^{(c)}}\nonumber\\
&=  \E_X\left[\left \langle \xi^{(c)}(T,\Delta),\xi^{(c)}(T',\Delta') \right \rangle_{\mathcal H^{(c)}} \right] \\
&= \E_X\left[h^{(c)}((T,\Delta),(T',\Delta'))\right],
\end{align*}
where the third equality is due to the linearity  of expectation and the inner product, the fourth equality follows from the definition of norm (and since we are taking supremum in the unit ball), and the second to last equality is, again, due to the linearity of the expectation and inner product.
\end{proof}
\subsubsection{Explicit computation of $h^{(c)}$}\label{sec:explih}
Denote  $\phi(x,\delta)=\delta \frac{\lambda_0'(x)}{\lambda_0(x)}-\lambda_0(x)$, and $L_1(x,y)= \frac{\partial}{\partial x}K^{(c)}(x,y)$, $L_2(x,y) = \frac{\partial}{\partial y}K^{(c)}(x,y)$ and  $L=\frac{\partial^2}{\partial x \partial y} K^{(c)}(x,y)$. For simplicity of exposition, we will drop the superindex $(c)$ in all cases.

\paragraph{Survival Stein operator} $(c=s)$: 
For this case, we have
\begin{align*}
\xi(x,\delta)=(\mathcal{T}_0K)((x,\delta),\cdot)&=\delta \frac{\partial}{\partial x}K(x,\cdot)+\left(\delta \frac{\lambda_0'(x)}{\lambda_0(x)}-\lambda_0(x)\right)K(x,\cdot)+\lambda_0(0)K(0,\cdot)\\
             &=\delta L_1(x,\cdot)+\phi(x,\delta)K(x,\cdot)+\lambda_0(0)K(0,\cdot).
\end{align*}
Notice that a simple computation shows that $L(x,y)=\left\langle  L_1(x,\cdot),L_1(y,\cdot)\right\rangle_{\mathcal{H}}$, then
\begin{align*}
h^{(s)}((x,\delta),(x',\delta'))
&=\delta\delta'
L(x,x')+\delta\phi(x',\delta')L_1(x,x')+\delta\lambda_0(0)L_1(x,0)\\
 &\quad+\phi(x,\delta)\delta'L_2(x,x')+\phi(x,\delta)\phi(x',\delta')K(x,x')+\phi(x,\delta)\lambda_0(0)K(x,0)\\
 &\quad+\lambda_0(0)\delta'L_2(0,x')+\lambda_0(0)\phi(x',\delta')K(0,x')+\lambda_0(0)^2K(0,0).
\end{align*}

\paragraph{Martingale Stein operator} $(c=m)$: Observe that in this case
\begin{align*}
\xi(x,\delta)&=(\mathcal T_0 K)((s,\delta),\cdot)= \frac{\delta}{\lambda_0(x)}L_1(x,\cdot)- K(x,\cdot) + K(0,\cdot).
\end{align*}
Then, by the  reproducing kernel property
\begin{align*}
h^{(m)}(x,\delta),(x',\delta'))
&=\frac{\delta}{\lambda_0(x)}\frac{\delta'}{\lambda_0(x')}
L(x,x') - \frac{\delta}{\lambda_0(x)}L_1(x,x') + \frac{\delta}{\lambda_0(x)}L_1(x,0)\\
 &\quad - \frac{\delta'}{\lambda_0(x')}L_2(x,x')+ K(x,x') - K(x,0)\\
 &\quad + \frac{\delta'}{\lambda_0(x')}L_2(0,x')- K(0,x') + K(0,0).
\end{align*}

 \paragraph{Proportional Stein operator} $(c=p)$: Notice that, in this case, we use $\widehat{\mathcal T}_0^{(p)}$, given in Equation (16),  to compute $\widehat\xi^{(p)}(x,\delta)=(\widehat{\mathcal{T}}_0^{(p)}K^{(p)})((x,\delta),\cdot)$ since $\mathcal{T}_0^{(p)}$ is not available, as it depends on $S_C$, which is unknown even under the null hypothesis. Then,
\begin{align*}
\widehat\xi(x,\delta)&=(\widehat{\mathcal{T}}_0K)((x,\delta),\cdot)=\left(L_1(x,\cdot)+\frac{\lambda_0'(x)}{\lambda_0(x)}K(x,\cdot)\right)\frac{\delta}{Y(x)/n}.
\end{align*}
 Define $K^\star(x,y) = \left(\frac{\partial^2}{\partial x \partial y} \lambda_0(x)\lambda_0(y)K(x,y)\right)$. Then, by the reproducing kernel property,
\begin{align*}
\widehat h^{(p)}((x,\delta),(x',\delta'))
&= n^2\frac{\delta \delta'}{Y(x)Y(x')}K^\star(x,x').
\end{align*}
Recall that $Y(t) = \sum_{k=1}^n \ind_{\{T_k\geq t\}}$ denotes the risk function, which depends on all the data points, hence we write $\widehat{h}^{(p)}$ to recall the reader that this kernel is a random one.

\subsection{Proofs of Section 5: Goodness-of-fit via c-KSD} 
The following lemmas show that, under Conditions c) and d) (depending on which case), the kernels $h^{(c)}$ have finite first and second moment.  These  moment conditions on the kernel are important to deduce asymptotic results.

\begin{lemma}\label{lemma:conditiondkernels}
Let $(T',\Delta')$ and  $(T,\Delta)$ be independent samples from $\mu_X$, and assume that Condition $d)$ holds. Then, 
\begin{align*}
    \E_X\left[|h^{(c)}((T,\Delta),(T,\Delta))|\right]<\infty,\quad\text{and}\quad     \E_X\left[|h^{(c)}((T,\Delta),(T',\Delta'))|\right]<\infty
\end{align*}
for $c \in \{s,m,p\}$, under the alternative hypothesis.
\end{lemma}

\begin{lemma}\label{lemma:conditiondkernels2}
Let $(T',\Delta')$ and  $(T,\Delta)$ be independent samples from $\mu_0$, and assume that Condition $c)$ holds. Then 
\begin{align*}
    \E_0\left[|h^{(c)}((T,\Delta),(T,\Delta))|\right]<\infty,\quad\text{and}\quad    \E_0\left[h^{(c)}((T,\Delta),(T',\Delta'))^2\right]<\infty
\end{align*}
for $c \in \{s,m,p\}$, under the null hypothesis.
\end{lemma}
We just proof Lemma~\ref{lemma:conditiondkernels} since the proof of Lemma~\ref{lemma:conditiondkernels2} is essentially the same.

\begin{proof}[Proof of Lemma~\ref{lemma:conditiondkernels}]
First of all, note that for any kernel (positive-definite function), it holds
\begin{align*}
h^{(c)}((x,\delta),(x',\delta'))\leq \frac{1}{2}h^{(c)}((x,\delta),(x,\delta))+\frac{1}{2}h^{(c)}((x',\delta'),(x',\delta')),
\end{align*}
hence, it is enough to only prove the first part of the lemma.

\item \textbf{Survival Stein operator} $(c=s)$:
Recall $\xi^{(s)}(x,\delta)=\delta L_1(x,\cdot)+\phi(x,\delta)K(x,\cdot)+\lambda_0(0)K(0,\cdot)$, where $L_1(x,y)=\frac{\partial}{\partial x}K(x,y)$ $\phi(x,\delta)=\delta \frac{\lambda_0'(x)}{\lambda_0(x)}-\lambda_0(x)$, then
\begin{align*}
\E_X\left[|h^{(s)}((T,\Delta),(T,\Delta))|\right]
&=\E_X\left[\left\|\xi^{(s)}(T,\Delta)\right\|_{\mathcal{H}^{(s)}}^2\right]\\
&\leq 4\E_X\left[\left\|\Delta L_1(T,\cdot)\right\|_{\mathcal{H}^{(s)}}^2+\left\|\phi(T,\Delta)K(T,\cdot)\right\|_{\mathcal{H}^{(s)}}^2\right] +4\left\|\lambda_0(0)K(0,\cdot)\right\|_{\mathcal{H}^{(s)}}^2\\
&\leq 4\E_X\left[\left\|\Delta L_1(T,\cdot)\right\|_{\mathcal{H}^{(s)}}^2\right]+4\E_X\left[\left\|\phi(T,\Delta)K(T,\cdot)\right\|_{\mathcal{H}^{(s)}}^2\right]+4\lambda_0(0)^2K(0,0).
\end{align*}

The first and third term in the previous equation are finite under the technical Conditions a) and b). Thus, we only need to check
\begin{align*}
\E_X\left[\left\|\phi(T,\Delta)K(T,\cdot)\right\|_{\mathcal{H}^{(s)}}^2\right]=\E_X\left[\phi(T,\Delta)^2|K(T,T)|\right]<\infty,
\end{align*}
which is guaranteed by Condition d).

\item \textbf{Martingale Stein operator} $(c=m)$: Recall that $\xi^{(m)}(x,\delta)
=\phi(x,\delta)L_1(x,\cdot)- K(x,\cdot) + K(0,\cdot)$, where $L_1(x,y)=\frac{\partial}{\partial x}K(x,y)$ and $\phi(x,\delta)=\frac{\delta}{\lambda_0(x)}$. Then
\begin{align*}
\E_X\left[|h^{(m)}((T,\Delta),(T,\Delta))|\right]
&=\E_X\left[\left\|\xi^{(m)}(T,\Delta)\right\|_{\mathcal{H}^{(m)}}^2\right]\\
&\leq 4\E_X\left[\left\|\phi(T,\Delta)L_1(T,\cdot)\right\|_{\mathcal{H}^{(s)}}^2\right]+4\E\left[\left\|K(T,\cdot)\right\|_{\mathcal{H}^{(s)}}^2\right]+4\left\|K(0,\cdot)\right\|_{\mathcal{H}^{(s)}}^2.
\end{align*}
Observe that the second and third term are finite under Condition a). Additionally, define $L(x,y)=\frac{\partial^2}{\partial x\partial y}K(x,y)$ and notice that
\begin{align*}
\E_X\left[\left\|\phi(T,\Delta)L_1(T,\cdot)\right\|_{\mathcal{H}^{(s)}}^2\right]
&=\E_X\left[\phi(T,\Delta)^2L(T,T)\right]=\E_X\left[\frac{\Delta}{\lambda_0(T)^2}L(T,T)\right]<\infty
\end{align*}
holds under Condition d) (Notice that $L=K^\star$ in Condition d.2)).

\item \textbf{Proportional Stein operator ($c=p$)}. This case follows directly from Condition d.3).
\end{proof}

\subsubsection{Proof of Theorem \ref{thm:asymptotic_h1}}
We distinguish between two cases: first, when $h^{(c)}$ is a deterministic kernel (that is $c \in \{s, m\}$), and second, when $\widehat h^{(c)}$ is a random kernel, meaning $c = p$. 

\paragraph{Deterministic kernel $(c\in\{s,m\})$:}
For the first case, we have
\begin{align*}
\widehat{\operatorname{c-KSD}}^2(f_X||f_0)=\frac{1}{n^2}\sum_{i=1}^n\sum_{j=1}^n h^{(c)}((T_i,\Delta_i),(T_j,\Delta_j)),
\end{align*}
which is a V-statistic of order 2. Thus, by using the law of large numbers for V-statistics, we deduce
\begin{align*}
\widehat{\operatorname{c-KSD}}^2(f_X||f_0)\overset{a.s.}{\to}\E_X\left( h^{(c)}((T,\Delta),(T',\Delta'))\right)=\operatorname{c-KSD}^2(f_X||f_0),
\end{align*}
as n grows to infinity. Notice that the previous limit result requires the following conditions: $\E_X\left(|h^{(c)}((T,\Delta),(T,\Delta))|\right)<\infty$ and $\E_X\left(|h^{(c)}((T,\Delta),(T',\Delta'))|\right)<\infty$, which are satisfied under Condition d) by Lemma~\ref{lemma:conditiondkernels}. 

\paragraph{Random kernel $(c=p)$:}
For the second case, recall that
\begin{align}
 \widehat{\operatorname{p-KSD}}^2(f_X||f_0) = \sum_{i=1}^n \sum_{j=1}^n \widehat h^{(p)}((T_i,\Delta_i),(T_j,\Delta_j)),   \label{eqn:randomrhjgh193s}
\end{align}
where $\widehat{h}^{(p)}$ is a random kernel. Our first step will be to assume that we can replace the random kernel  $\widehat h^{(p)}$, given by $\widehat h^{(p)}((x,\delta),(x',\delta')) = n^2\frac{\delta \delta' K^\star(x,x')}{Y(x)Y(x')}$, by its limit
 $h^{(p)}((x,\delta),(x',\delta')) =  \frac{\delta \delta'K^\star(x,x')}{S_T(x)S_T(x')}$,
where $K^\star(x,y) = \left(\frac{\partial^2}{\partial x \partial y} K(x,y)\lambda_0(x)\lambda_0(y)\right)$. We claim that
\begin{align}
    \frac{1}{n^2}\sum_{i=1}^n\sum_{j=1}^n\widehat{h}^{(p)}((T_i,\Delta_i),(T_j,\Delta_j))=    \frac{1}{n^2}\sum_{i=1}^n\sum_{j=1}^n{h}^{(p)}((T_i,\Delta_i),(T_j,\Delta_j))+o_p(1),\label{eqn:prove}
\end{align}
and then we have that 
\begin{align*}
\widehat{\operatorname{p-KSD}}^2(f_X||f_0)&=\frac{1}{n^2}\sum_{i=1}^n\sum_{j=1}^n \widehat h^{(p)}((T_i,\Delta_i),(T_j,\Delta_j))\\
&=\frac{1}{n^2}\sum_{i=1}^n\sum_{j=1}^n h^{(p)}((T_i,\Delta_i),(T_j,\Delta_j))+o_p(1)\\
&=\E_X(h^{(p)}((T,\Delta),(T',\Delta')))+o_p(1)=\operatorname{p-KSD}^2(f_X||f_0)+o_p(1),
\end{align*}
where the third equality is due to the standard law of large numbers for V statistics, and by Condition d.3) and Lemma \ref{lemma:conditiondkernels}. 

We finish the proof by proving the claim made in Equation~\eqref{eqn:prove}. Recall that
\begin{align*}
\frac{1}{n^2}\sum_{i=1}^n\sum_{j=1}^n \widehat h^{(p)}((T_i,\Delta_i),(T_j,\Delta_j))=\left\|\frac{1}{n}\sum_{i=1}^n\widehat{\xi}^{(p)}(T_i,\Delta_i)\right\|_{\mathcal{H}^{(p)}}^2,
\end{align*}
and
\begin{align}
\frac{1}{n^2}\sum_{i=1}^n\sum_{j=1}^n  h^{(p)}((T_i,\Delta_i),(T_j,\Delta_j))=\left\|\frac{1}{n}\sum_{i=1}^n{\xi}^{(p)}(T_i,\Delta_i)\right\|_{\mathcal{H}^{(p)}}^2,\label{eqn:hfun}
\end{align}
where $\widehat\xi^{(p)}(x,\delta)=n\frac{\left(K(x,\cdot)\lambda_0(x)\right)'}{\lambda_0(x)}\frac{\delta}{Y(x)}$ and $\xi^{(p)}(x,\delta)=\frac{\left(K(x,\cdot)\lambda_0(x)\right)'}{\lambda_0(x)}\frac{\delta}{S_T(x)}.$ 
Then, by the triangular inequality, and by taking square (notice that $\|b\|-\|a-b\|\leq \|a\|\leq \|b\|+ \|a-b\|$), the claim in Equation \eqref{eqn:prove} follows from proving:
\begin{itemize}
\item[i)] $\left\|\frac{1}{n}\sum_{i=1}^n{\widehat \xi}^{(p)}(T_i,\Delta_i)-{\xi}^{(p)}(T_i,\Delta_i)\right\|_{\mathcal{H}^{(p)}}=o_p(1)$, and

\item[ii)]$\left\|\frac{1}{n}\sum_{i=1}^n{\xi}^{(p)}(T_i,\Delta_i)\right\|_{\mathcal{H}^{(p)}}=O_p(1).$
\end{itemize}
Notice that item ii) holds trivially by Equation \eqref{eqn:hfun}, and by the law of large numbers for V-statistics, which can be applied due to Lemma 9, under Condition d). We finish by proving the result in item i). Following the same steps used in Equation~\eqref{eqn:sumsplit}, we have that
\begin{align}
\left\|\frac{1}{n}\sum_{i=1}^n{\widehat \xi}^{(p)}(T_i,\Delta_i)-{\xi}^{(p)}(T_i,\Delta_i)\right\|_{\mathcal{H}^{(p)}}
&\quad=\left\|\frac{1}{n}\sum_{i=1}^n\frac{\left(K(T_i,\cdot)\lambda_0(T_i)\right)'}{\lambda_0(T_i)}\left(\frac{\Delta_i}{Y(T_i)/n}-\frac{\Delta_i}{S_T(T_i)}\right)\right\|_{\mathcal{H}^{(p)}}\nonumber\\
&\quad=\sup_{\omega\in B_1(\mathcal{H}^{(p)})}\frac{1}{n}\sum_{i=1}^n \frac{(\omega(T_i)\lambda_0(T_i))'}{\lambda_0(T_i)}\left(\frac{\Delta_i}{Y(T_i)/n}-\frac{\Delta_i}{S_T(T_i)}\right)\nonumber\\
&\quad\leq \sup_{\omega\in B_1(\mathcal{H}^{(p)})}\frac{1}{n}\sum_{i=1}^n \frac{(\omega(T_i)\lambda_0(T_i))'}{\lambda_0(T_i)}\left(\frac{\Delta_i}{Y(T_i)/n}-\frac{\Delta_i}{S_T(T_i)}\right)\ind_{\{T_i\leq t_{\varepsilon}\}}\label{eqn:random1928dnvdz}\\
&\quad\quad+ \sup_{\omega\in B_1(\mathcal{H}^{(p)})}\frac{1}{n}\sum_{i=1}^n \frac{(\omega(T_i)\lambda_0(T_i))'}{\lambda_0(T_i)}\left(\frac{\Delta_i}{Y(T_i)/n}-\frac{\Delta_i}{S_T(T_i)}\right)\ind_{\{T_i> t_{\varepsilon}\}},\label{eqn:randomxureyqzx}
\end{align}
where $\varepsilon>0$ and $t_{\varepsilon}>0$, and $t_{\epsilon}$ is the infimum over all $t>0$ such that
\begin{align*}
\int_{t}^{\infty}\int_{t}^{\infty}\frac{|K^\star(t,s)|}{\lambda_0(t)\lambda_0(s)S_T(t)S_T(s)}S_C(t)S_C(s)f_X(t)f_X(s)dtds\leq \varepsilon.    
\end{align*}
Notice that such a  $t_{\epsilon}$ is well-defined by Lemma~\ref{lemma:conditiondkernels} and Condition d.3). For the term in Equation \eqref{eqn:random1928dnvdz}, observe that
\begin{align}
&\left(\sup_{\omega\in B_1(\mathcal{H}^{(p)})}\frac{1}{n}\sum_{i=1}^n \frac{(\omega(T_i)\lambda_0(T_i))'}{\lambda_0(T_i)}\left(\frac{\Delta_i}{Y(T_i)/n}-\frac{\Delta_i}{S_T(T_i)}\right)\ind_{\{T_i\leq t_{\varepsilon}\}}\right)^2\label{eqn:thm6part1repeat} \\
&\leq\sup_{t\leq t_{\varepsilon}}\left(\frac{1}{Y(t)/n}-\frac{1}{S_T(t)}\right)^2\frac{1}{n^2}\sum_{i=1}^n\sum_{j=1}^n \Delta_i\Delta_j \frac{K^\star(T_i,T_j)}{\lambda_0(T_i)\lambda_0(T_j)}\ind_{\{T_i\leq t_{\varepsilon}\}}\ind_{\{T_j\leq t_{\varepsilon}\}}\nonumber\\
&=o_p(1)
\end{align}
where the last line holds since $\sup_{t\leq t_{\epsilon}}\left|\frac{1}{Y(t)/n}-\frac{1}{S_T(t)}\right|=o_p(1)$ a.s., by an application of Glivenko-Cantelli, and since the double sum converges to 
$$\E\left(\Delta_1\Delta_2 \frac{K^\star(T_1,T_2)}{\lambda_0(T_1)\lambda_0(T_2)}\ind_{\{T_1\leq t_{\varepsilon}\}}\ind_{\{T_2\leq t_{\varepsilon}\}}\right),$$ 
which is finite by Lemma~\ref{lemma:conditiondkernels} and Condition d.3).

Finally, we prove that the term in Equation~\eqref{eqn:randomxureyqzx} is $o_p(1)$. Define  $R(t) = \left|\frac{S_T(t)}{Y(t)/n}-1\right|$.  \citet{gill1983large} proved that $\sup_{t\leq \tau_n} R(t) = O_p(1)$ where $\tau_n = \max\{T_1,\ldots, T_n\}$. By using this result, the term in Equation \eqref{eqn:randomxureyqzx} satisfies

\begin{align*}
& \left(\sup_{\omega\in B_1(\mathcal{H}^{(p)})}\frac{1}{n}\sum_{i=1}^n \frac{(\omega(T_i)\lambda_0(T_i))'}{\lambda_0(T_i)}\left(\frac{\Delta_i}{Y(T_i)/n}-\frac{\Delta_i}{S_T(T_i)}\right)\ind_{\{T_i> t_{\varepsilon}\}}\right)^2 \nonumber\\
&\leq\frac{1}{n^2}\sum_{i=1}^n \sum_{j=1}^n \frac{\Delta_i \Delta_j|K^\star(T_i,T_j)|}{\lambda_0(T_i)\lambda_0(T_j)S_T(T_i)S_T(T_j)}R(T_i)R(T_j)\ind_{\{T_i > t_{\varepsilon}\}}\ind_{\{T_j > t_{\varepsilon}\}}\\
&=O_p(1)\frac{1}{n^2}\sum_{i=1}^n \sum_{j=1}^n  \frac{\Delta_i \Delta_j|K^\star(T_i,T_j)|}{\lambda_0(T_i)\lambda_0(T_j)S_T(T_i)S_T(T_j)}\ind_{\{T_i > t_{\varepsilon}\}}\ind_{\{T_j > t_{\varepsilon}\}}\nonumber\\
&=O_p(1) \int_{t_{\varepsilon}}^{\infty}\int_{t_{\varepsilon}}^{\infty}\frac{|K^\star(t,s)|}{\lambda_0(t)\lambda_0(s)S_T(t)S_T(s)}S_C(t)S_C(s)f_X(t)f_X(s)dtds\\
&=O_p(1)\varepsilon,\nonumber
\end{align*}
where in the second line we used that $\sup_{t\leq \tau_n}R(t) = O_p(1)$, and in the fourth line we used the law of large numbers, and the definition of $t_{\varepsilon}$. Since $\varepsilon$ is arbitrary, we conclude that equation \eqref{eqn:randomxureyqzx} tends to 0 in probability.

\subsubsection{Proof of Theorem \ref{thm:injective}}
\paragraph{Survival Stein operator (c=s):}
We proceed by contradiction. Assume that $f_X\neq f_0$ but $\operatorname{c-KSD}(f_X \| f_0)=\sup_{\omega\in B_1(\mathcal{H}^{(s)})}\E_X((\mathcal{T}_0^{(s)}\omega)(T,\Delta))=0$.
Recall that
\begin{align*}
&\E_X((\mathcal{T}_0^{(s)}\omega)(T,\Delta))\\
&=\E_X((\mathcal{T}_0\omega)(T,\Delta))\\
&=\E_X\left[\Delta\omega'(T)+\Delta\omega(T)\frac{f_0'(T)}{f_0(T)}-\Delta\omega(T)\lambda_C(T)\right]+\omega(0)f_0(0).
\end{align*}
Similarly, define 
\begin{align*}
(\mathcal{T}_X\omega)(x,\delta)&=\delta\omega'(x)+\delta\omega(x)\frac{f_X'(x)}{f_X(x)}-\delta\omega(x)\lambda_C(x)+\omega(0)f_X(0),
\end{align*}
and notice that $\E_X((\mathcal{T}_X\omega)(T,\Delta))=0$ by the Stein's identity. Then
\begin{align*}
\E_X\left((\mathcal{T}^{(s)}_0\omega)(T,\Delta)\right)&=\E_X\left((\mathcal{T}_0\omega)(T,\Delta)\right)\\
&=\E_X\left((\mathcal{T}_0 \omega)(T,\Delta)-(\mathcal{T}_X \omega)(T,\Delta)\right)\\
&=\E_X\left(\Delta\omega(T)\left(\frac{f_0'(T)}{f_0(T)}-\frac{f_X'(T)}{f_X(T)}\right)+\omega(0)(f_0(0)-f_X(0))\right)\\
&=\E_X\left(\Delta\omega(T)\left(\log \frac{f_0(T)}{f_X(T)}\right)'\right)+\omega(0)(f_0(0)-f_X(0)),
\end{align*}
and thus, we have
\begin{align*}
0=\operatorname{s-KSD}(f_X \| f_0)&=\sup_{\omega\in B_1(\mathcal{H}^{(s)})}\E_X((\mathcal{T}_0^{(s)}\omega)(T,\Delta))\\
&=\sup_{\omega\in B_1(\mathcal{H}^{(s)})}\E_X\left(\Delta\omega(T)\left(\log \frac{f_0(T)}{f_X(T)}\right)'\right)+\omega(0)(f_0(0)-f_X(0))\\
&=\sup_{\omega\in B_1(\mathcal{H}^{(s)})}\left\langle\omega,\int_{0}^\infty K(x,\cdot)d\nu(x)\right\rangle=\left\|\int_{0}^\infty K(x,\cdot)d\nu(x)\right\|_{\mathcal{H}^{(s)}},
\end{align*}
where $d\nu(x)=\left(\log\frac{f_0(x)}{f_X(x)}\right)'S_C(x)f_X(x)dx+(f_0(x)-f_X(x))\delta_0(x)$, and where we identify $\int_0^\infty K(x,\cdot)d\nu(x)$ as the mean kernel embedding of the measure $\nu$.  We shall assume that the above embedding is well-defined, otherwise we have $\operatorname{s-KSD}(f_X \| f_0)\neq 0$. Since the kernel is $c_0$-universal, the previous set of equations implies $\nu$ is the zero measure, which implies that $f_0(0)=f_X(0)$, and 
\begin{align}
\left(\log \frac{f_0(x)}{f_X(x)}\right)'=0,\label{eqn:consistency1}
\end{align}
 as long as $f_X(x)>0$ implies $S_C(x)f_X(x)>0$ (which does, since we assume $S_C(x)=0$ implies $S_X(x)=\int_x^\infty f_X(x)dx=0$). Equation \eqref{eqn:consistency1} yields $f_0\propto f_X$ and $f_X=f_0$ since both, $f_0$ and $f_X$, are probability density functions. This finalizes our proof.

\paragraph{Martingale Stein operator (c=m):} Define 
\begin{align*}
(\mathcal{T}_X^{(m)}\omega)(x,\delta)&=\omega'(x)\frac{\delta}{\lambda_X(x)}-(\omega(x)-\omega(0)),
\end{align*}
and notice that $\E_X((\mathcal{T}_X^{(m)}\omega)(T,\Delta))=0$ follows from the martingale identity.
Observe that
\begin{align*}
\operatorname{m-KSD}(f_X \| f_0)&=\sup_{\omega\in B_1(\mathcal{H}^{(m)})}\E_X((\mathcal{T}_0^{(m)}\omega)(T,\Delta))\\
&=\sup_{\omega\in B_1(\mathcal{H}^{(m)})}\E_X((\mathcal{T}_0^{(m)}\omega)(T,\Delta))-\E_X((\mathcal{T}_X^{(m)}\omega)(T,\Delta))\\
&=\sup_{\omega\in B_1(\mathcal{H}^{(m)})}\E_X\left(\omega'(T)\Delta\left(\frac{1}{\lambda_0(T)}-\frac{1}{\lambda_X(T)}\right)\right)\nonumber\\
&= \sup_{\omega\in B_1(\mathcal{H}^{(m)})}  \int_0^{\infty} \omega'(x)\left(\frac{1}{\lambda_0(x)}-\frac{1}{\lambda_X(x)} \right)f_X(x)S_C(x)dx.
\end{align*}
Denote $\alpha(x) = \left(\frac{1}{\lambda_0(x)}-\frac{1}{\lambda_X(x)} \right)f_X(x)S_C(x)$, and, as usual, $K^\star(x,y) = \frac{\partial^2}{\partial x \partial y} K(x,y)$. Then, 
\begin{align*}
\operatorname{m-KSD}^2(f_X \| f_0)
&= \int_0^{\infty}\int_0^{\infty} \alpha(x)K^\star(x,y)\alpha(y)dxdy.
\end{align*}
Since $K^\star$ is $c_0$-universal by Condition a), the previous term is equal to $0$ if and only if $\alpha(x) = 0$ for all $x>0$. Now, $\alpha(x) = 0$ if and only if $\frac{1}{\lambda_0(x)}-\frac{1}{\lambda_X(x)}=0$, which holds if and only if $f_0(x) = f_X(x)$ for all $x >0$. 

\subsubsection{Proof of Theorem \ref{thm:h_0}}
\paragraph{Deterministic kernels $(c\in\{s,m\})$:}
For $c\in\{s,m\}$ which are associated to a deterministic kernel function $h^{(c)}((T,\Delta),(T',\Delta'))$, the result follows from the classical theory of V-statistics since $h^{(c)}$ are degenerate kernels, and under the following moment conditions:
\begin{itemize}
\item[i)] $\E_0(|h^{(c)}((T,\Delta),(T,\Delta))|)<\infty$, and 
\item[ii)] $\E_0(h^{(c)}((T,\Delta),(T',\Delta'))^2)<\infty$,
\end{itemize}
which are satisfied due to Lemma~\ref{lemma:conditiondkernels2}.

\paragraph{Random kernel $(c\in\{p\})$:}
Observe that
\begin{align*}
\sqrt{n}\widehat{\operatorname{c-KSD}}(f_X \| f_0)&=\sup_{\omega\in B_1(\mathcal{H}^{(p)})}\frac{1}{\sqrt{n}}\sum_{i=1}^n\frac{(\omega(T_i)\lambda_0(T_i))'}{\lambda_0(T_i)}\frac{\Delta_i}{Y(T_i)/n}\\
&=\sup_{\omega\in B_1(\mathcal{H}^{(p)})}\frac{1}{\sqrt{n}}\int_{0}^{\tau_n}\frac{(\omega(x)\lambda_0(x))'}{\lambda_0(x)}\frac{1}{Y(x)/n}dN(x),
\end{align*}
where $dN(x)=\sum_{i=1}^n\Delta_i\delta_{T_i}(x)$.  By hypothesis, $\int_0^{\infty}(\omega(x)\lambda_0(x))'dx=0$ for all $\omega\in\mathcal{H}^{(p)}$, then
\begin{align*}
\sqrt{n}\widehat{\operatorname{c-KSD}}(f_X \| f_0)
&=\sup_{\omega\in B_1(\mathcal{H}^{(p)})}\frac{1}{\sqrt{n}}\int_{0}^{\tau_n}\frac{(\omega(x)\lambda_0(x))'}{\lambda_0(x)}\frac{1}{Y(x)/n}dN(x)-\sqrt{n}\int_0^{\infty}(\omega(x)\lambda_0(x))'dx\\
&=\sup_{\omega\in B_1(\mathcal{H}^{(p)})}\frac{1}{\sqrt{n}}\int_{0}^{\tau_n}\frac{(\omega(x)\lambda_0(x))'}{\lambda_0(x)}\frac{1}{Y(x)/n}dM(x)-\sqrt{n}\int_{\tau_n}^{\infty}(\omega(x)\lambda_0(x))'dx
\end{align*}
where $dM(x)=dN(x)-Y(x)\lambda_0(x)dx$. Therefore we conclude that $\sqrt{n}\widehat{\operatorname{c-KSD}}(f_X \| f_0)\in [a-b,a+b]$, where 
\begin{align*}
a&= \sup_{\omega\in B_1(\mathcal{H}^{(p)})}\frac{1}{\sqrt{n}}\int_{0}^{\tau_n}\frac{(\omega(x)\lambda_0(x))'}{\lambda_0(x)}\frac{1}{Y(x)/n}dM(x), \text{ and}\\
b&=\sup_{\omega\in B_1(\mathcal{H}^{(p)})}\sqrt{n}\int_{\tau_n}^{\infty}(\omega(x)\lambda_0(x))'dx
\end{align*}
  
We will prove that $b=o_p(1)$. Let $K^\star(x,y) = \left(\frac{\partial ^2}{\partial x \partial y} \lambda_0(x)\lambda_0(y)K(x,y) \right)$, then
\begin{align*}
\left(\sup_{\omega\in B_1(\mathcal{H}^{(p)})}\sqrt{n}\int_{\tau_n}^{\infty}(\omega(x)\lambda_0(x))'dx\right)^2
&=n \int_{\tau_n}^{\infty}\int_{\tau_n}^{\infty} \frac{K^\star(x,y)}{f_T(x)f_T(y)}f_T(x)f_T(y)dxdy\\
& \leq n S_T(\tau_n)^{1/2}\left(\int_{\tau_n}^{\infty}\left(\int_{\tau_n}^{\infty} \frac{K^\star(x,y)}{f_T(x)f_T(y)}f_T(x)dx\right)^2f_T(y)dy\right)^{1/2}\\
& \leq n S_T(\tau_n)\left(\int_{\tau_n}^{\infty}\int_{\tau_n}^{\infty} \frac{K^\star(x,y)^2}{f_T(x)^2f_T(y)^2}f_T(x)f_T(y)dxdy\right)^{1/2},\\
\end{align*}
where the two inequalities above follow from the Cauchy-Schwarz inequality, by the fact that  $n S_T(\tau_n)=O_p(1)$ \citep{Yang1994}, and  the previous double integral converges to 0 by Condition c.3), since $\tau_n=\max\{T_1,\ldots,T_n\}\to\infty$. From the previous result, we deduce
\begin{align*}
\sqrt{n}\widehat{\operatorname{c-KSD}}(f_X \| f_0)&=\sup_{\omega\in B_1(\mathcal{H}^{(p)})}\frac{1}{\sqrt{n}}\int_{0}^{\tau_n}\frac{(\omega(x)\lambda_0(x))'}{\lambda_0(x)}\frac{1}{Y(x)/n}dM(x)+o_p(1).
\end{align*}
The previous step is important in our analysis as it allows us to write $\sqrt{n}\widehat{\operatorname{c-KSD}}(f_X \| f_0)$ in terms of $M(x)$. Our next step is to prove that we can replace the term $Y(x)/n$, in the previous equation, by $S_T(x)$. Observe
\begin{align*}
&\sqrt{n}\widehat{\operatorname{c-KSD}}(f_X \| f_0)\\
&\quad=\sup_{\omega\in B_1(\mathcal{H}^{(p)})}\frac{1}{\sqrt{n}}\int_{0}^{\tau_n}\frac{(\omega(x)\lambda_0(x))'}{\lambda_0(x)}\left(\frac{1}{Y(x)/n}-\frac{1}{S_T(x)}+\frac{1}{S_T(x)}\right)dM(x)+o_p(1)\\
&\quad=\sup_{\omega\in B_1(\mathcal{H}^{(p)})}\frac{1}{\sqrt{n}}\int_{0}^{\tau_n}\frac{(\omega(x)\lambda_0(x))'}{\lambda_0(x)}\frac{1}{S_T(x)}dM(x)\\
&\quad\quad\pm \sup_{\omega\in B_1(\mathcal{H}^{(p)})}\frac{1}{\sqrt{n}}\int_{0}^{\tau_n}\frac{(\omega(x)\lambda_0(x))'}{\lambda_0(x)}\left(\frac{1}{Y(x)/n}-\frac{1}{S_T(x)}\right)dM(x)+o_p(1).
\end{align*}
The $\pm$ notation above denotes  lower, given by $-$, and upper, given by $+$, bounds for $\sqrt{n}\widehat{\operatorname{c-KSD}}(f_X \| f_0)$. Finally, by taking square, the result is deduced by proving 
\begin{align*}
\sup_{\omega\in B_1(\mathcal{H}^{(p)})}\frac{1}{\sqrt{n}}\int_{0}^{\tau_n}\frac{(\omega(x)\lambda_0(x))'}{\lambda_0(x)}\left(\frac{1}{Y(x)/n}-\frac{1}{S_T(x)}\right)dM(x)=o_p(1),    
\end{align*}
and 
\begin{align*}
\sup_{\omega\in B_1(\mathcal{H}^{(p)})}\frac{1}{\sqrt{n}}\int_{0}^{\tau_n}\frac{(\omega(x)\lambda_0(x))'}{\lambda_0(x)}\frac{1}{S_T(x)}dM(x)=O_p(1). 
\end{align*}
The second equation won't be verified as, at the end of this proof, we will show that such a quantity converges in distribution to some random variable, thus it will be bounded in probability. For the first equation, notice that
\begin{align*}
&\left(\sup_{\omega\in B_1(\mathcal{H}^{(p)})}\frac{1}{\sqrt{n}}\int_{0}^{\tau_n}\frac{(\omega(x)\lambda_0(x))'}{\lambda_0(x)}\left(\frac{1}{Y(x)/n}-\frac{1}{S_T(x)}\right)dM(x)\right)^2\\
&=\frac{1}{n}\int_0^{\tau_n}\int_0^{\tau_n}\frac{K^\star(x,y)}{\lambda_0(x)\lambda_0(y)}\left(\frac{1}{Y(x)/n}-\frac{1}{S_T(x)}\right)\left(\frac{1}{Y(y)/n}-\frac{1}{S_T(y)}\right)dM(x)dM(y),
\end{align*}
is a double integral with respect to the $M(x)$. Then, by  Theorem 17 of \citet{fernandez2019reproducing}, it is enough to verify $$\frac{1}{n}\int_0^{\tau_n}\frac{K^{\star}(x,x)}{\lambda_0(x)^2}\left(\frac{1}{Y(x)/n}-\frac{1}{S_T(x)}\right)^2Y(x)\lambda_0(x)dx=o_p(1).$$ Observe that
\begin{align*}
\frac{1}{n}\int_0^{\tau_n}\frac{K^{\star}(x,x)}{\lambda_0(x)^2}\left(\frac{1}{Y(x)/n}-\frac{1}{S_T(x)}\right)^2Y(x)\lambda_0(x)dx 
&=\int_0^{\tau_n}\frac{K^{\star}(x,x)}{\lambda_0(x)^2}\left(1-\frac{Y(x)/n}{S_T(x)}\right)^2\frac{1}{Y(x)/n}\lambda_0(x)dx\\
&=O_p(1)\int_0^{\tau}\frac{K^{\star}(x,x)}{\lambda_0(x)^2}\left(1-\frac{Y(x)/n}{S_T(x)}\right)^2\frac{1}{S_T(x)}\lambda_0(x)dx\\
&=o_p(1),
\end{align*}
where the second equality follows from $n/Y(x)=O_p(1)1/S_T(x)$ uniformly for all $x\leq \tau_n$ \citep{gill1983large}, and the last equality is due to dominated convergence in sets of probability as high as desired, as $\left(1-\frac{Y(x)/n}{S_T(x)}\right)\to 0$ for all $x<\infty$ from the Glivenko Cantelli Theorem, and 
\begin{align*}
\frac{K^{\star}(x,x)}{\lambda_0(x)^2}\left(1-\frac{Y(x)/n}{S_T(x)}\right)^2\frac{1}{S_T(x)}\lambda_0(x)&=O_p(1)\frac{K^{\star}(x,x)}{f_0(x)^2 S_C(x)}f_0(x),
\end{align*}
which is integrable by Condition c.3).

Putting everything together, we have shown that
\begin{align*}
\sqrt{n}\widehat{\operatorname{c-KSD}}^2(f_X \| f_0)
&\quad=\left(\sup_{\omega\in B_1(\mathcal{H}^{(p)})}\frac{1}{\sqrt{n}}\int_{0}^{\tau_n}\frac{(\omega(x)\lambda_0(x))'}{\lambda_0(x)}\frac{1}{S_T(x)}dM(x)\right)^2+o_p(1)\\
&\quad=\frac{1}{n}\int_0^{\tau_n}\int_0^{\tau_n}\frac{K^\star(x,y)}{f_0(x)f_0(y)S_C(x)S_C(y)}dM(x)dM(y)+o_p(1)\\
&\quad=\frac{1}{n}\sum_{i=1}^n\sum_{j=1}^n\int_0^{X_i}\int_0^{X_j}\frac{K^\star(x,y)}{f_0(x)f_0(y)S_C(x)S_C(y)}dM_j(x)dM_i(y)+o_p(1)\\
&\quad=\frac{1}{n}\sum_{i=1}^n\sum_{j=1}^nJ((T_i,\Delta_i),(T_j,\Delta_j))+o_p(1),
\end{align*}
where $M_i(x)=N_i(x)-\int_0^{x}\ind_{\{T_i\geq y\}}\lambda_0(y)dy=\Delta_i\ind_{\{T_i\leq x\}}-\int_0^{x}\ind_{\{T_i\geq y\}}\lambda_0(y)dy$. Notice that the process $M_i(x)$ only depends on the $i$-th observation $(T_i,\Delta_i)$. Notice that the previous expression is approximately a V-statistic with kernel given by $J((T_i,\Delta_i),(T_j,\Delta_j))=\int_0^{T_i}\int_0^{T_j}\frac{K^\star(x,y)}{f_0(x)f_0(y)S_C(x)S_C(y)}dM_j(x)dM_i(y)$. By proposition 23 of \citet{fernandez2019reproducing}, we have that  $\E(J((T_i,\Delta_i),(T_j,\Delta_j))|T_i,\Delta_i)=0$, thus $J$ is a degenerate V-statistic kernel. 

By the classical theory of V-statistics, 
\begin{align*}
  \frac{1}{n}\sum_{i=1}^n\sum_{j=1}^nJ((T_i,\Delta_i),(T_j,\Delta_j))\overset{\mathcal{D}}{\to} r_p+\mathcal{Y}_p, 
\end{align*}
where $r_p$ is a constant and $\mathcal{Y}_p$ is a (potentially) infinite sum of independent $\chi^2$ random variables, as long as the following moment conditions are satisfied: 
\begin{align*}
\text{i) }\E_0(|J((T_1,\Delta_1),(T_1,\Delta_1))|)<\infty,\quad \text{and}\quad \text{ii) }\E_0(J((T_1,\Delta_1),(T_2,\Delta_2))^2)<\infty.
\end{align*}
Again, by Proposition 23 of   \citet{fernandez2019reproducing}, checking those moment conditions is equivalent to verify:
\begin{align*}
\text{i) }\E_0\left[\frac{K^\star(T,T)\Delta}{(f_0(T)S_C(T))^2}\right]<\infty\quad \text{and}\quad \text{ii) }\E_0\left[\frac{K^\star(T,T')^2\Delta\Delta'}{(f_0(T)f_0(T')S_C(T)S_C(T'))^2}\right]<\infty,
\end{align*}
which are exactly the conditions assumed in Condition c.3).

\section{Known Identities}\label{app:knwon_identity}

In Section 3.2, to derive the martingale Stein operator, we use the following identity
\begin{align*}
     \E_0\left[\Delta \phi(T)-\int_0^{T} \phi(t)\lambda_0(t)dt\right]=0,
 \end{align*}
which holds under the null hypothesis, where $\lambda_0$ is the hazard function under the null. 

Let $N_i(x)$ and $Y_i(x)$ be the individual counting and risk processes, defined by by $N_i(x)=\Delta_i\ind_{\{T_i\leq x\}}$ and $Y_i(x)=\ind_{\{T_i\geq x\}}$, respectively.  Then, the individual zero-mean martingale for the i-th individual corresponds to $M_i(x)=N_i(x)-\int_0^xY_i(y)\lambda_0(y)dy$, where $\E_0(M_i(x))=0$ for all $x$. 

Additionally, let $\phi:\R_+\to\R$ such that $\E_0\left|\int_0^x \phi(y) dM_i(y)\right|<\infty$ for all $x$, then $\int_0^x \phi(y) dM_i(y)$ is a zero-mean $(\mathcal{F}_x)$-martingale (See Chapter 2 of \citep{aalen2008survival}). The, taking expectation, we have 
\begin{align*}
\E_0 \left[\int_0^\infty \phi(x)dM_i(x)\right] = \E_0\left[\int_0^\infty \phi(x)(dN_i(x)-Y_i(x)\lambda_0(x)dx) \right] = \E_0\left[\Delta\phi(T) - \int_0^{T}\phi(x)\lambda_0(x)dx\right]=0.
\end{align*}

\section{Additional Experiments}\label{sec:add_exp}

\subsection{Weibull experiments: small deviations from the null}

\subsubsection*{Sample size: 30, and censoring percentages of $30\%$, $50\%$ and $70\%$}
\begin{figure}[H]
    \centering
    \includegraphics[scale=0.21]{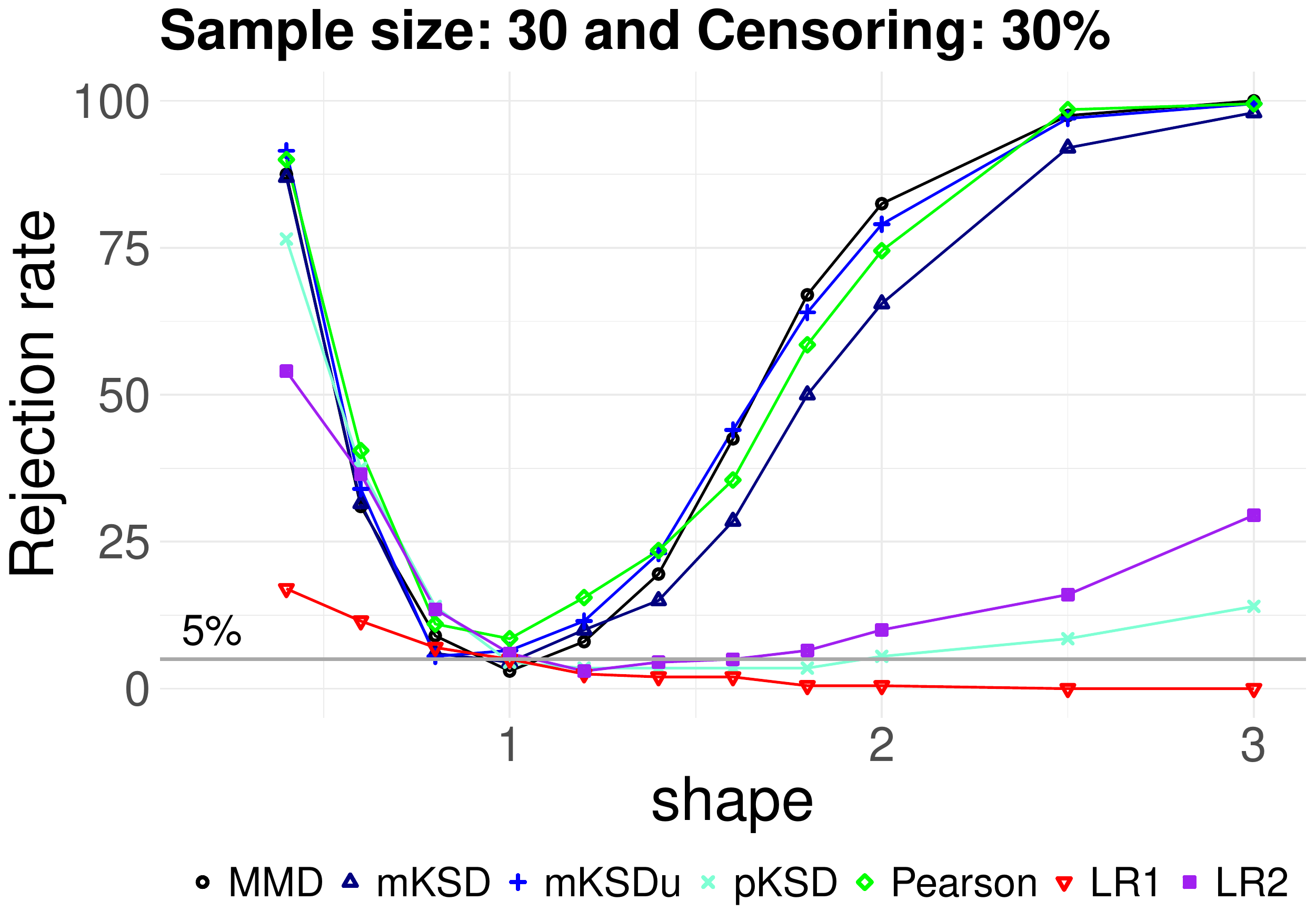}
    \includegraphics[scale=0.21]{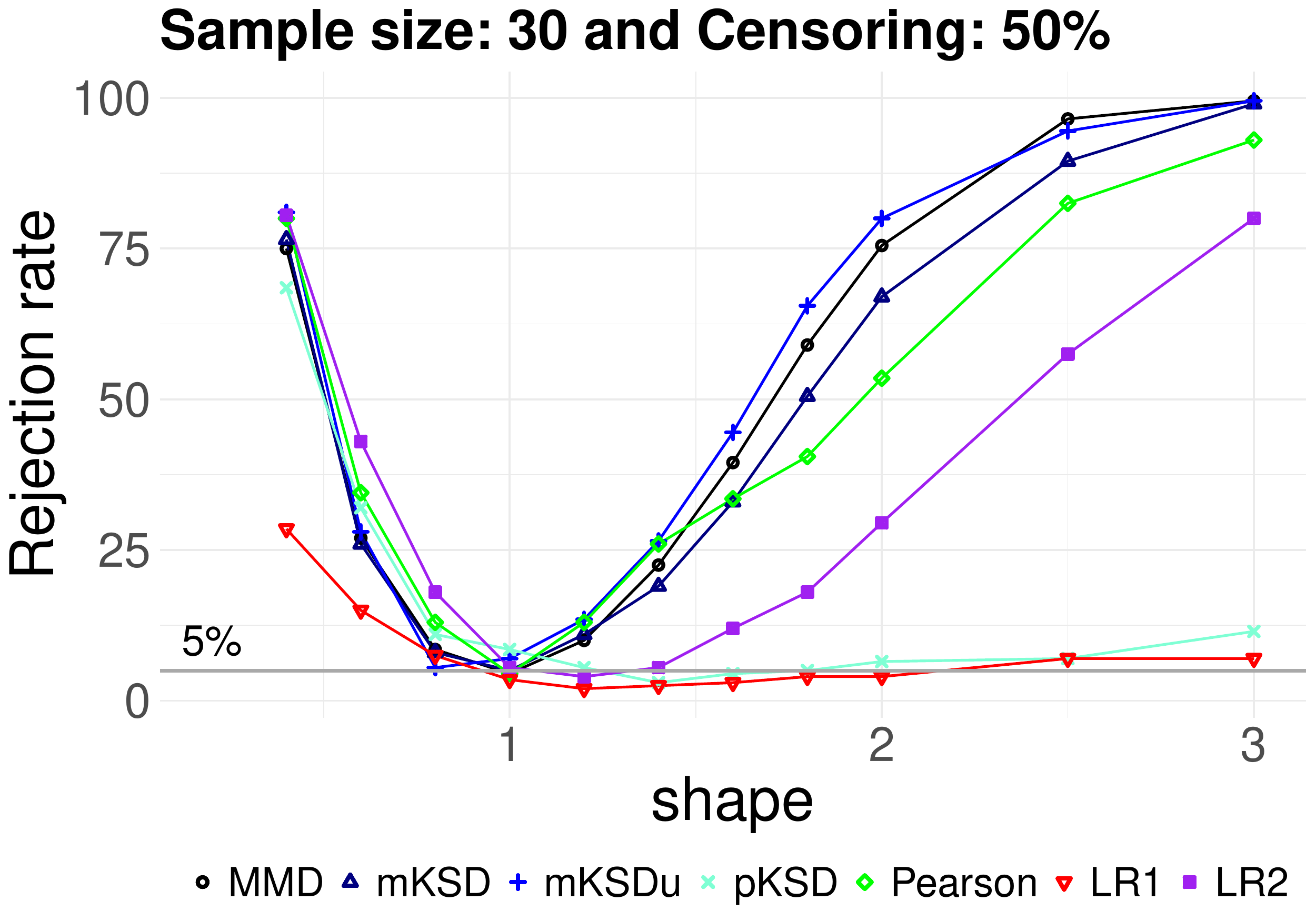}
    \includegraphics[scale=0.21]{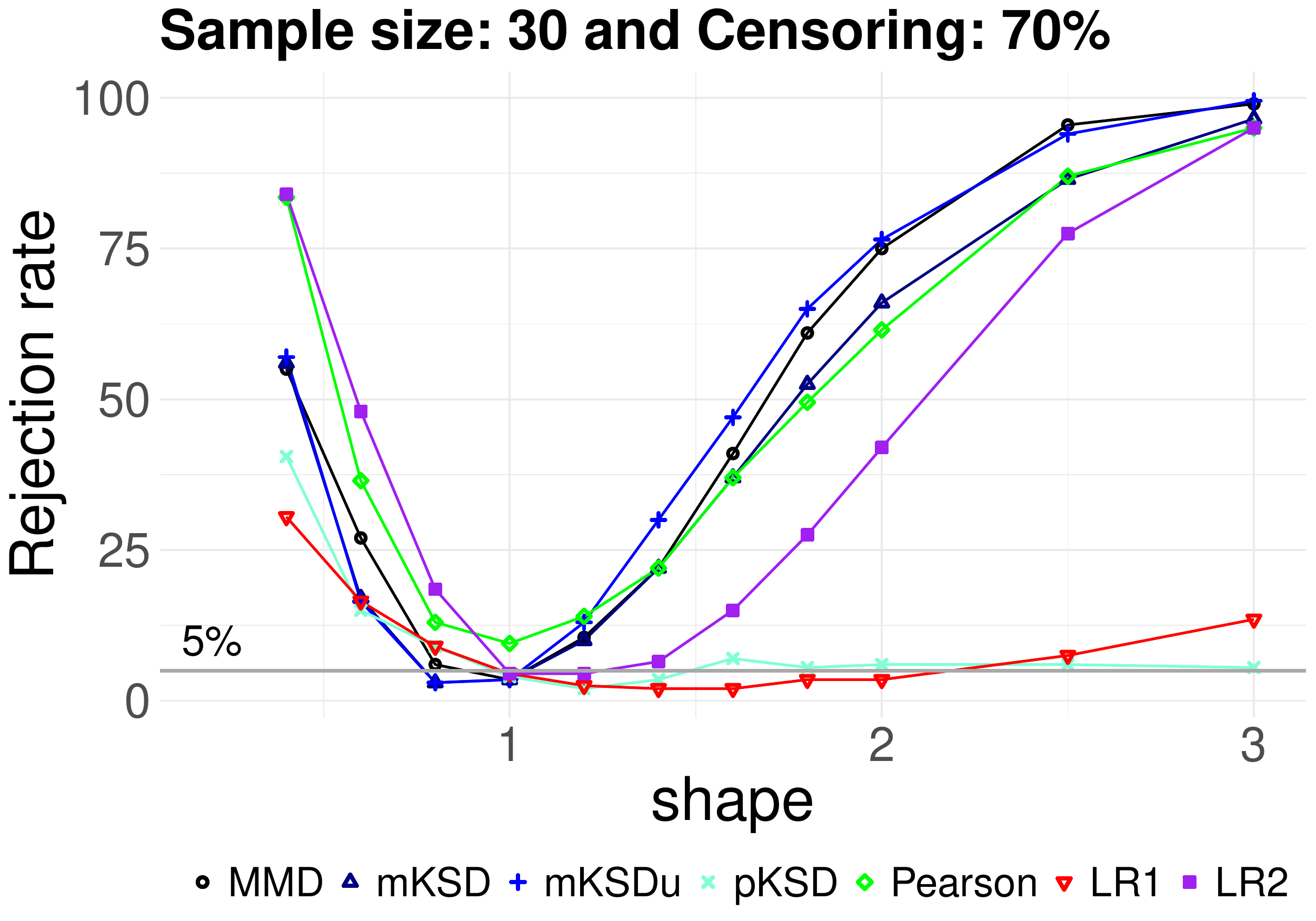}
\end{figure}
\subsubsection*{Sample size: 50, and censoring percentages of $30\%$, $50\%$ and $70\%$}
\begin{figure}[H]
    \centering
    \includegraphics[scale=0.21]{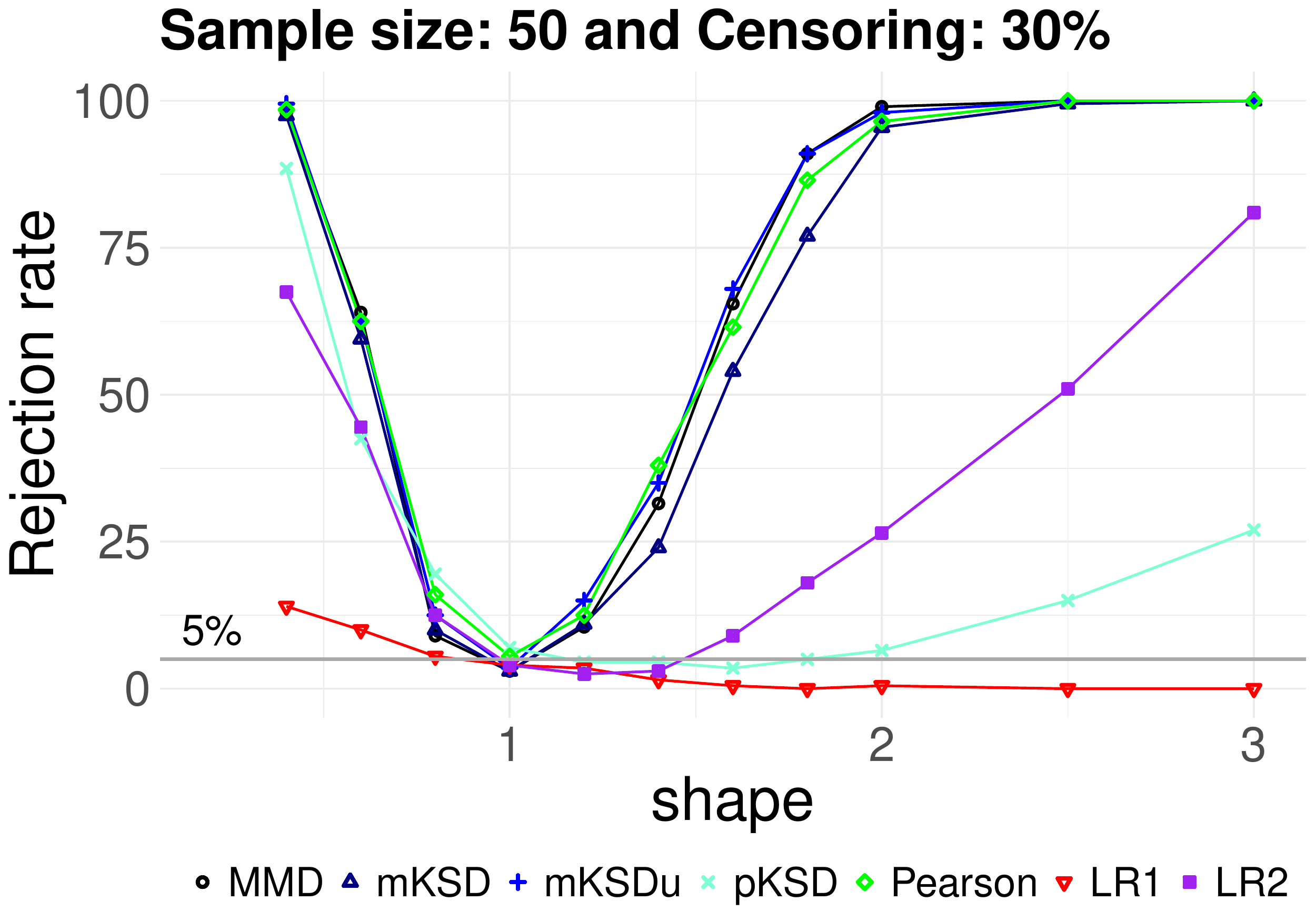}
    \includegraphics[scale=0.21]{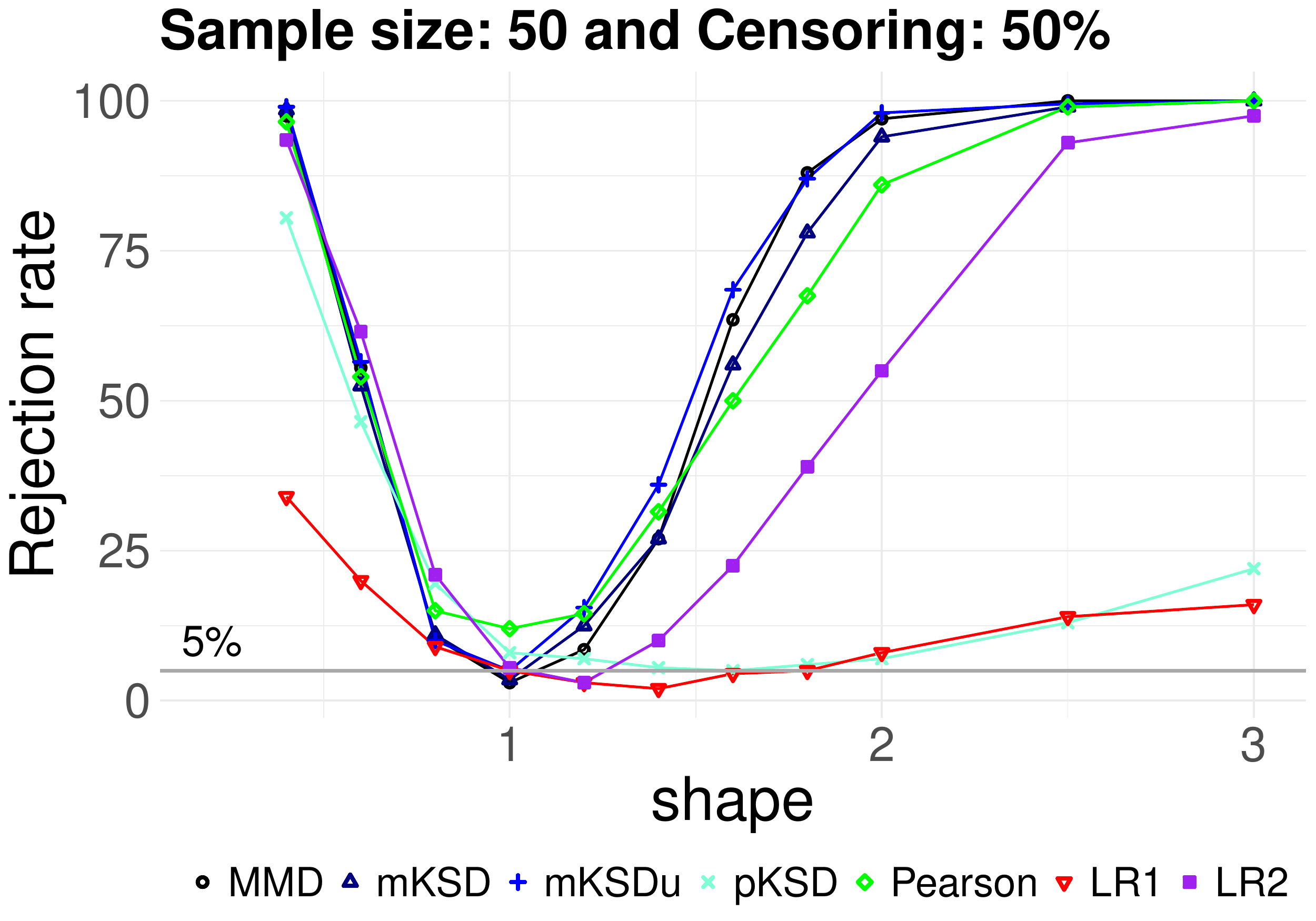}
    \includegraphics[scale=0.21]{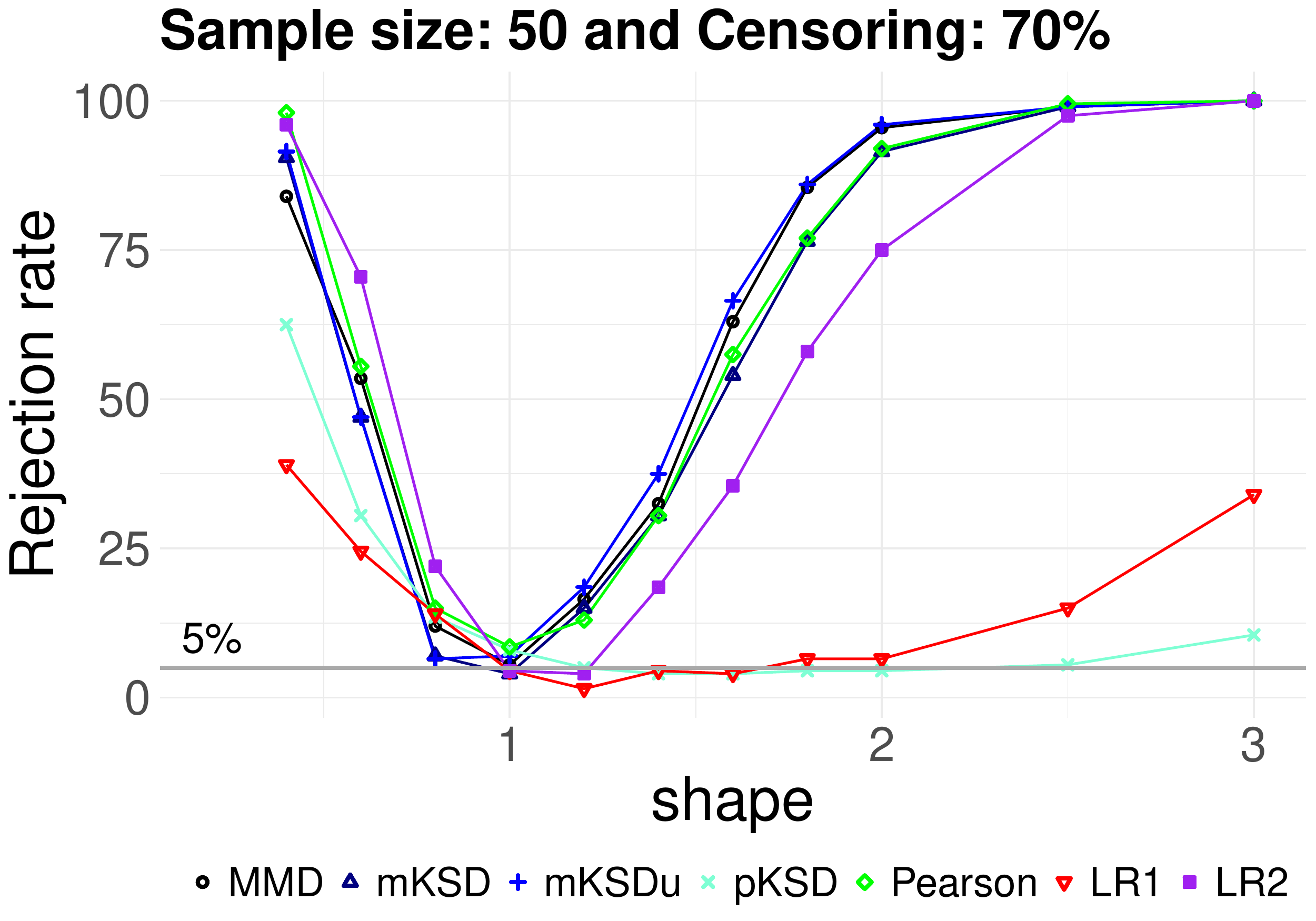}
\end{figure}
\subsubsection*{Sample size: 100, and censoring percentages of $30\%$, $50\%$ and $70\%$}
\begin{figure}[H]
    \centering
    \includegraphics[scale=0.21]{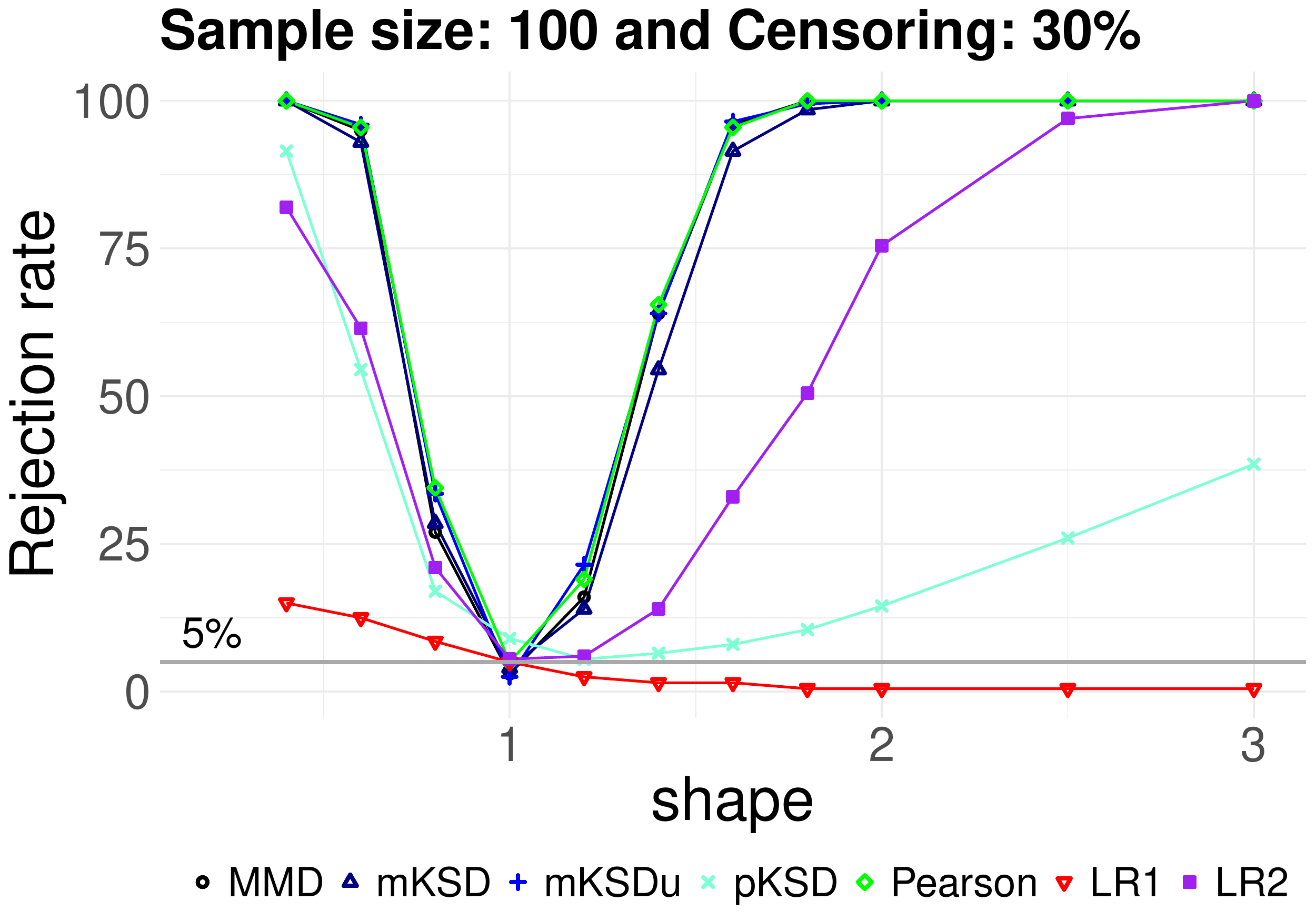}
    \includegraphics[scale=0.21]{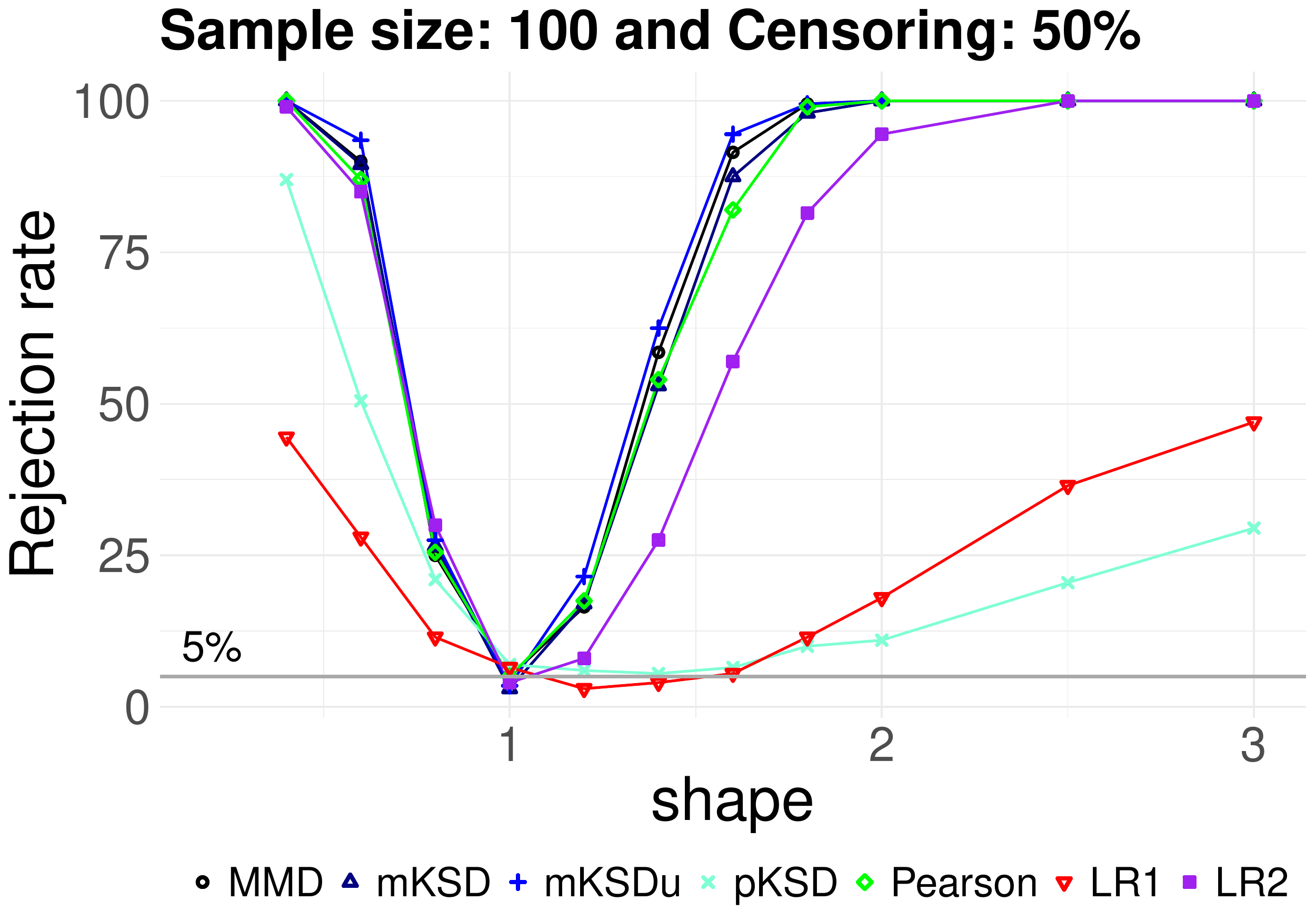}
    \includegraphics[scale=0.21]{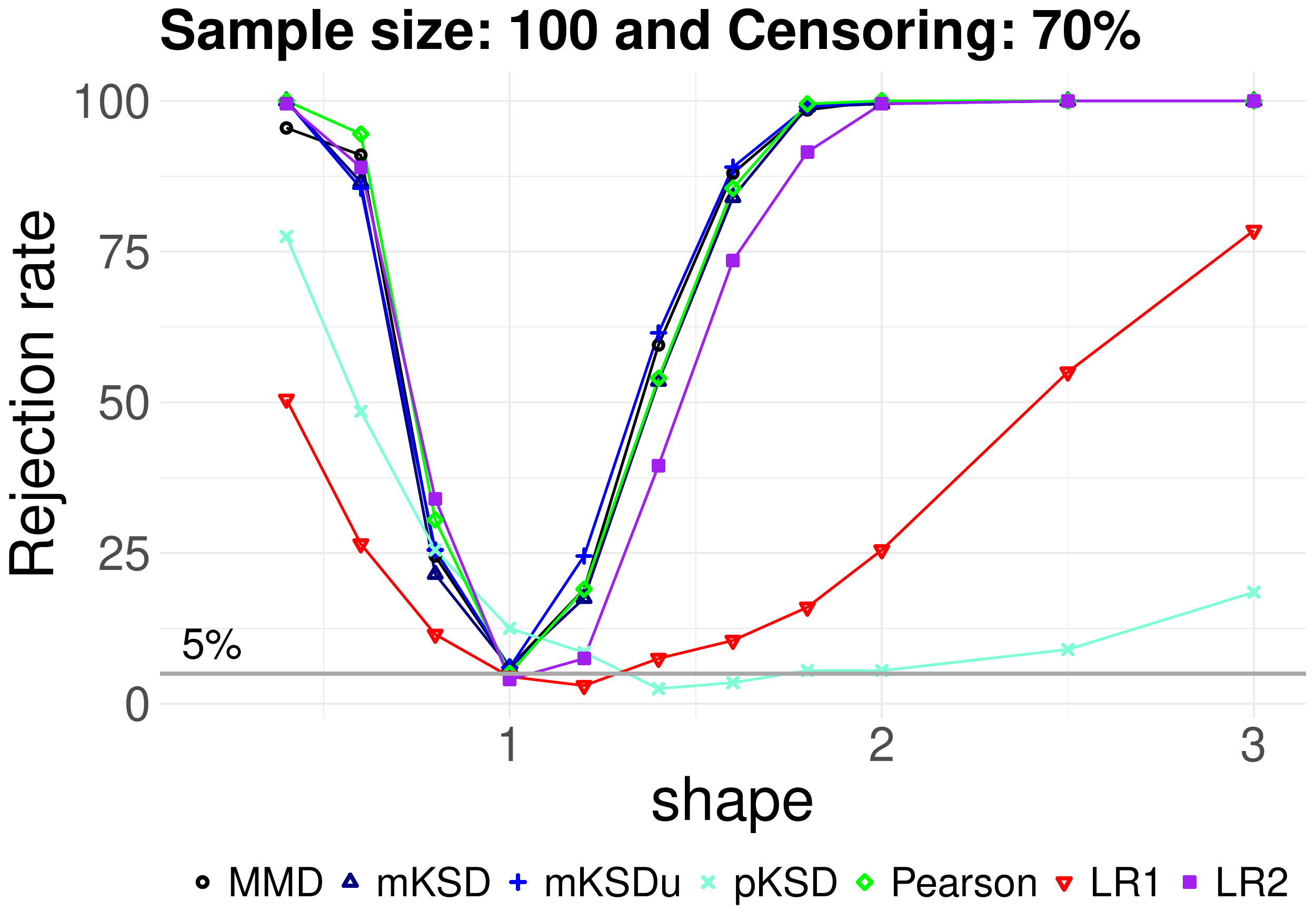}
\end{figure}
\subsubsection*{Sample size: 200, and censoring percentages of $30\%$, $50\%$ and $70\%$}
\begin{figure}[H]
    \centering
    \includegraphics[scale=0.21]{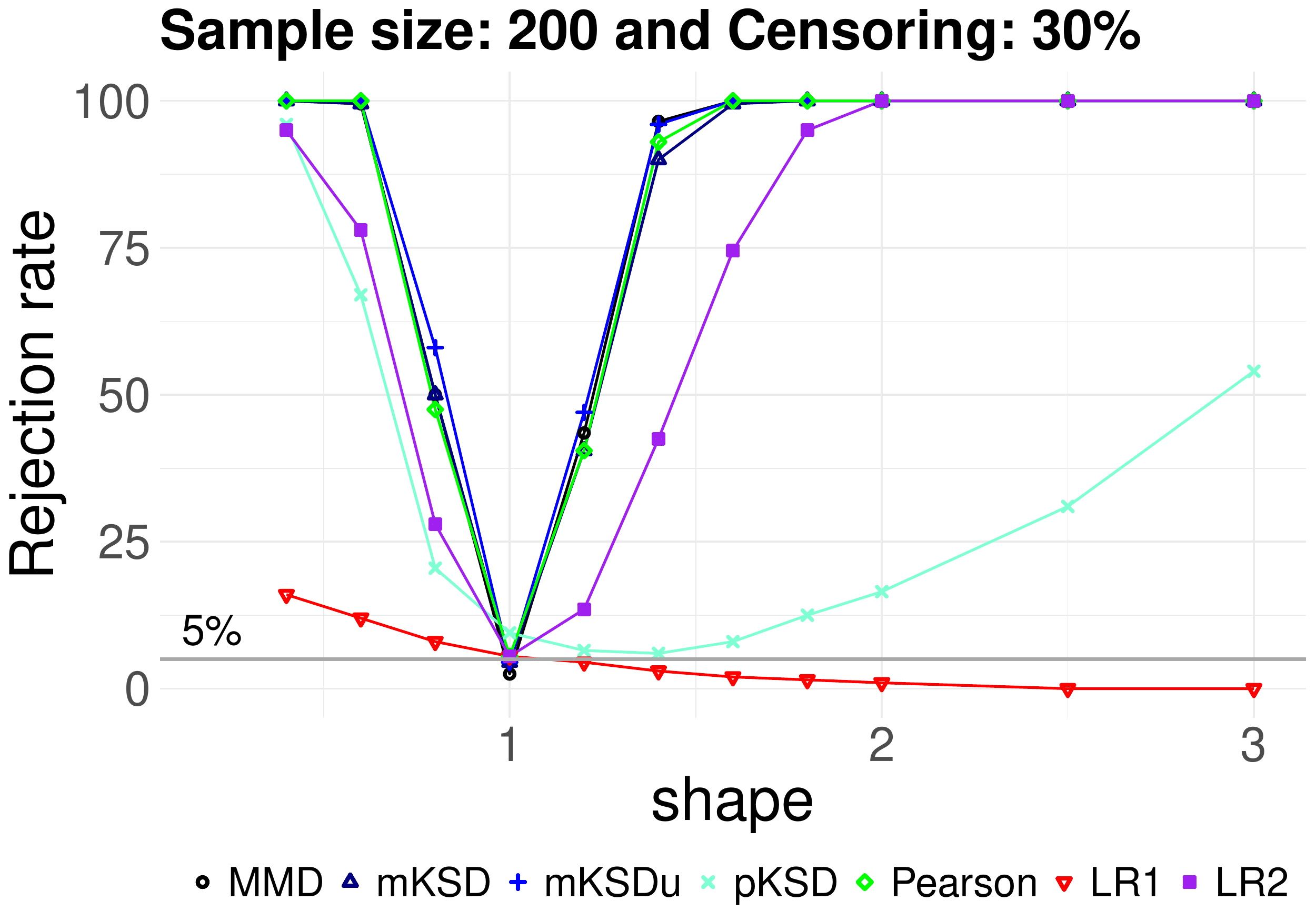}
    \includegraphics[scale=0.21]{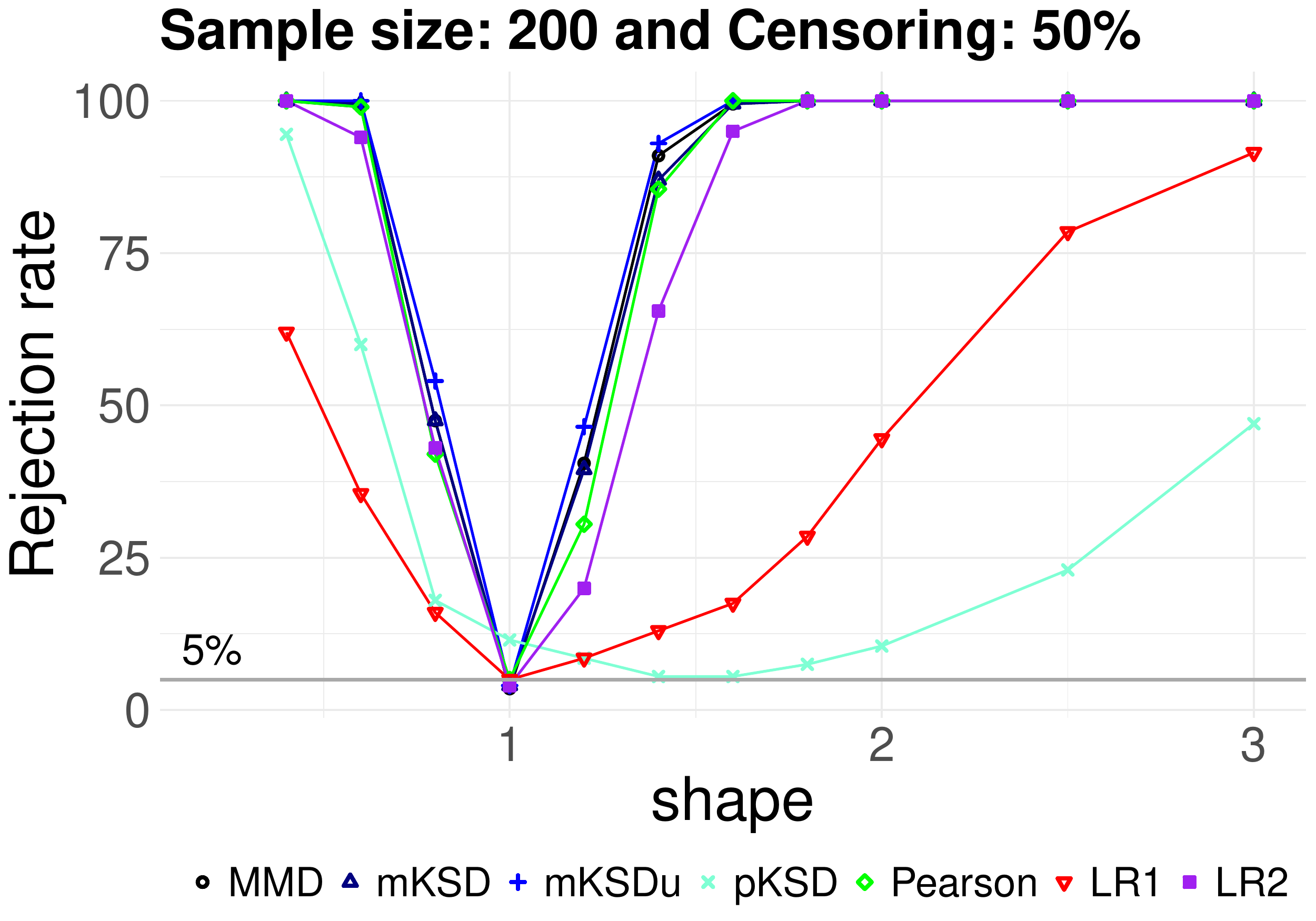}
    \includegraphics[scale=0.21]{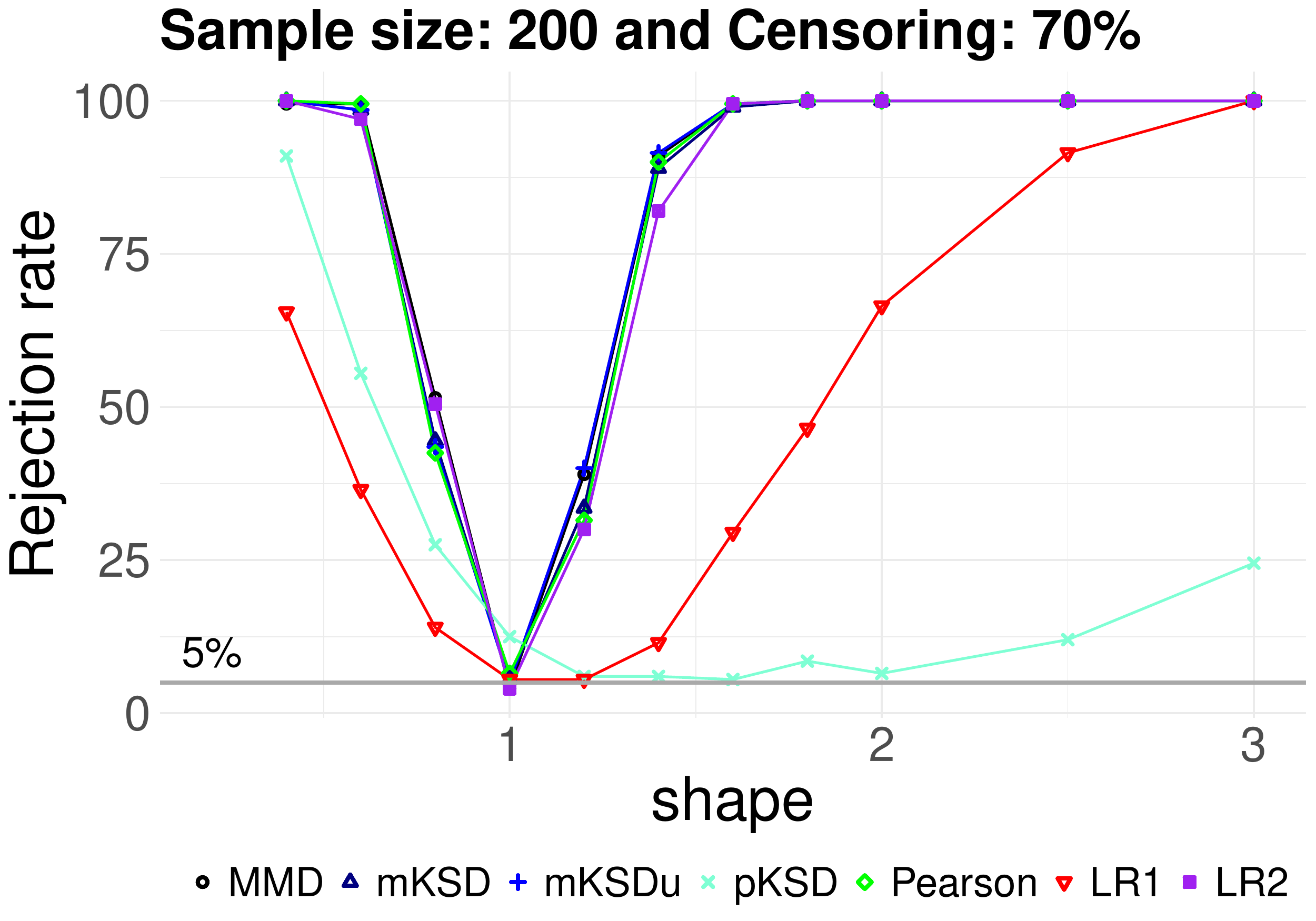}
\end{figure}

\subsection{Weibull experiments: increasing sample size}
\subsubsection*{Shape: $0.6$ and censoring percentages of $30\%$ and $70\%$}
\begin{figure}[H]
    \centering
    \includegraphics[scale=0.3]{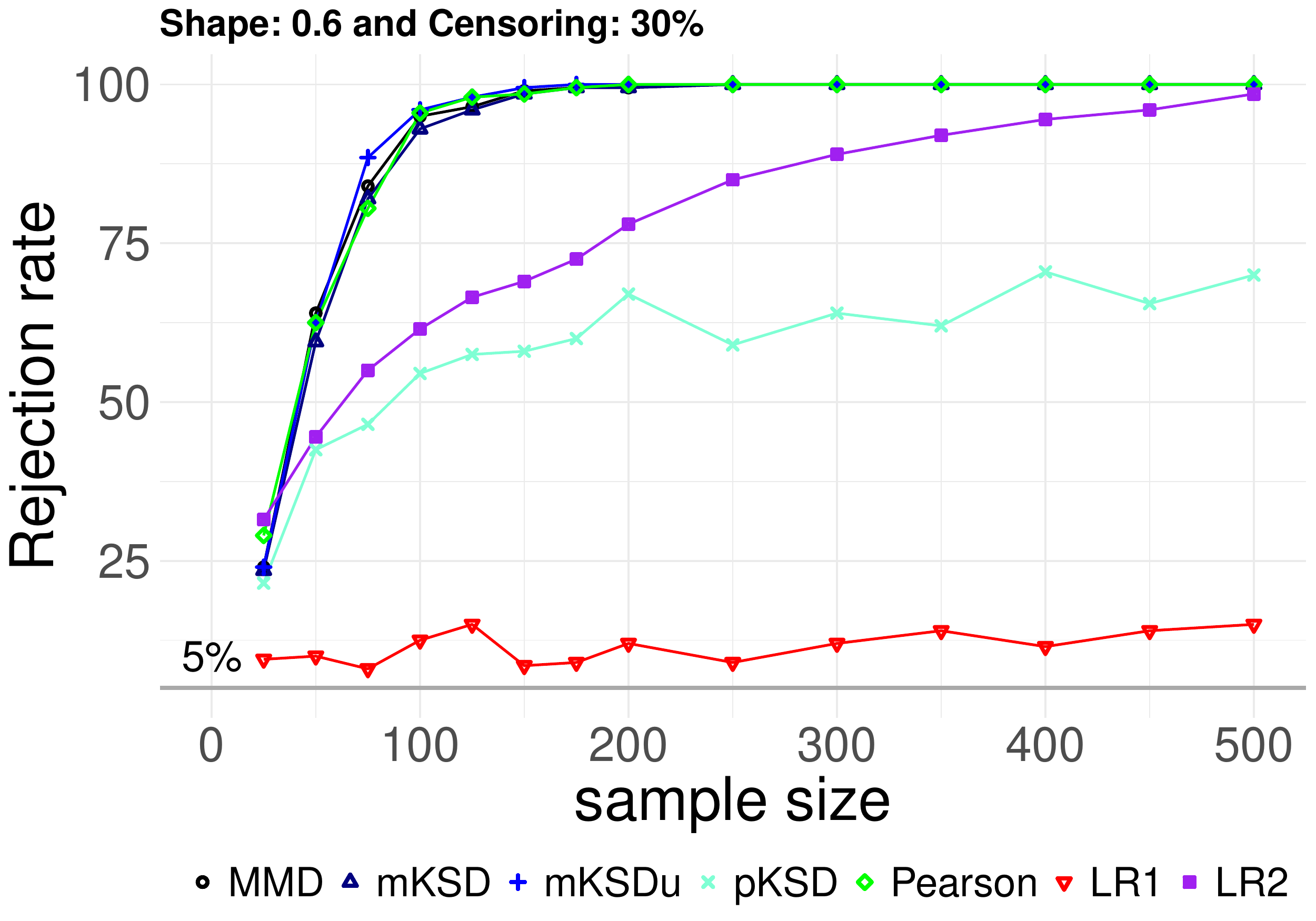}
    \includegraphics[scale=0.3]{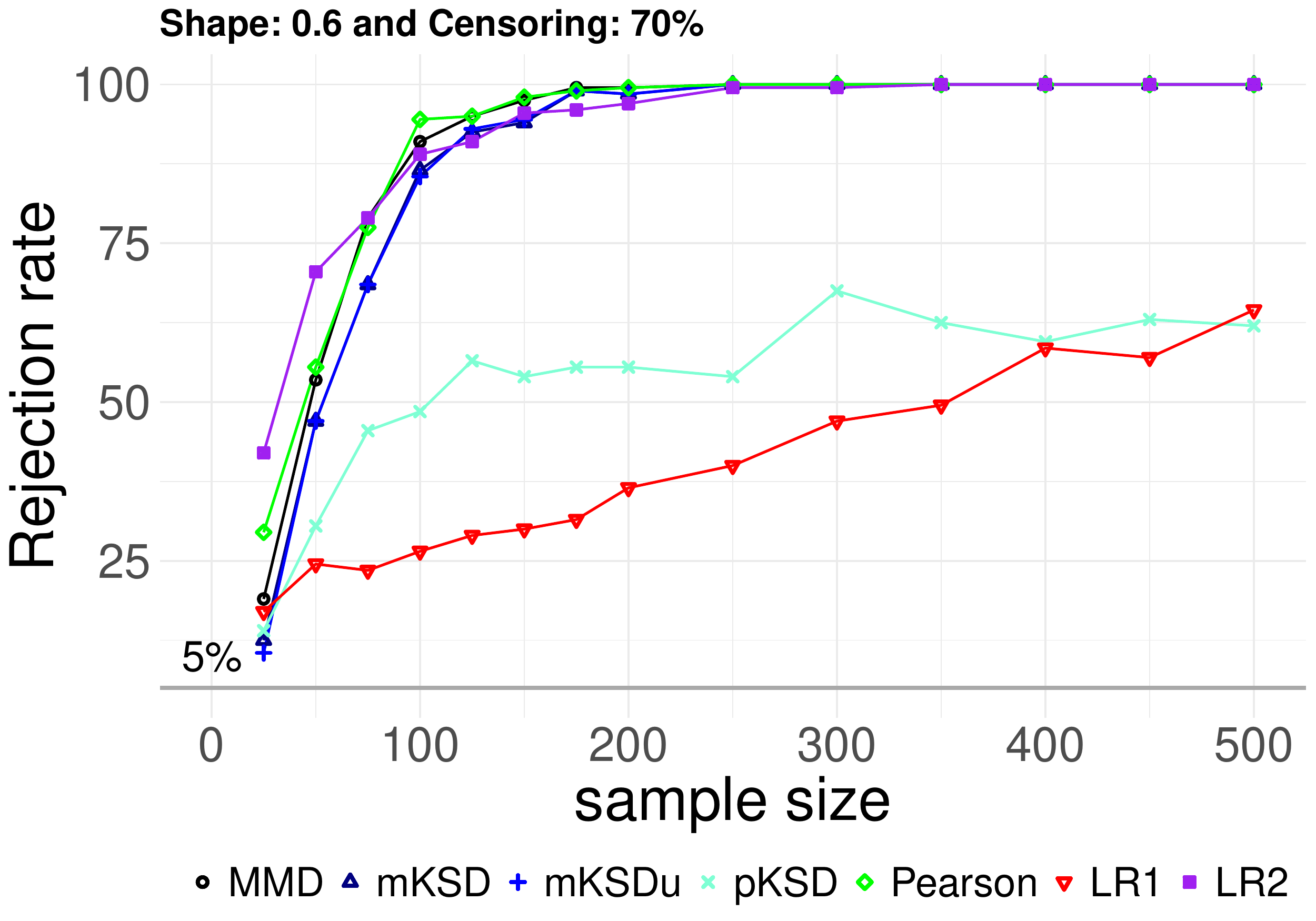}
\end{figure}

\subsubsection*{Shape: $1.4$ and censoring percentages of $30\%$ and $70\%$}
\begin{figure}[H]
    \centering
    \includegraphics[scale=0.3]{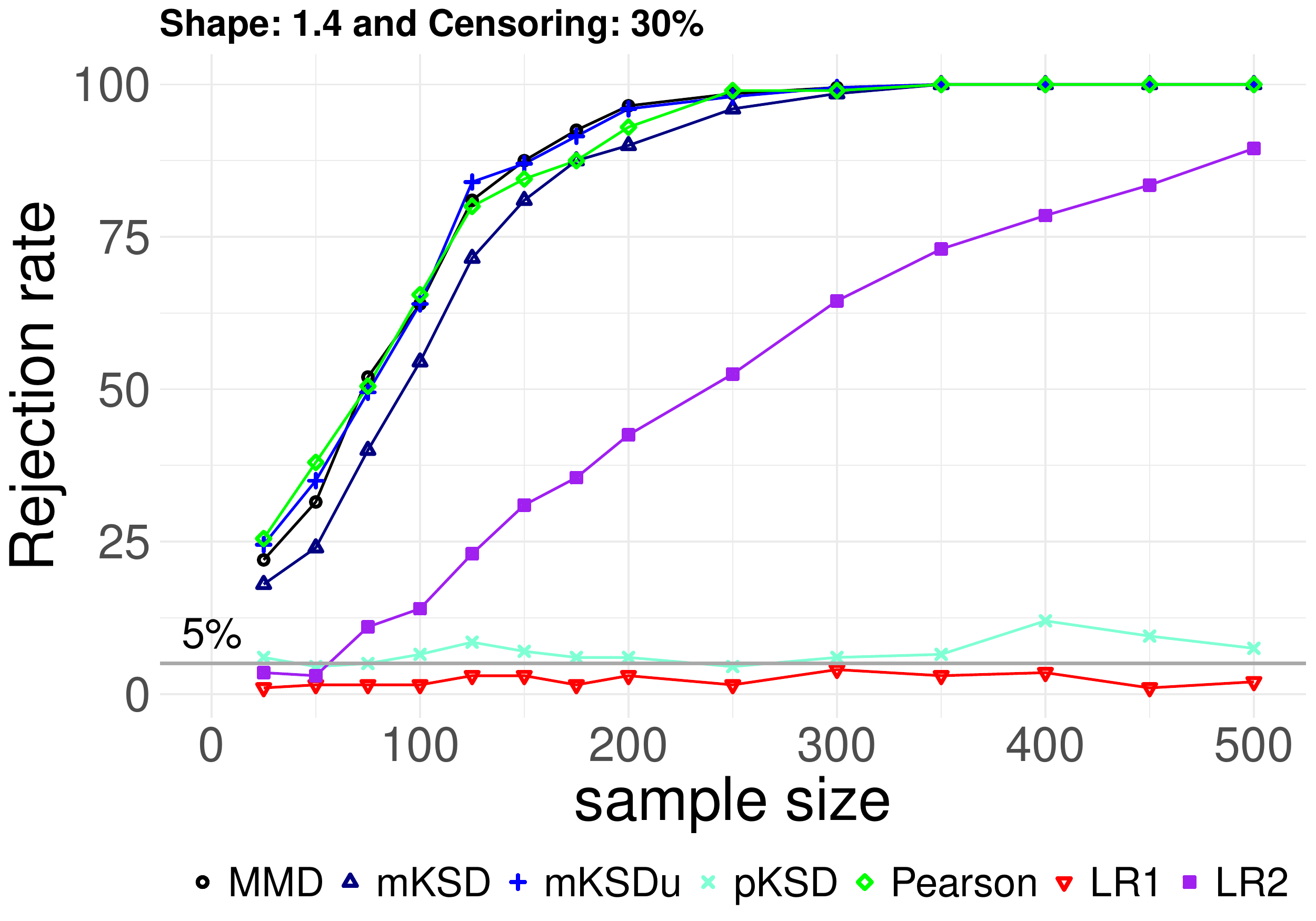}
    \includegraphics[scale=0.3]{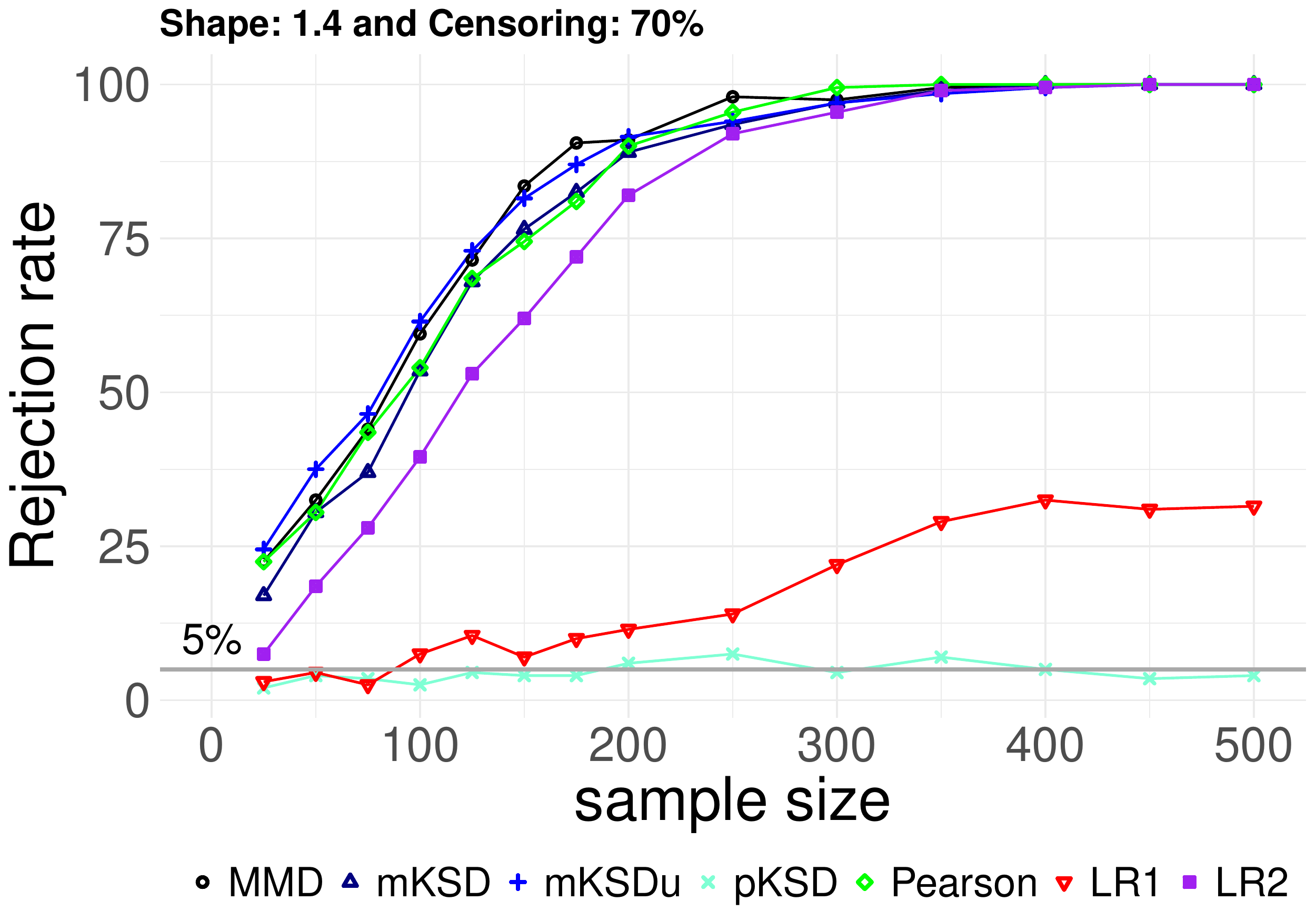}
\end{figure}

\subsection{Periodic experiments: small deviations from the null}
\subsubsection*{Sample size: 30, and censoring percentages of $30\%$, $50\%$ and $70\%$}
\begin{figure}[H]
    \centering
    \includegraphics[scale=0.21]{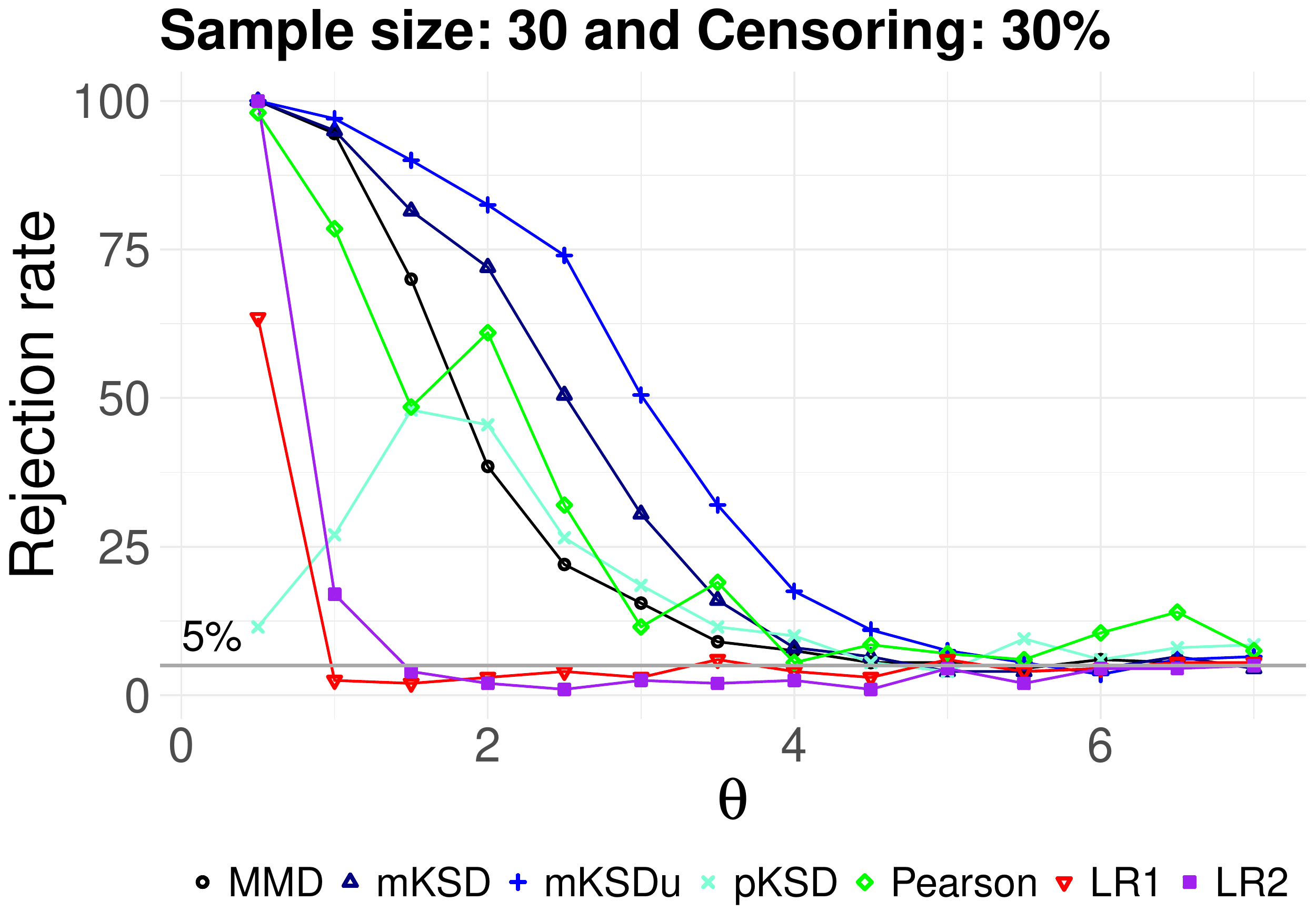}
    \includegraphics[scale=0.21]{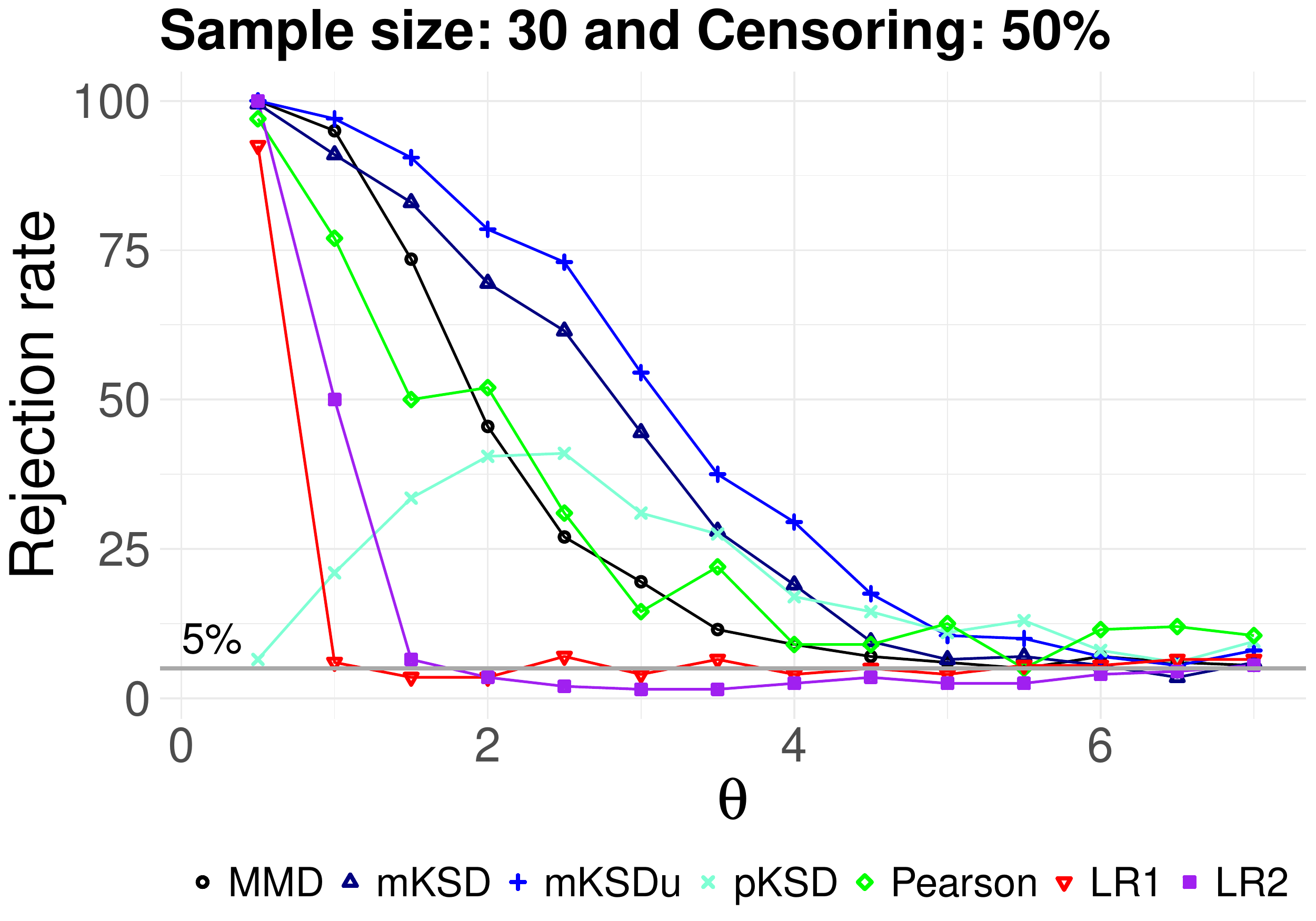}
    \includegraphics[scale=0.21]{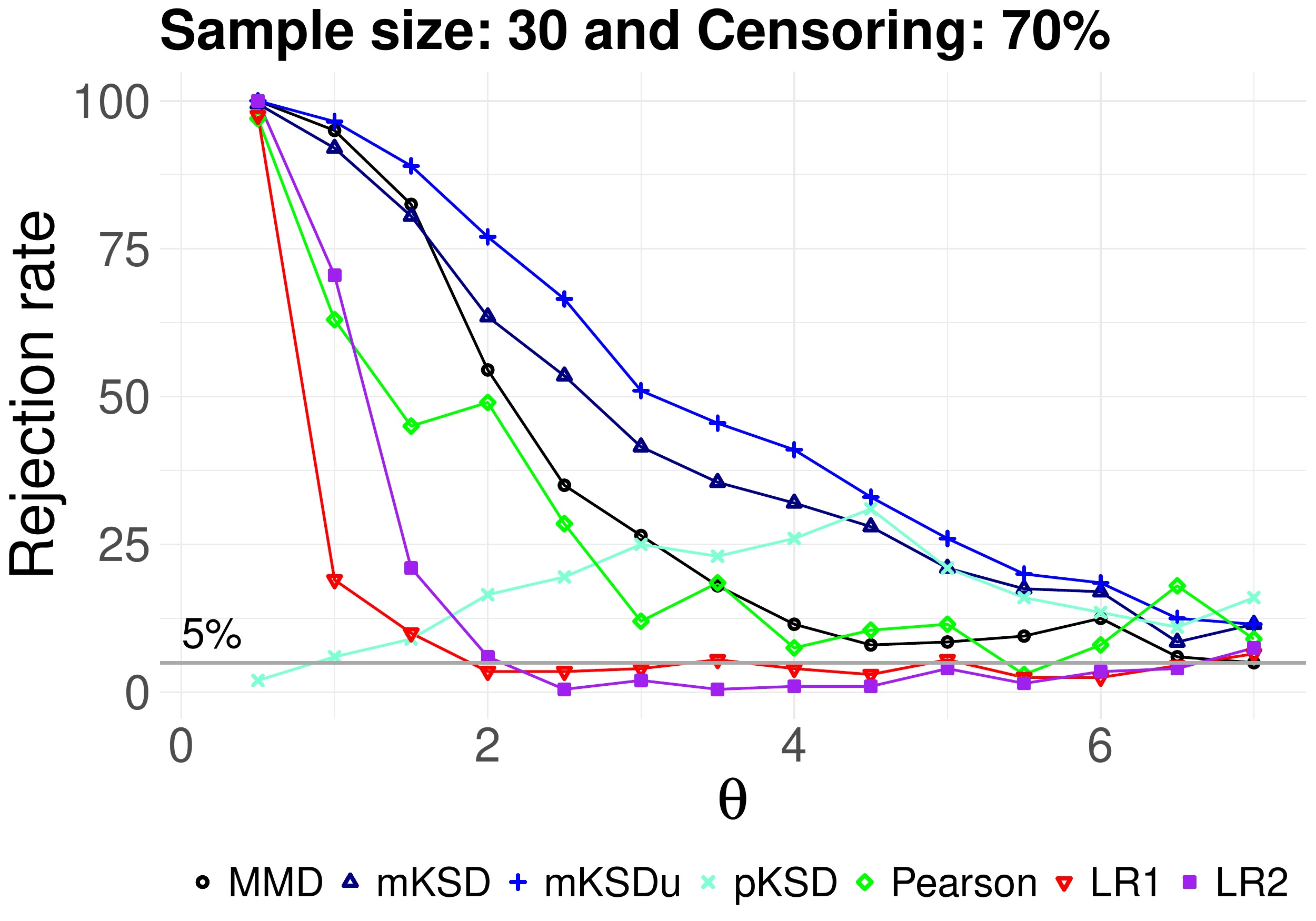}
\end{figure}
\subsubsection*{Sample size: 50, and censoring percentages of $30\%$, $50\%$ and $70\%$}
\begin{figure}[H]
    \centering
    \includegraphics[scale=0.21]{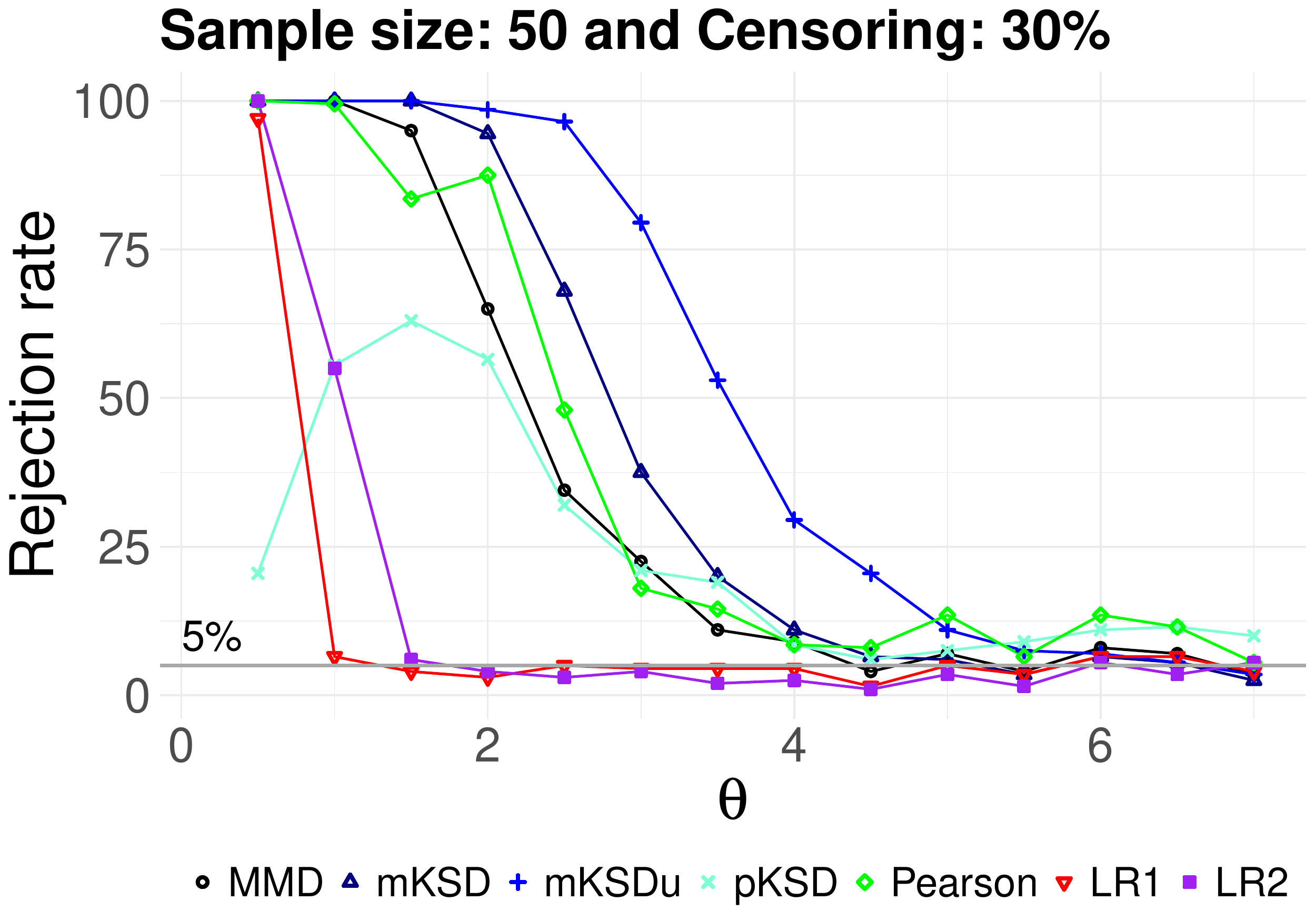}
    \includegraphics[scale=0.21]{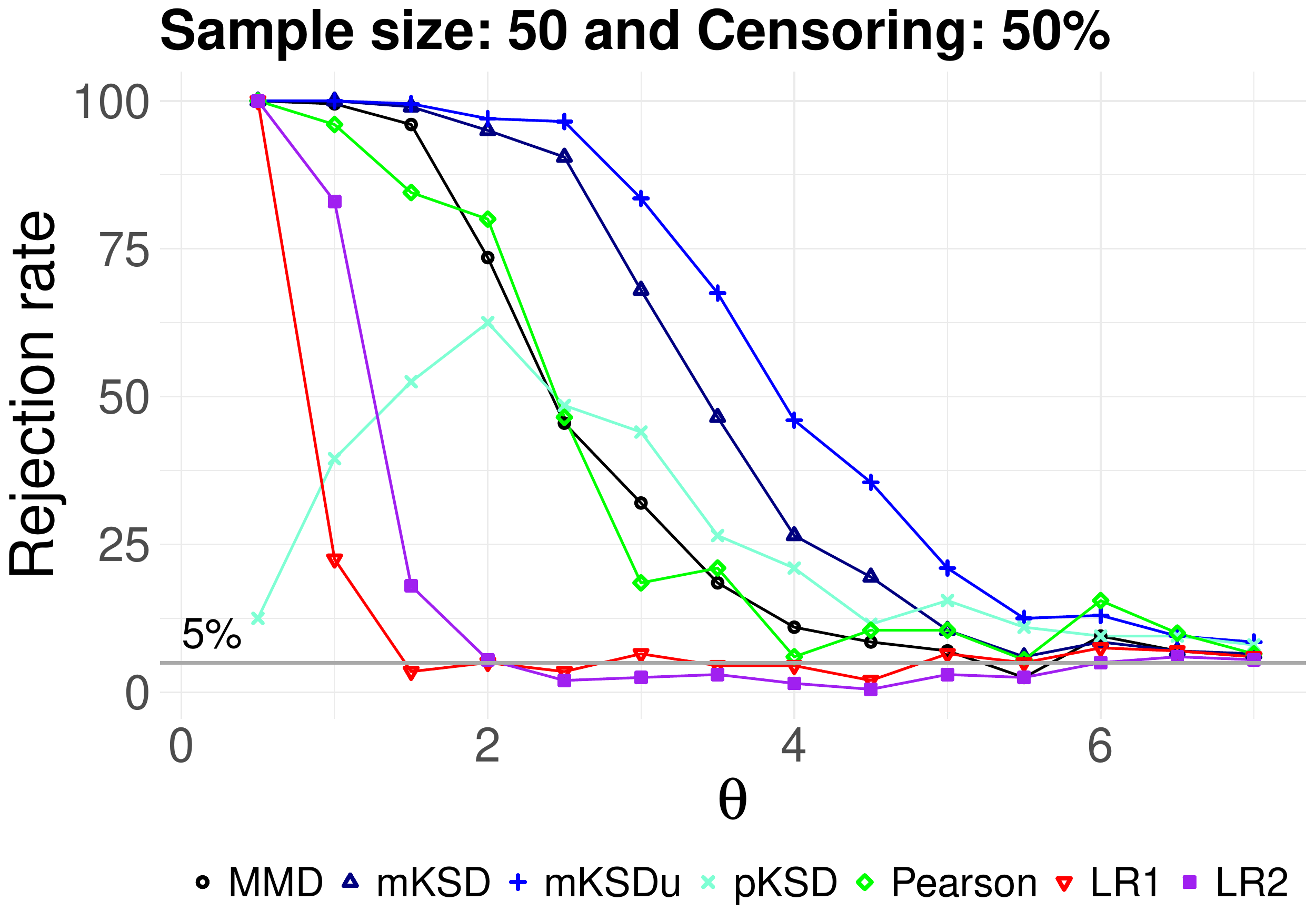}
    \includegraphics[scale=0.21]{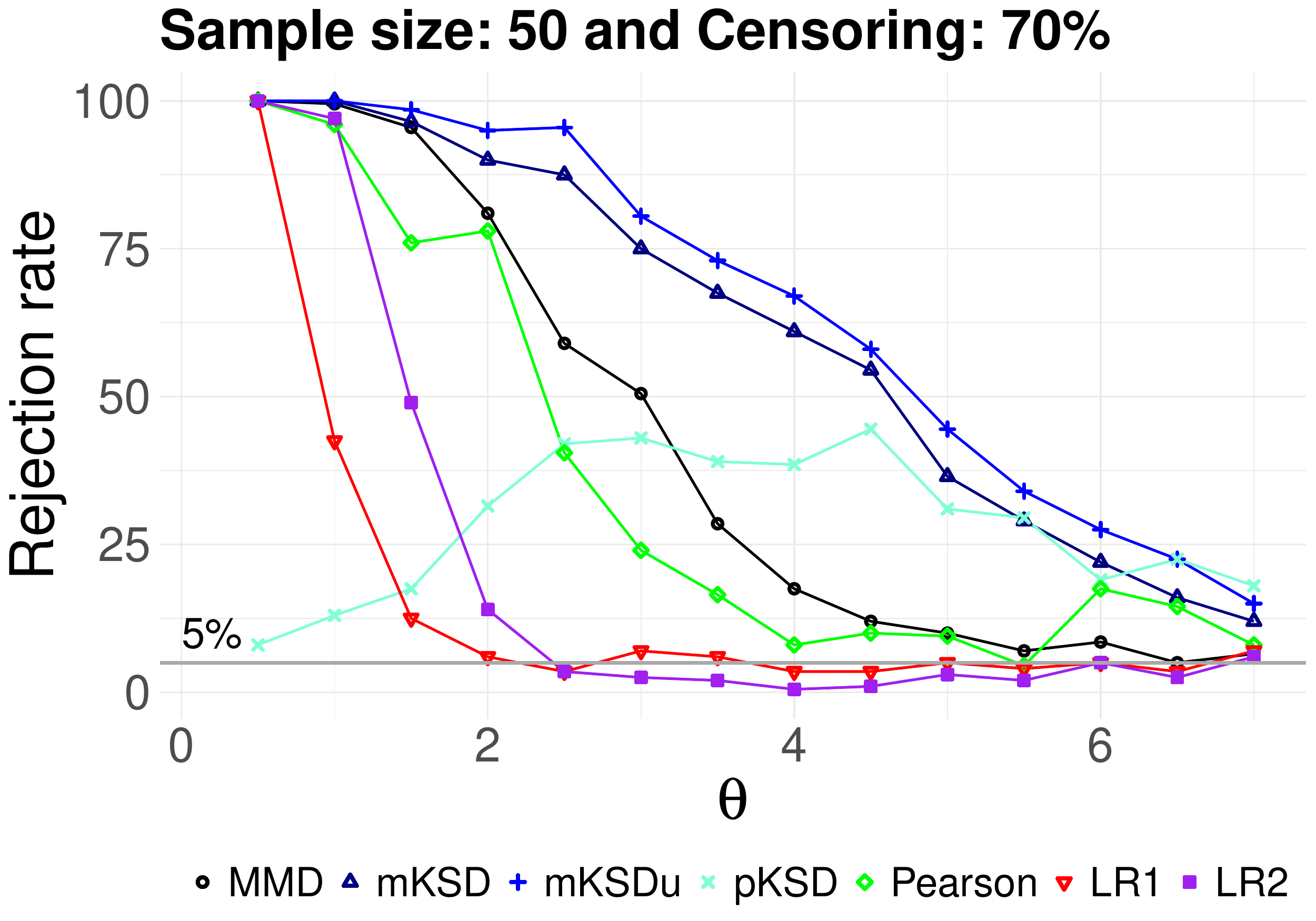}
\end{figure}
\subsubsection*{Sample size: 100, and censoring percentages of $30\%$, $50\%$ and $70\%$}
\begin{figure}[H]
    \centering
    \includegraphics[scale=0.21]{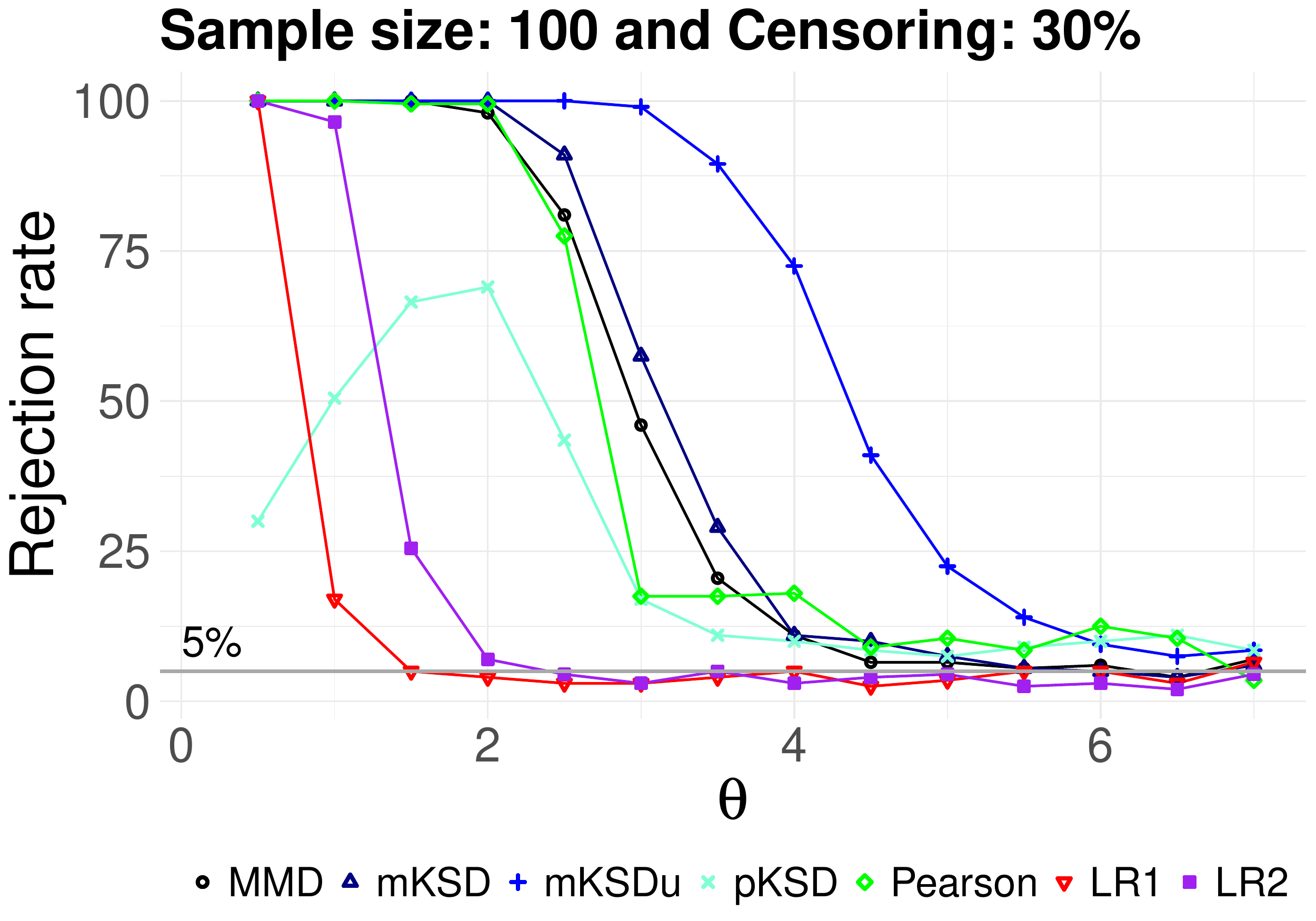}
    \includegraphics[scale=0.21]{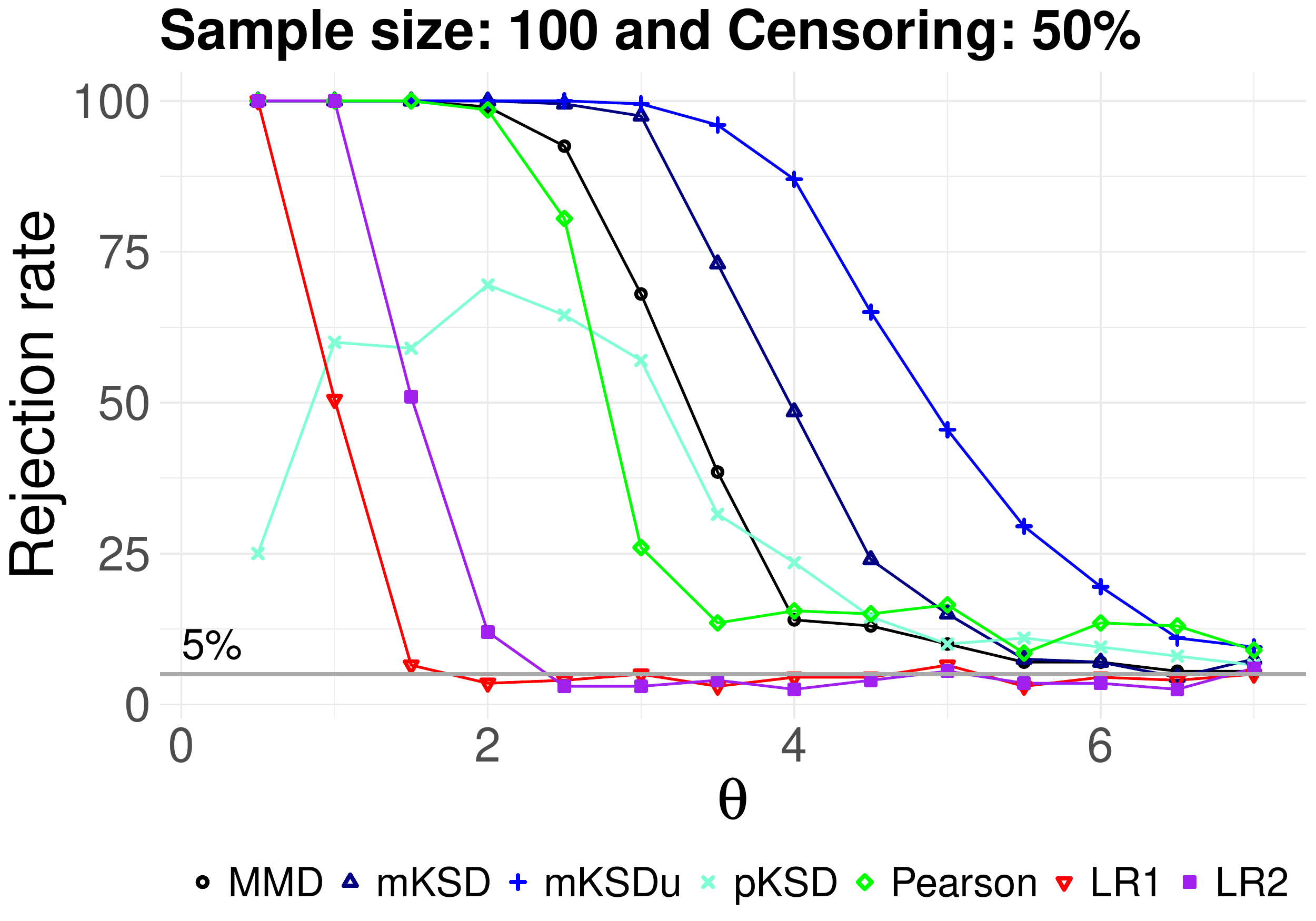}
    \includegraphics[scale=0.21]{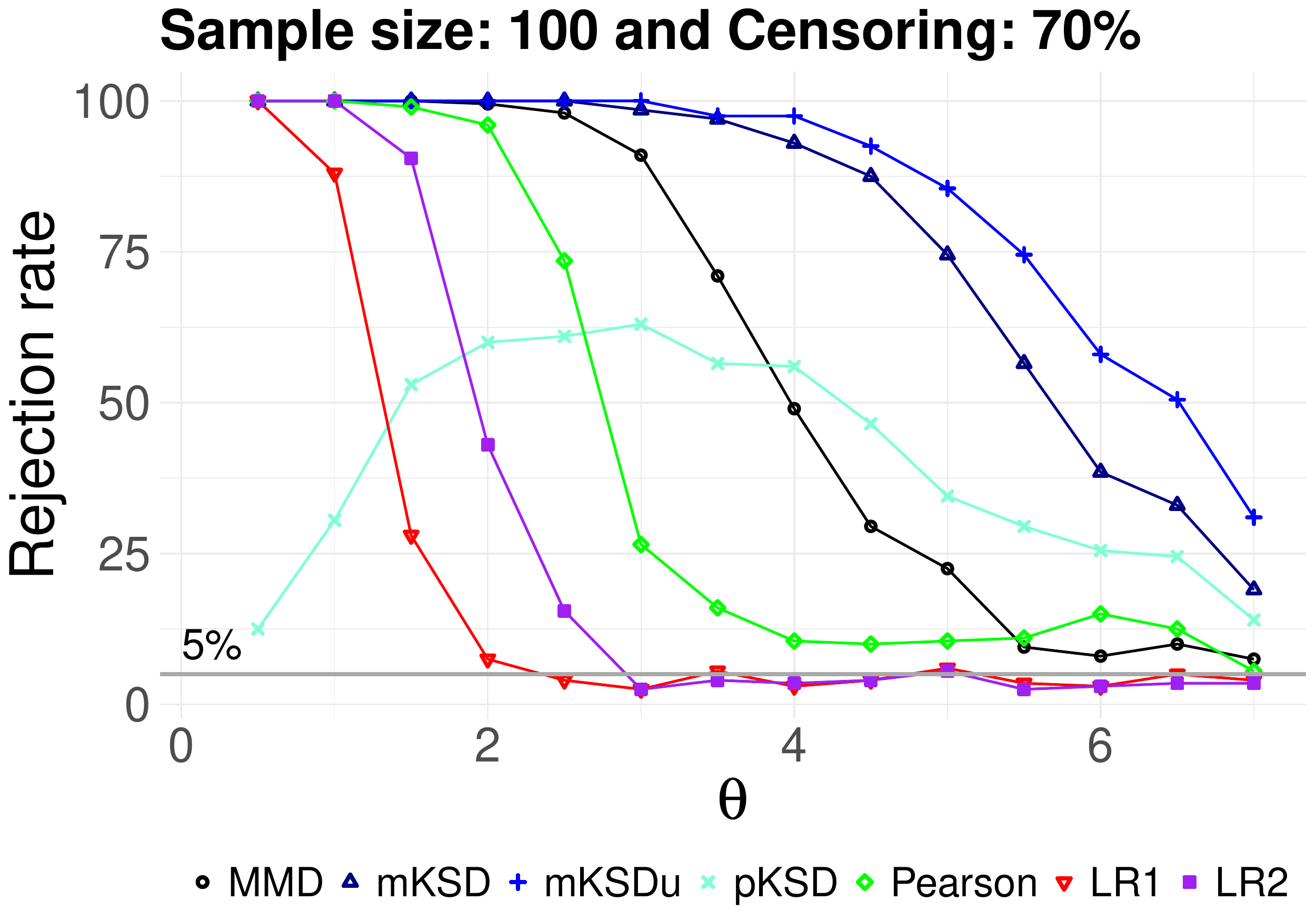}
\end{figure}
\subsubsection*{Sample size: 200, and censoring percentages of $30\%$, $50\%$ and $70\%$}
\begin{figure}[H]
    \centering
    \includegraphics[scale=0.21]{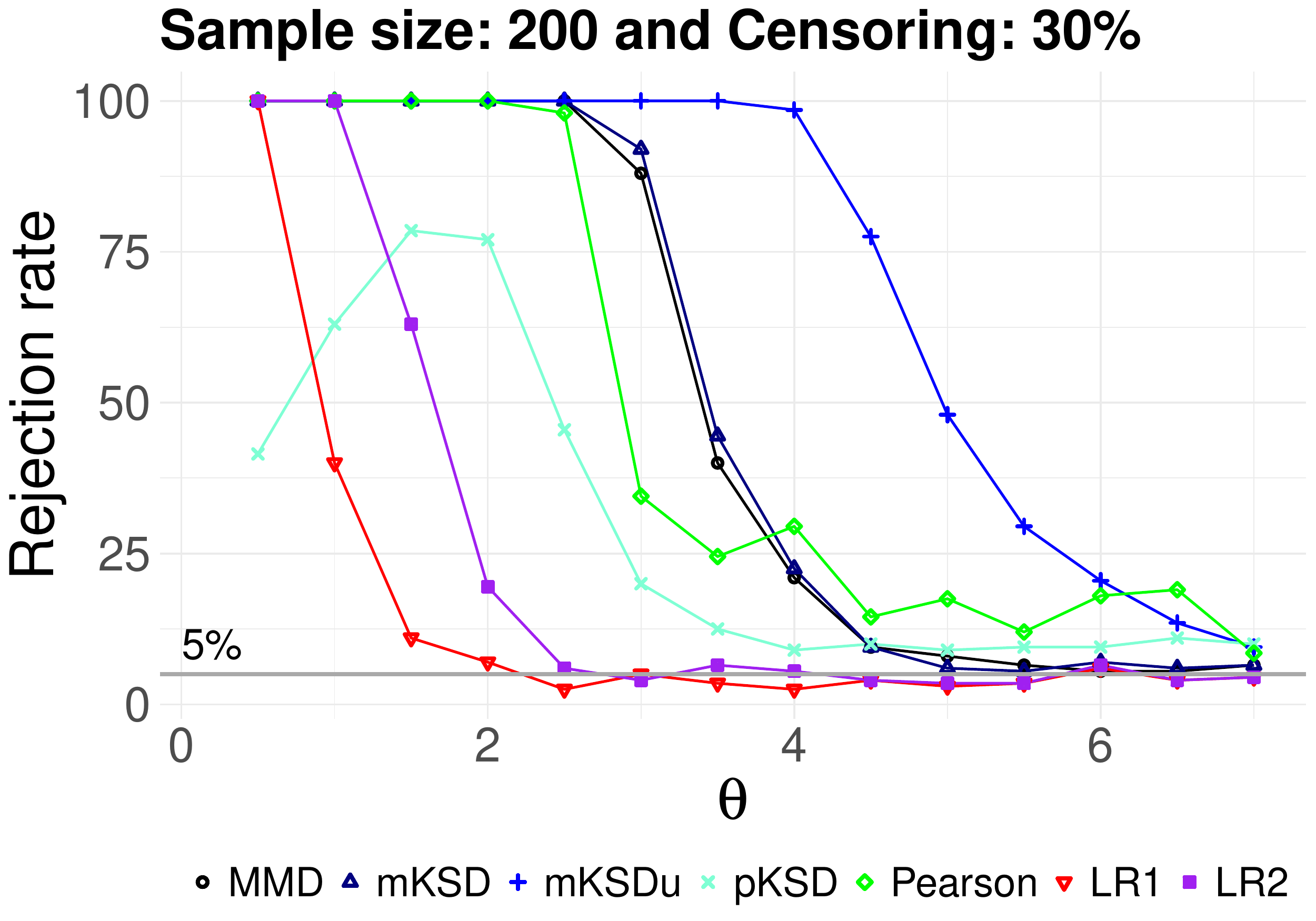}
    \includegraphics[scale=0.21]{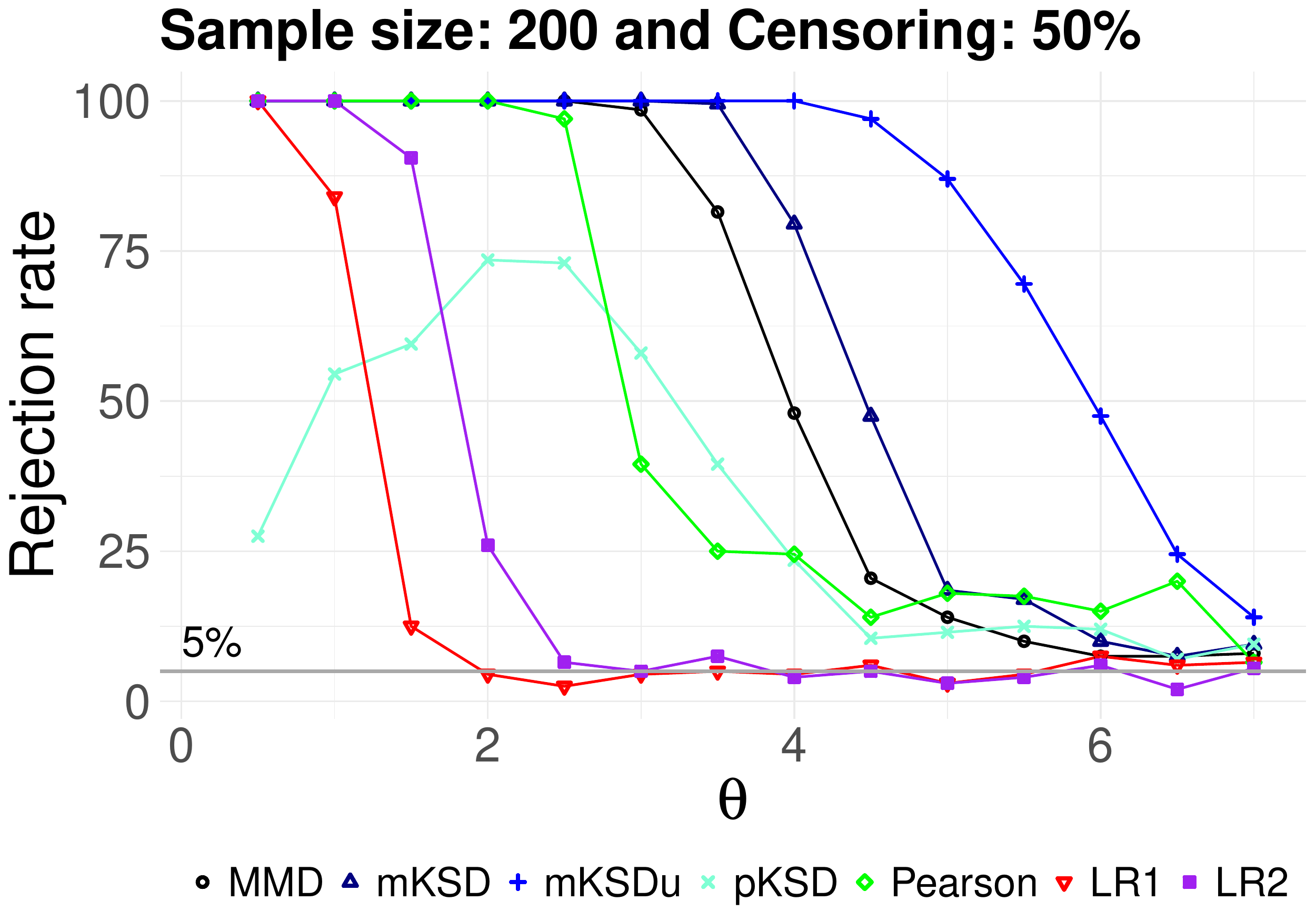}
    \includegraphics[scale=0.21]{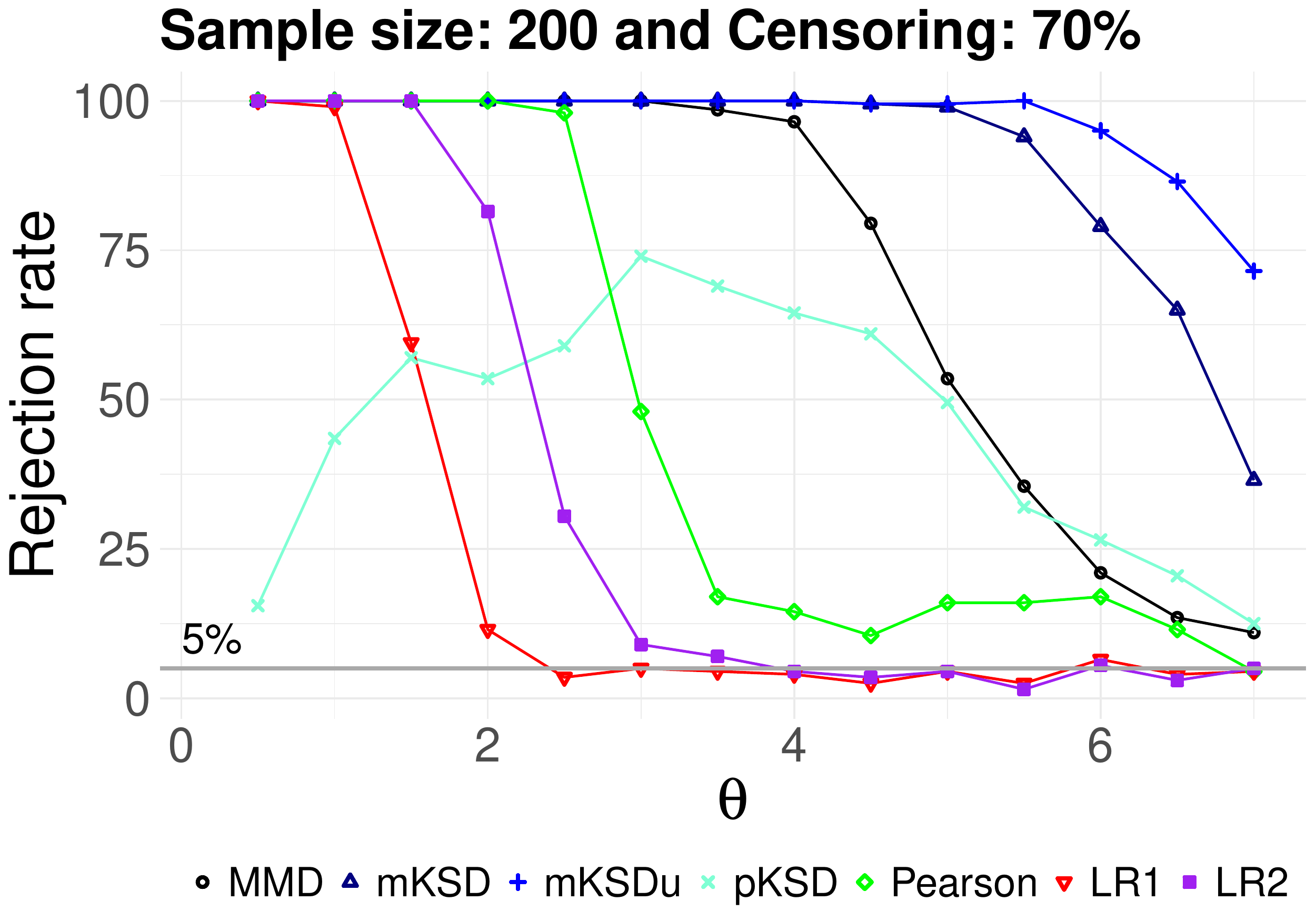}
\end{figure}

\subsection{Periodic experiments: increasing sample size}

\subsubsection*{Frequency: $3$ and censoring percentages of $30\%$ and $70\%$}
\begin{figure}[H]
    \centering
    \includegraphics[scale=0.3]{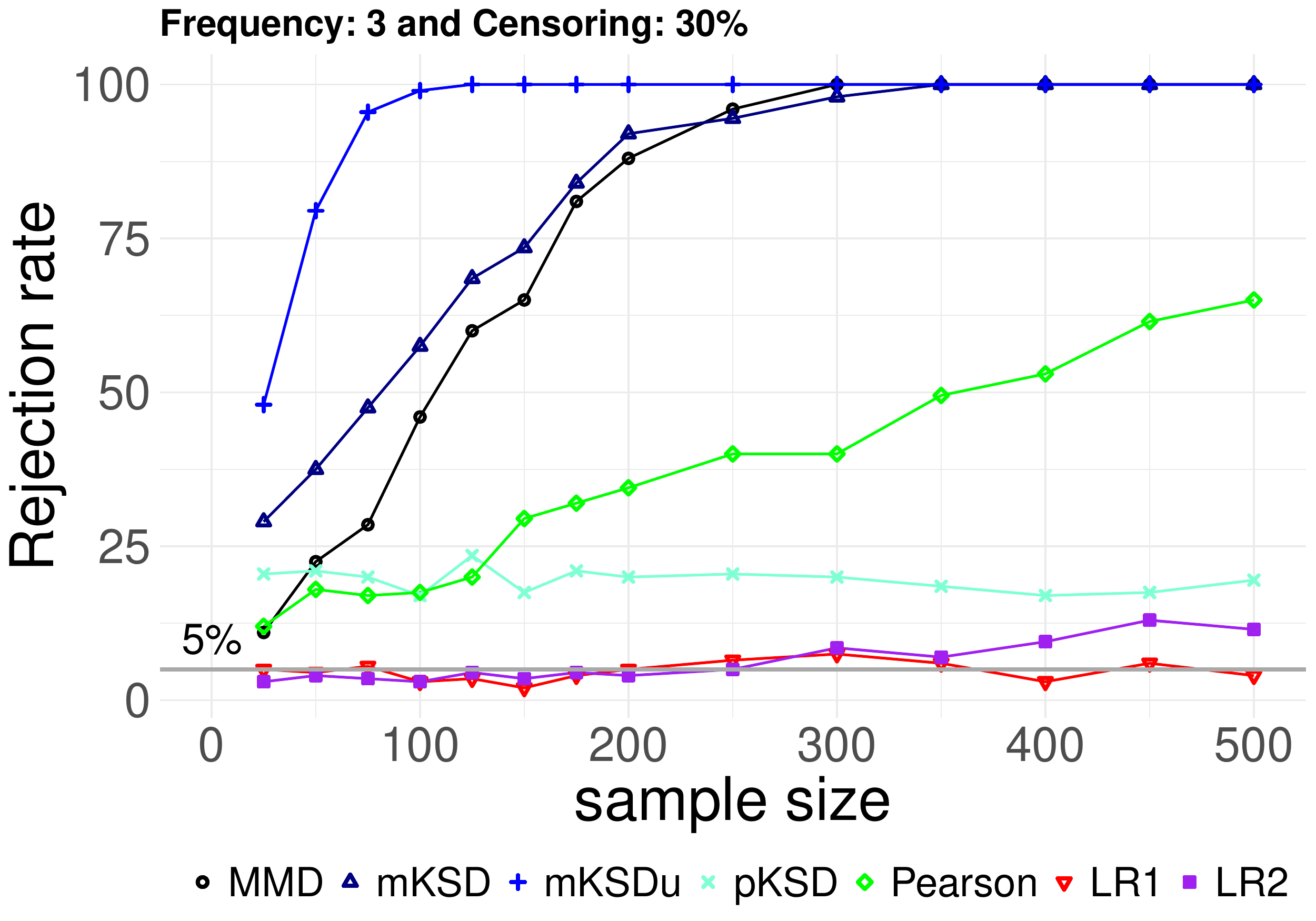}
    \includegraphics[scale=0.3]{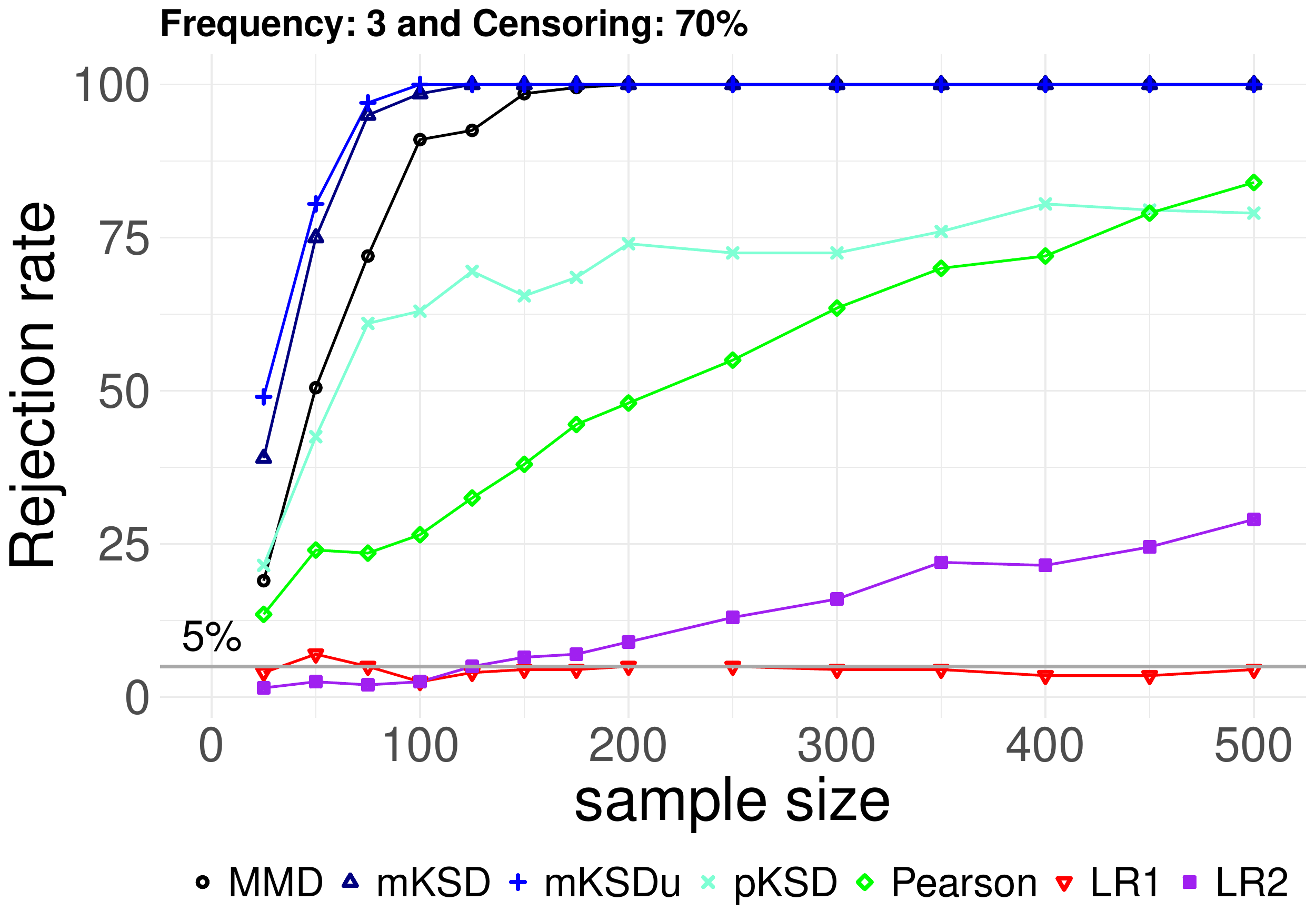}
\end{figure}

\subsubsection*{Frequency: $6$ and censoring percentages of $30\%$ and $70\%$}
\begin{figure}[H]
    \centering
    \includegraphics[scale=0.3]{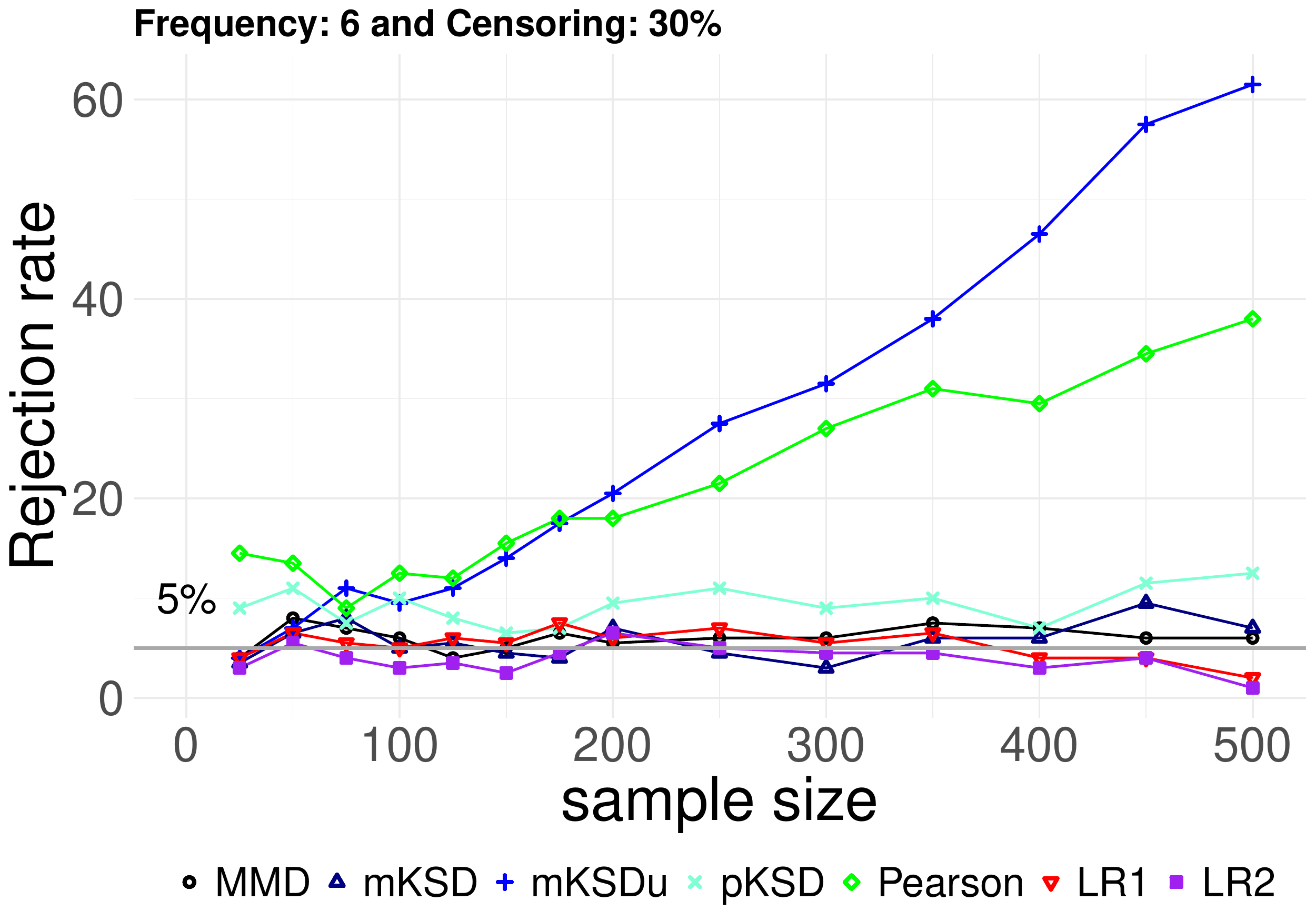}
    \includegraphics[scale=0.3]{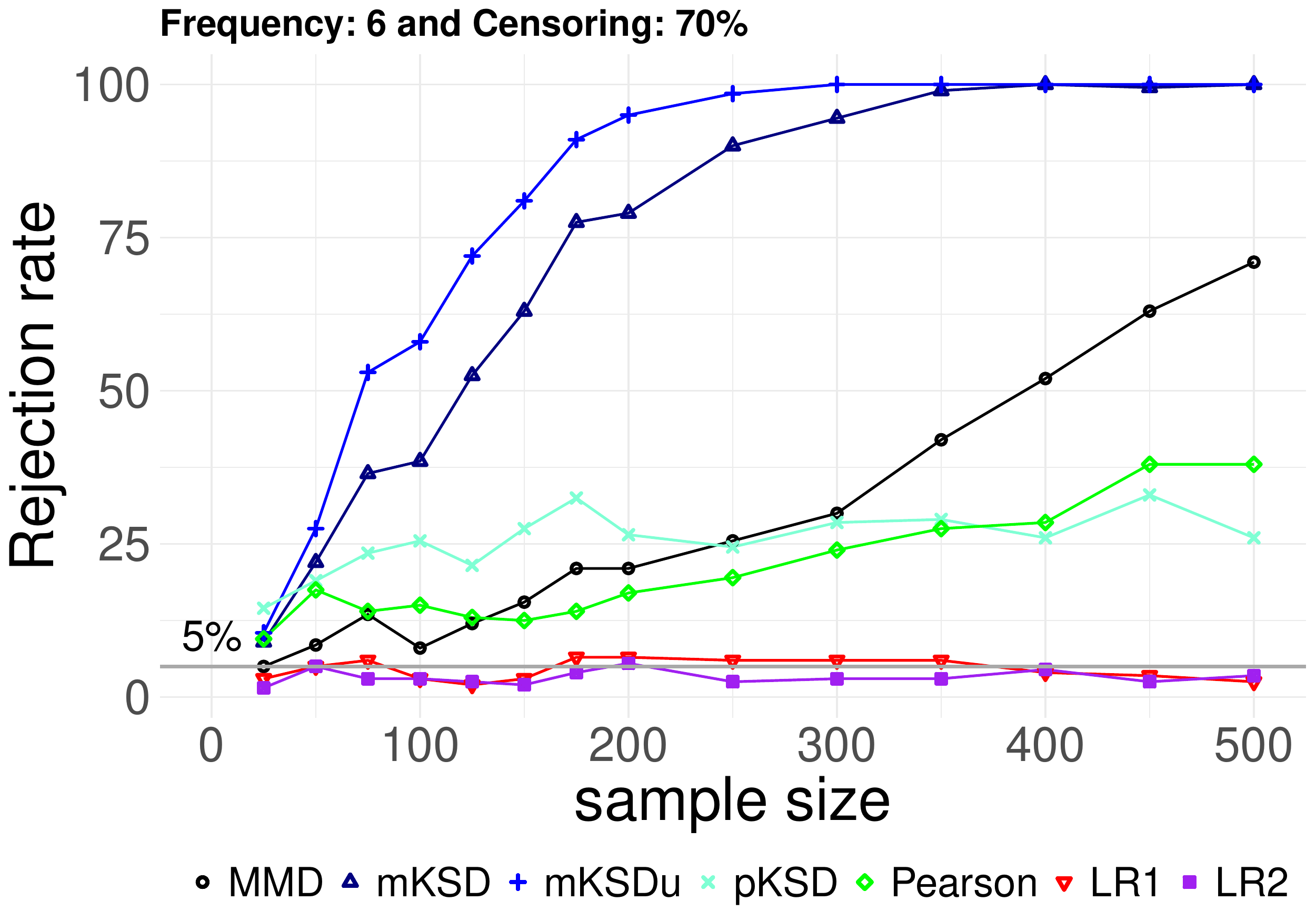}
\end{figure}

\bibliography{stein}
\bibliographystyle{apalike}

\end{document}